\documentclass[10pt,conference,onecolumn]{IEEEtran}
\usepackage{amsmath, amssymb, graphicx, mathtools, comment, xcolor, booktabs, footmisc, threeparttable,bbm, float}

\usepackage{cite}

\usepackage[colorlinks=true, 
            linkcolor=blue, 
            citecolor=blue, 
            urlcolor=blue]{hyperref}

\usepackage{tcolorbox}
\tcbuselibrary{breakable}
\newtheorem{thm}{Theorem}
\newtheorem{conj}{Conjecture}
\newtheorem{proof}{Proof}

\tcolorboxenvironment{thm}{
  colback=blue!10,
  colframe=gray!70,
  boxrule=0.5pt,
  left=1.0em,
  right=1em,
  top=0.3em,
  bottom=0.6em,
  before skip=10pt,
  after skip=10pt,
  breakable
}

\tcolorboxenvironment{conj}{
  colback=orange!12,
  colframe=gray!70,
  boxrule=0.5pt,
  left=1.0em,
  right=1em,
  top=0.3em,
  bottom=0.6em,
  before skip=10pt,
  after skip=10pt,
  breakable
}

\tcolorboxenvironment{proof}{
  colback=gray!30,
  colframe=gray!70,
  boxrule=0.5pt,
  left=1.0em,
  right=1em,
  top=0.3em,
  bottom=0.6em,
  before skip=10pt,
  after skip=10pt,
  breakable
}

\IEEEoverridecommandlockouts

\title{A Theory of Inference Compute Scaling: Reasoning through Directed Stochastic Skill Search}

\author{
\IEEEauthorblockN{Austin R.\ Ellis-Mohr, Anuj K.\ Nayak, and Lav R.\ Varshney}
\IEEEauthorblockA{University of Illinois Urbana-Champaign, Urbana, IL, USA}
\IEEEauthorblockA{\{austine4, anujk4, varshney\}@illinois.edu}
}

\begin{document}
\thispagestyle{plain}
\pagestyle{plain}
\maketitle

\begin{abstract}
Large language models (LLMs) demand considerable computational, energy, and financial resources during both training and deployment. While scaling laws for training have guided much of the field’s recent progress, inference costs now represent a significant and growing component of the overall resource burden, particularly for reasoning-focused models. Existing characterizations of compute-optimality that consider model size, dataset size, and inference tokens in isolation or in fixed combinations risk overlooking more efficient operating points. We introduce directed stochastic skill search (DS3), a general framework that represents inference as stochastic traversal over a learned skill graph. From a simplified yet expressive instantiation, we derive closed-form expressions for task success and compute cost across a wide range of inference strategies---including chain-of-thought (CoT) and tree-of-thought (ToT)---enabling comparative analysis as a function of task difficulty and model capability. To that end, we extend a prior first-principles tripartite graph framework of LLM training to incorporate inference, and separately bridge DS3 with empirical methods that characterize LLM scaling behavior. We theoretically recover empirically observed patterns, including: linear accuracy scaling with logarithmic compute; variation in preferred inference strategies as a function of task difficulty and model capability; emergent behavior elicited by reasoning even when performance plateaus under parameter scaling; and both best-of-$N$ (Bo$N$) and majority voting behavior captured within a unified analytical framework. By explicitly characterizing training-inference interdependencies, our framework deepens theoretical understanding and supports principled algorithmic design and resource allocation.
\end{abstract}

\section{Introduction}
Large language model (LLM) performance on tasks has been shown to improve by scaling the computation spent on learning (i.e., training compute), driven by scaling up both model size and dataset volume \cite{Kaplan2020, Hoffmann2022}. This scaling produces progressively richer and more structured internal representations, enabling models to better capture abstract concepts and generalize to unseen tasks \cite{Finzi2025computeoptimalllmsprovablygeneralize}. As a result, the training compute used for frontier models has grown by an estimated 4--5$\times$~per year \cite{EpochAI2024computegrows}. Concurrently, algorithmic improvements have shifted the compute–performance frontier \cite{Pilz2024}, reducing the physical compute required to reach a given performance level by roughly 2–6$\times$~annually \cite{Hernandez2020, Ho2024}.

Performance has also been shown to scale with computation spent on reasoning (i.e., inference compute, also referred to as test-time compute, search, or thinking). Therefore, reasoning models have been adopted by industry leaders including Anthropic \cite{Anthropic2025ExtendedThinking}, DeepSeek \cite{DeepSeekAI2025}, Google DeepMind \cite{Kavukcuoglu2025Gemini25}, IBM \cite{IBM2025Granite32}, Microsoft \cite{Microsoft2025phi}, OpenAI \cite{OpenAI2024o1, OpenAI2025o3}, and xAI \cite{xAI2025Grok3}. As inference-time compute is scaled, a model's capacity to dynamically traverse and leverage its learned internal structure may be enhanced.

Yet, the performance gains from inference scaling come at a cost. Whereas training involves a substantial but largely one-time expenditure, inference uses energy throughout the model’s operational lifetime. For frontier LLMs, inference compute often substantially outweighs training compute\footnote{Although a 2022 study on Hugging Face's BLOOM 176B parameter LLM has been cited to the contrary \cite{DeVries2023growingenergy}, this interpretation is misleading due to the light usage. The reported emissions were 50.5~tonnes of CO\textsubscript{2}-eq for training (including manufacturing and operational energy use) and at least 368~mg per inference when eliminating the estimated memory overhead~\cite{Luccioni2023bloom}---which implies that 137 million inferences would match total training emissions. At production scale (billions of queries per day), inference rapidly dominates carbon footprint and cost.}. Google reported that 60\% of their total machine learning energy use between 2019 and 2021 was attributable to inference \cite{Patterson2022}, while Meta (formerly Facebook AI) reported a 65\% inference share for their LLM in 2022 \cite{Wu2022sustainable}. Although the cost per unit of inference has fallen---by an estimated 90\% over 18 months---and the cost per task is likewise decreasing \cite{OpenAI2025o3}, with further reductions expected due to algorithmic advances, intensified competition, and the shift from general-purpose GPUs to custom ASICs \cite{Barclays2024}, the total burden of inference is likely to continue to grow. 
Such efficiency gains may ultimately drive greater inference demand, fueled both by enhanced capabilities and by broader sociotechnical dynamics \cite{Luccioni2025Jevons}. This mirrors the Jevons Paradox, the 19th-century observation that greater efficiency in coal use led to increased overall consumption \cite{Jevons1866}. With usage now scaling to billions of queries daily \cite{Verge2024chatgpt} and accelerating \cite{Verge2025chatgpt, DemandSage2025chatgpt}, the cumulative cost of inference is poised to grow substantially \cite{DeVries2023growingenergy}. This is further amplified by the inference scaling paradigm, where costs grow nonlinearly as task lengths increase. Improved model capabilities, multimodal models, and agentic workflows invite their use in increasingly demanding, multi-step tasks. Usage patterns, deployment environments, algorithmic structures, and hardware architectures all contribute to highly-variable energy profiles. Nonetheless, compared to training, inference has historically attracted less attention with respect to energetic costs \cite{Verdecchia2023greenAI}.

Compute scaling during both training and inference shapes model capabilities and forms a deeply interdependent system. Energy costs, architectural decisions, and performance tradeoffs all arise from these coupled dynamics. A unified perspective is essential to guide resource allocation, research priorities, and policy frameworks toward sustainable AI grounded in the co-evolution of training and inference.

\subsection{Games}
The domain of games offers an early window into how inference and training became increasingly interdependent in AI systems. Many early expert AI systems concentrated most of their compute on inference‑time search. For example, during play, IBM’s Deep Blue evaluated over one hundred million chess positions per second using brute-force tree search \cite{Campbell2002}. 
Exhaustive enumeration and handcrafted pruning heuristics stood in for generator and verifier models, both of which were not yet feasible with learning methods at the time. In contrast, DeepMind’s AlphaGo (2016) introduced a clear division of labor: a policy network acted as a generator, proposing promising actions, while a value network served as a verifier, scoring leaf positions \cite{Silver2016AlphaGo}. Trained through supervised and reinforcement learning, the networks were paired at inference time via Monte Carlo tree search. This allowed AlphaGo to achieve superhuman performance in Go---a game more combinatorially complex than chess---while evaluating orders of magnitude fewer positions than Deep Blue.

One year later, AlphaGo Zero collapsed the distinction between generator and verifier into a single neural network, trained entirely through reinforcement learning via self-play \cite{Silver2017AlphaGoZero}. Despite substantial base-level competence, it revealed a stark gap between the raw network and its performance when paired with search, underscoring the impact of inference-time compute. Moreover, performance scaled with thinking time and with stronger move generation enabled by training; notably, the scaling of performance with thinking time was stronger than for conventional engines \cite{Silver2017ChessShogiGo}. The ability to scale performance through both training and inference has since been demonstrated in other perfect \cite{Jones2021} and imperfect information games \cite{Brown2017}, including the multi-agent, natural language game Diplomacy \cite{Meta2022}.

Notably, AlphaGo Zero also exhibited emergent training dynamics: small increases in training time led to large improvements in performance, followed by plateaus \cite{Silver2017AlphaGoZero}.

\subsection{LLM Training Scaling}
\begin{figure}[htbp]
    \centering    \includegraphics[width=0.8\textwidth]{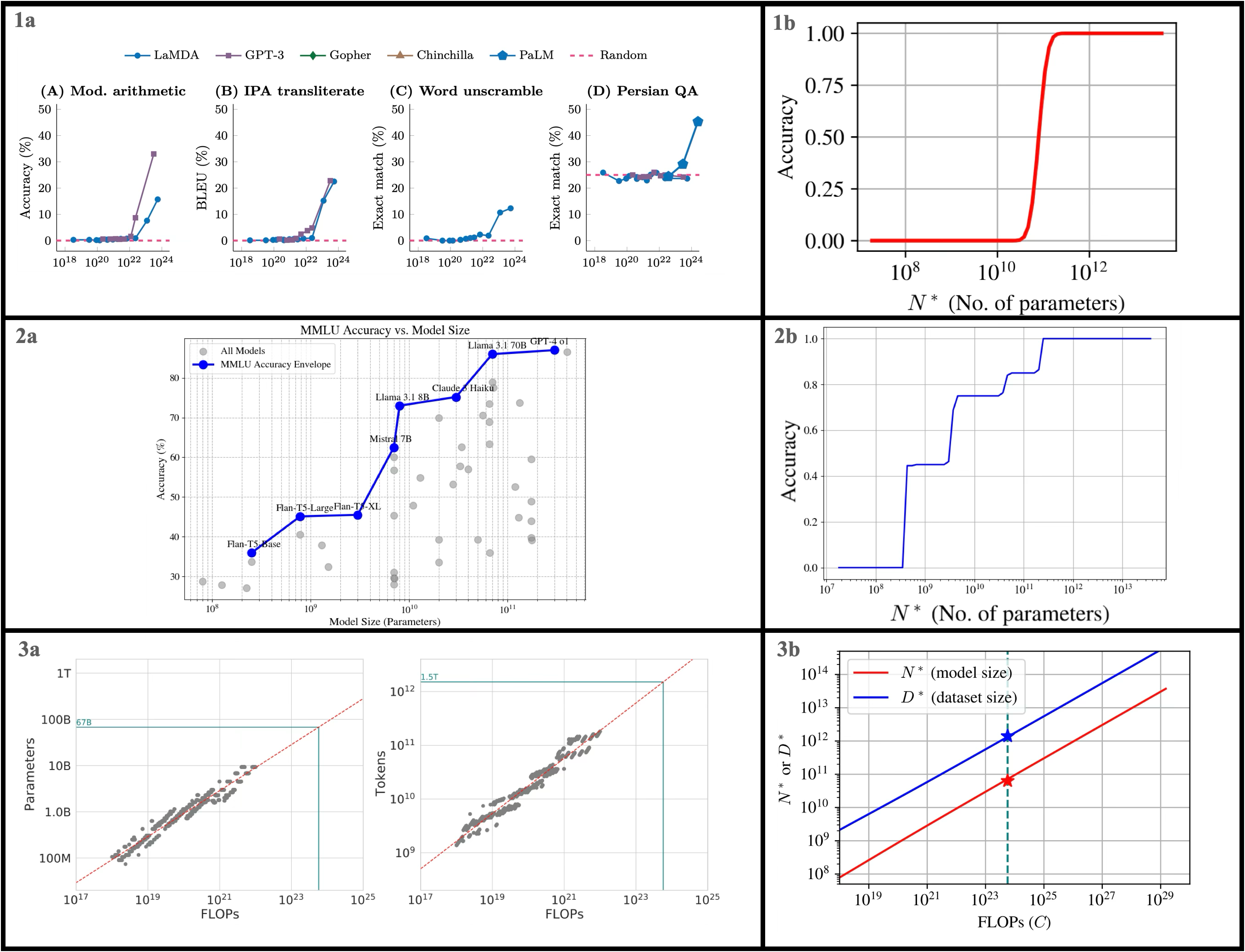}
    \caption{
    1.a. Empirical accuracy scaling on various tasks under increasing model size, showing sharp capability transitions for various models \cite{Wei2022Emergence}.
    1.b. Theoretical prediction of emergent behavior using the hierarchical skill–text tripartite graph framework introduced in \cite{Nayak2025} and adapted in this paper, with success rate transitioning sharply as model size crosses the task-specific threshold.
    2.a. Empirical accuracy on MMLU as a function of model size, with envelope highlighting best-reported performance \cite{MMLULeaderboard}.
    2.b. Theoretical plateauing as a function of model size due to a hierarchical distribution of tasks. Same framework as in 1.b \cite{Nayak2025}.
    3.a. Empirical Chinchilla scaling of model size (left) and dataset size (right) with training FLOPs, showing trends with consistent proportionality constants \cite{Hoffmann2022}.
    3.b. Theoretical training-time-compute-optimal model size and data sizes. Same framework as in 1.b \cite{Nayak2025}.
    }
    \label{fig:ChinEmergPlat}
\end{figure}

Emergence and plateauing are not unique to games; these phenomena have also been documented in LLMs \cite{Wei2022Emergence}. Abrupt performance jumps across tasks with increasing model size, a hallmark of emergence, are illustrated in Fig.~\ref{fig:ChinEmergPlat}.2a. Separately, plateauing, alongside emergence, is visible in the frontier performance trends \cite{MMLULeaderboard} on the Multi-task Language Understanding benchmark \cite{Hendrycks2021MMLU}, as shown in Fig.~\ref{fig:ChinEmergPlat}.3a. While the origins of such emergent behavior remain debated \cite{Schaeffer2023mirage}, a growing body of theoretical work attributes emergence to structured learning dynamics.

The quantization hypothesis posits that knowledge is acquired in discrete units called \textit{quanta}, and the number of quanta a model can learn is proportional to its size (i.e., number of parameters) \cite{Michaud2023}. Work in mechanistic interpretability supports this view: circuit-level analyses suggest that models acquire discrete knowledge features during training \cite{Ameisen2025Circuit}. Similarly, curriculum learning studies show that training on simpler examples (e.g., frequent words) before more complex ones (e.g., rare words or long sequences) enhances convergence and generalization—consistent with staged, sequential concept learning \cite{Bengio2009curriculum}. Prior works on skill-text graph representations for LLMs \cite{Arora2023, Michaud2023, Liao2025semantic, Nayak2025} and experimental results in the Skill-Mix evaluation framework \cite{Yu2023skillmix} demonstrate the emergence of compositional abilities in multi-skill tasks as model size increases. These findings parallel insights from cognitive science, where plateaus and emergent leaps in learning have long been observed \cite{Gallistel2004plateaus, Gray2017plateausandleaps}. Recent theoretical work explicitly characterizes these transitions using discrete task difficulty regimes, in which learners remain at a plateau until they cross a difficulty threshold, triggering sudden performance gains \cite{Lu2025plateausandleaps}, supporting a view of skill structure as hierarchical.

Similarly, a graphical framework in prior work by two of the present authors parameterizes the learning process via discrete task difficulties, treating language understanding as a hierarchical system of skills composed from more foundational prerequisite skills \cite{Nayak2025}. This framework explains a range of phenomena observed in LLMs---including emergent capabilities, performance plateauing, and training-compute-optimal scaling as demonstrated in Fig.~\ref{fig:ChinEmergPlat}.1-3b---by framing concept acquisition as an iterative peeling process, akin to the decoding of low-density parity-check codes over erasure channels \cite{Sourlas1989spin, Newell1994unified}, and by applying tools from random graph theory \cite{Posfai2016network} to characterize skill emergence as a function of model scale.

Building on this foundation  allows us to explore from first principles the intricate interaction between training and inference scaling paradigms. However, our framework is also compatible with existing scaling laws for pretraining loss, including their power-law structure and interpretation. While our focus will ultimately be on downstream task performance, we demonstrate how these results can be connected back to the well-characterized scaling of pretraining loss with model size and data.

\subsubsection{Pretraining Loss vs. Task Accuracy}
Many scaling laws focus on predicting pretraining loss and not task accuracy. Early work demonstrated a power-law decay of loss with scale \cite{Cortes1993scalinglaws, Hestness2017scalinglaws, Rosenfeld2019scalinglaws, Kaplan2020, Brown2020fewshotlearners, Hoffmann2022}. A widely used formulation expresses pretraining loss, $\mathcal{L}$, as the sum of an irreducible error, $E$ \cite{Hestness2017scalinglaws}---often attributed to the intrinsic entropy of natural text \cite{Shannon1951LanguageEntropy}---and diminishing return terms associated with model size and dataset size:
\begin{equation}\label{eqn:pretrainingloss}
\mathcal{L}=E+A/N^{\alpha}+B/D_{\text{tr}}^{\beta},
\end{equation}
where $N$ is the number of non-embedding model parameters, $D_{\text{tr}}$ is the number of training tokens, and $A,\ B,\ \alpha,\ \text{and }\beta$ are empirical fitting parameters.

These formulations require empirical fitting and do not offer an explanation for why these behaviors emerge. Recent theoretical work has sought to explain the origins of scaling laws \cite{Bahri2024explainingscaling} including through the aforementioned quantization hypothesis \cite{Michaud2023}.

Inference scaling is generally applied on top of pretrained models with the goal of improving downstream task performance, something only indirectly reflected in the loss itself. The connection between the two is not trivial (e.g., emergent accuracy jumps have been attributed to quantization artifacts in pretraining loss \cite{Schaeffer2023mirage}). As inference compute is increasingly scaled in practice, pretraining loss alone may offer limited guidance for system design, motivating the need to study task-level accuracy directly\footnote{One recent approach extends power-law scaling with a hypothetical logit-based adjustment to account for inference behavior \cite{Wu2024towards}. The formulation is illustrative without empirical fitting and may not readily generalize across inference strategies.}.

Two-step forecasting methods have emerged as a practical and effective strategy for characterizing downstream performance directly. These methods first fit a scaling law for pretraining loss and then apply a secondary mapping, typically a sigmoid function\cite{Meta2024llama3, Xiao2024densinglawllms, Owen2024downstreamtasksigmoid, Ruan2024observationalscalinglaws}:
\begin{equation}\label{eq:sig}
    \text{Accuracy} = \frac{a}{1+\exp(-b(\mathcal{L}_0-\mathcal{L}))}+d,
\end{equation}
where $a,\, b,\, d,\, \mathcal{L}_0$ are fitting parameters, or decaying exponential \cite{Gadre2024}---to translate loss into task performance. We demonstrate how this two-step structure can be utilized in our own framework, integrating it with a theoretical framework of task difficulty and skill composition. Other recent work has introduced flexible functional forms to capture non-monotonic behavior---such as phase transitions and double descent---across a wide range of metrics, including both loss and downstream tasks \cite{Caballero2023brokenneuralscalinglaws}.

While improvements in pretraining are generally associated with downstream accuracy gains \cite{Isik2025scalinglawsdownstreamtask}, downstream performance can plateau or saturate, even as upstream loss continues to improve---an effect observed in both vision models \cite{Abnar2021upstreamtodownstreamsaturation} and LLMs \cite{Diaz2024scalinglawsevaluations, Liu2025notjustscalinglaws, Liu2022betterdownstream}. When loss-based predictions falter, alternative indicators---such as evaluation-set size \cite{Diaz2024scalinglawsevaluations}, data-alignment features \cite{Liu2025notjustscalinglaws}, and model flatness that measures the sharpness of the loss landscape around the learned parameters \cite{Liu2022betterdownstream}---have been shown to correlate more strongly with downstream task accuracy. Moreover, overtraining during finetuning may even harm downstream performance while smoothly improving loss \cite{Springer2025overtrainedfinetuning, Chen2025rethinkingfinetuning}. Certain benchmarks display inverse or U-shaped trends, where larger models underperform smaller ones \cite{Lin2022inversescalingtruth, Parrish2022biggermorebias, Wei2023inversescaling}.

Chain-of-thought (CoT) reasoning---where models sequentially generate intermediate reasoning steps before arriving at a final answer---can be induced either explicitly through training \cite{DeepSeekAI2025, Muennighoff2025, Wang2025RLoneexample} or simply by prompting the model to ``think step by step'' \cite{Kojima2023}. CoT has been found to induce emergent capabilities on tasks that otherwise exhibit flat saturated downstream scaling and to shift inverse scaling trends toward positive or flat scaling \cite{Suzgun2022bigbenchCoT, Wei2023inversescaling}. Our theory provides an explanation for such shifts in the scaling curve, as discussed in Section~\ref{subsec:reasoning-emergence}. Recent work shows that increasing inference compute can degrade task accuracy despite improved cross-entropy loss, due to a mismatch: high-entropy outputs may aid sampling-based strategies but are penalized during training. Directly optimizing for the pass@$k$ metric yields better alignment and scaling \cite{Chen2025rethinkingfinetuning}. In fact, it has been demonstrated that treating reasoning as a random walk over clustered CoT states---a latent reasoning graph---can improve model performance via data augmentation, compared to direct supervised fine-tuning. This lends support to the view that graph-based structure plays a meaningful role in the reasoning process itself \cite{Wang2024reasoninggraph}. This further underscores the interdependence between training and inference scaling paradigms, and the need for their joint optimization.

\subsection{LLM Inference Scaling}
Recently, many in the AI community contend that simply enlarging LLMs now yields diminishing or plateauing gains \cite{Sutskever2024NeurIPSTalk, Byrnes2023, Ritter2024}, and that such limitations would persist without inference scaling \cite{Brown2024TEDTalk}. This is due in part to the limited availability of high-quality, unique data \cite{Muennighoff2023, Goyal2024datafiltering}, as well as the growing cost of achieving order-of-magnitude improvements \cite{Cottier2024}. In contrast, scaling inference offers a more sample-efficient axis of improvement \cite{Muennighoff2025, Wang2025RLoneexample}. Still, while scaling inference emerged naturally in domains like games, training has historically remained the dominant paradigm for improving LLM performance.

A useful distinction in this context is between sequential  and parallel inference scaling. Sequential inference scaling involves deepening the model’s reasoning within a single inference trace (e.g., CoT or iterative refinement), often involving longer or more structured reasoning traces. In contrast, parallel inference scaling involves generating multiple independent samples (e.g., best-of-$N$ (Bo$N$), which necessitates selecting among outputs, or majority voting). Hybrid approaches, such as tree-of-thought (ToT), combine elements of both---replicating and refining multiple reasoning paths. Their tradeoffs and implementation barriers are distinct, and the most effective balance between them depends on the base model capabilities, task difficulty, and computational constraints.

The historical focus on training stems from both technical and economic factors. Applying conventional inference scaling techniques to LLMs poses several challenges. Earlier base models had limited generative quality, making it more effective to invest compute in producing better outputs than in generating a greater number of weak ones. In addition, a range of challenges persist for sequential reasoning more broadly---including context length limitation, difficulties with attention mechanisms over extended contexts, and susceptibility to bias \cite{Turpin2023, Saparov2023, Xiao2024efficientstreaming, Yan2025attentionCoT, Zheng2025cursecot, Wu2025moreisless}. Meanwhile, for many past LLM applications, it was not significantly harder to generate a correct answer than to verify one. Additionally, the search space in games is far more constrained than in open-ended token generation \cite{DeepSeekAI2025}. While training explicit verifiers has shown promise under both outcome and process supervision \cite{Lightman2023}, these approaches often suffer from overfitting and high computational costs \cite{DeepSeekAI2025, Cobbe2021, Singhi2025solveverifycomputeoptimal}. On the economic side, inference costs are substantial and recurrent, whereas training is a capital expenditure---often making the latter more compatible with the incentive structures of industry leaders. Moreover, in most early LLM applications, users expected near-instantaneous responses, further discouraging deeper or longer inference strategies. While technically simple methods such as majority voting (MV) \cite{Wang2022self, Lewkowycz2022} or iterative refinement \cite{Madaan2023} have been technically available, the marginal accuracy gains they offer have often failed to justify their computational costs \cite{Erol2025economicLLM}.

Yet, recent developments have reshaped the viability of inference scaling. Advances in base model quality, long-context handling, and reasoning stability have made generated outputs more coherent and reliable, increasing the effectiveness of inference-based strategies. In parallel, the complexity of target tasks---such as scientific reasoning, coding, and long-horizon planning---has grown, widening the gap between generation and verification. These tasks also inherently benefit from multi-step reasoning, and users are often more willing to tolerate slower inference in exchange for higher-quality responses. Modern LLMs now exhibit more structured generation, enabling more focused search over output spaces, akin to policy-value decompositions seen in AlphaGo \cite{Silver2016AlphaGo}. Reinforcement learning techniques have further enabled unified architectures---reminiscent of AlphaGo Zero \cite{Silver2017AlphaGoZero}---that support both generation and evaluation within the same model, eliminating the need to train a separate verifier network. Collectively, these shifts have opened the door to inference scaling as a practical and increasingly central tool for improving model performance. These developments in inference scaling require a reevaluation of what it means for a system to be compute-optimal.

\subsection{Compute-Optimality}
Compute-optimal scaling laws describe how model components---such as parameters, training dataset size, and inference tokens---should be allocated under a fixed compute budget to maximize performance. These laws\footnote{Alternative scaling laws (e.g., transfer learning \cite{Hernandez2021TransferScaling}, contrastive learning \cite{Cherti2023constrastivescaling}, data-curation \cite{Goyal2024datafiltering}, distillation \cite{Busbridge2025}, mixture of experts \cite{Krajewski2024scalinglawsMOE}) suggest that scaling strategies extend well beyond dataset and parameter size, with benefits to tailoring for a specific task and learning framework.} constitute forms of allometric scaling, analogous to the systematic scaling relationships observed in biological systems \cite{Thompson1917, Haldane1926, Manger2013}. In biology, evolutionary pressures produced divergent scaling laws for brain-to-body mass relationships among hominids, non-human primates, and other mammals \cite{Manger2013}. Analogously, scaling laws for LLMs are not static: they evolve \cite{Sutskever2024NeurIPSTalk} due to a variety of factors (e.g., differing incentive structures and applications, hardware efficiency \cite{Sardana2024}, data efficiency \cite{DeepSeekAI2024llmscaling} and algorithmic variations \cite{Chen2025}). In particular, the rise of inference scaling exerts novel pressures, causing deviations from previously established compute-optimal trajectories.

\subsubsection{Training Compute}
Prior to the rise of inference scaling, training compute dominated discussions of performance optimization. Foundational studies proposed scaling laws that minimize an empirically-fitted pretraining loss \eqref{eqn:pretrainingloss}, subject to a fixed training compute budget \cite{Kaplan2020, Hoffmann2022}. For large transformer architectures, the cost\footnote{This assumes a fully connected network. For mixture of expert models, the forward pass is only over a subset of the parameters. We do not explicitly consider these models here, leaving such consideration for future work.} of a single forward pass is:
\begin{align}
C_{\text{forward}}^{\text{full}} &= C_{\text{forward}}^{\text{param}} + C_{\text{forward}}^{\text{attn}} \\
              &\approx 2 N + 2 n_{\text{layers}} d_{\text{attn}} n_{\text{ctx}}\label{eq:forwardcost}
\end{align}
where $n_{\text{layers}}$ is the number of layers, $d_{\text{attn}}$ is the dimension of the attention output\footnote{We assume single head attention throughout.}, and $n_{ctx}$ is the context length. The backward pass used during training is roughly twice the cost of the forward pass \cite{Kaplan2020}. Assuming training proceeds over fixed-length sequences of size $\tau_{\text{seq}}$ using autoregressive masking, this yields the training compute cost measured, e.g.\ in floating point operations (FLOPs):
\begin{align}
C_{\text{tr}} &= C_{\text{tr}}^{\text{param}} + C_{\text{tr}}^{\text{attn}} \\
              &\approx6 N D_{\text{tr}} + 3n_{\text{layers}} d_{\text{attn}}\tau_{\text{seq}}D_{\text{tr}}\approx6 N D_{\text{tr}}.
\end{align}
The parameter training cost is generally considered to dominate the attention cost, as short contexts are used for most of training and, in typical regimes, $N \gg n_{\text{layer}} \cdot d_{\text{attn}} \cdot \tau_{\text{seq}}$ \cite{Chen2025, Kaplan2020}.

This constraint became canonical in the Chinchilla framework, which advocates training smaller models on more data---than was previously common practice---to achieve superior performance \cite{Hoffmann2022}. Empirical findings, as in Fig.~\ref{fig:ChinEmergPlat}.1a, and shown by theoretical work given in Fig.~\ref{fig:ChinEmergPlat}.1b \cite{Nayak2025}, suggest that the Chinchilla-optimal choices of parameter size, $N_{\text{chin}}$, and number of training tokens, $D_{\text{chin}}$, scale proportionally according to a fixed ratio $\kappa$ \cite{Hoffmann2022, Anil2023, Besiroglu2024}\footnote{
From empirical fits, $\alpha\approx\beta\approx\frac{1}{3}$ and thus $\frac{\alpha}{\alpha+\beta}\approx\frac{\beta}{\alpha+\beta}\approx\frac{1}{2}$. Eq.~\eqref{eq:kappachinchilla} is the approximate solution to the constrained minimization of \eqref{eqn:pretrainingloss}, which yields:
\begin{equation}\label{eq:chinchilla}
    \kappa \coloneqq \left(\frac{\beta B}{\alpha A}\right)^{1/\alpha}, \ 
    N_{\text{chin}} = \kappa^{-\frac{\alpha}{\alpha+\beta}}\Bigl(\frac{C_{\text{tr}}}{6}\Bigr)^{\frac{\beta}{\alpha+\beta}}, \ D_{\text{chin}}=\Bigl(\kappa \frac{C_{\text{tr}}}{6}\Bigr)^{\frac{\alpha}{\alpha+\beta}}\mbox{.}
\end{equation}
}, i.e.,
\begin{equation} \label{eq:kappachinchilla}
    D_{\text{chin}}=\kappa N_{\text{chin}}.
\end{equation}

The Chinchilla-optimal ratio between training tokens and parameters was estimated to be approximately $20$ \cite{Hoffmann2022}, but subsequent scrutiny of the fitting methodology has suggested values closer to $25$ \cite{Besiroglu2024}. Enhancing data quality may shift the optimal scaling dynamics, favoring a steeper increase in parameter count relative to dataset size \cite{DeepSeekAI2024llmscaling}. In fact, when the unique dataset size is constrained below $D_{\text{chin}}$, repeating data across multiple epochs---up to a point of diminishing returns---yields better performance than significantly overparameterizing the model \cite{Muennighoff2023, Goyal2024datafiltering}.

\subsubsection{Inference Compute}
While these results offer valuable insights into training efficiency, they do not consider inference costs. However, as pretrained models are increasingly deployed across real-world applications, attention has shifted toward inference efficiency as a standalone objective. Given $I$ total tasks, for each $i\in \{1,...,I\}$ the computational cost of a single sequential inference is a forward pass for all $\Omega^{(i)}$ output tokens and the attention cost, and the upfront attention cost on the prompt of size $\Omega_{\text{prompt}}$ tokens\footnote{For a single task $i$, the full inference cost is given by the original prompts non-causal attention cost and the autoregressive attention cost over all forward passes during the output: $\sum_{j=1}^{\Omega}C_{\text{forward}}^{\text{total}} = 2N\Omega+2n_{\text{layers}}d_{\text{attn}}\left(\Omega_{\text{prompt}}^2+ \sum_{j=1}^{\Omega}(\Omega_\text{prompt}+j)\right)\approx 2N\Omega+2n_{\text{layers}}d_{\text{attn}}\left(\Omega_{\text{prompt}}^2+\Omega_\text{prompt}\, \Omega +\tfrac{1}{2}\Omega^2\right).$}:
\begin{align}\label{eq:totalinferencecost}
    C_{\text{inf}}^{\text{full}} &= C_{\text{inf}}^{\text{param}} + C_{\text{inf}}^{\text{attn-prompt}} +C_{\text{inf}}^{\text{attn-output}}\\
        &\approx 2 N \sum_{i=1}^I\Omega^{(i)} + 2n_{\text{layers}} d_{\text{attn}}\sum_{i=1}^I\left[(\Omega_{\text{prompt}}^{(i)})^2+\Omega_{\text{prompt}}^{(i)}\Omega^{(i)}+\frac{1}{2}(\Omega^{(i)})^2\right]\label{eq:FullInferenceCost}\mbox{.}
\end{align}
The attention cost is often disregarded for sufficiently short context lengths and output sequences \cite{Snell2024, Sardana2024}. While we adopt the same simplification in many of our illustrations where it does not affect the qualitative conclusions, this cost becomes important in certain settings. In particular, we later demonstrate its significance when comparing CoT and ToT inference strategies (Section~\ref{subsec:cotvstot}) and discuss its dependence on algorithmic and hardware choices in more detail in Section~\ref{subsec:algorithmsandhardware}.

Recent studies have shown that, under a fixed inference-compute budget, performance can be improved by using smaller pretrained models with more inference tokens, rather than larger models with fewer \cite{Snell2024, Wu2025}. In response, scaling laws have been proposed to jointly optimize inference strategy and model size \cite{Wu2025}. In one specific case---mathematics tasks using a hybrid ToT inference strategy with Chinchilla-pretrained models---the optimal tradeoff between model size and inference budget was empirically found to follow a power law \cite{Wu2025}. However, these studies have focused on inference scaling given pretrained models that are training-compute-optimal. These solutions are suboptimal---in particular, for inference-only users---since smaller, overtrained models can often achieve comparable performance or better performance while incurring lower inference costs as data is not considered \cite{Meta2024llama3, Devries2023, Gadre2024, Sardana2024}. Performance gains have been observed at training regimes beyond 10,000 tokens per parameter \cite{Sardana2024}. Furthermore, models trained in this overtrained regime often generalize predictably to downstream tasks \cite{Gadre2024}, though this behavior is not universally guaranteed \cite{Springer2025overtrainedfinetuning}.

\subsubsection{Total Compute}
Recent work has emphasized the importance of characterizing compute-optimality not only for training but for the full lifecycle of LLMs \cite{Devries2023, Sardana2024}. When accounting for total compute---where model parameters contribute to both training and inference costs---it is more efficient to shift toward smaller models trained on more data \cite{Sardana2024}. This preference for inference-efficient models leads to scaling strategies that deviate from the Chinchilla-optimal ratio, improving both total compute efficiency and ease of deployment \cite{Meta2024llama3}.  Incorporating the inference cost into the objective given a fixed total compute budget $C_{\text{total}}$ and the expected number of inference tokens processed over the model’s lifetime $D_{\text{inf}}=I\cdot\Omega$ gives an approximation of the total compute without attention costs:
\begin{equation}\label{eq:totalcompute}
C_{\text{total}}^{\text{param}} = 6 N D_{\text{tr}} + 2 N I \Omega =6ND_{\text{tr}}+2ND_{\text{inf}} \mbox{.}
\end{equation}
Minimizing the pretraining loss under this constraint has previously been solved numerically \cite{Sardana2024}. Total-compute-optimal allocation studies to date have focused on minimizing pretraining loss, rather than optimizing task performance, and omit explicit inference scaling. In practice, algorithmic choices and hardware characteristics further shape the cost structure, altering the resulting optimal allocations \cite{Snell2024, Sardana2024, Wu2025, Chen2025}. The framework introduced in this work naturally supports joint optimization of model size, training dataset size, and inference strategy to improve performance over a distribution of tasks, under flexibly defined cost constraints.

\subsection{From Empirical Findings to Theoretical Foundations}
Empirical studies have clarified both the promise and limits of inference scaling strategies. Notably, linear accuracy improvements have been observed in reasoning models with logarithmic increases in training compute, inference compute, and inference tokens \cite{OpenAI2024reasoning, Anthropic2025ExtendedThinking, DeepSeekAI2025}. A logarithmic tradeoff between training and inference compute has been observed in games such as Hex, where performance iso-contours indicate that increased inference can compensate for reduced training and vice versa \cite{Jones2021}. Similar tradeoffs have been considered more broadly including in LLMs, where meta-analyses show that increasing inference compute by orders of magnitude can offset training cost---though the optimal ratio varies with model scale and task difficulty \cite{EpochAI2023, Snell2024, Wu2025}. CoT prompting has been shown to induce emergent capabilities in models that otherwise exhibit flattened scaling curves. ToT strategies benefit harder tasks and weaker models more significantly, whereas purely sequential techniques tend to perform better on easier tasks \cite{Snell2024, Wu2025}. While task coverage, measured by pass@k with an oracle verifier, exhibits a power-law decay in failure rate as the number of samples increases, coverage gains from strategies such as MV and Bo$N$ with a reward model tend to saturate \cite{Brown2024repeatedsamplingsscaling, Wu2025}. Separate works have characterized repeated sampling performance in terms of population-level heterogeneity in pass or error rates \cite{Polo2025slothscalinglawsllm}, including models that assume beta-distributed task difficulty \cite{Levi2024simple}. Saturation in MV has also been explained through per-instance convergence dynamics \cite{Wu2025}.

The theory of inference presented in this work recovers the key scaling phenomena for both sequential and parallel inference strategies. It offers a computationally grounded and interpretable framework that aligns with prior empirical and theoretical results while situating such notions of inference scaling into a broader context.

\subsection{Contributions and Outline}
The remainder of the paper explains and analyzes a unified framework of training and inference compute scaling in LLMs. Section~\ref{sec:ds3} introduces \emph{directed stochastic skill search (DS3)}, a formalism that treats inference as structured traversal over a latent skill graph, and examines a concrete instantiation under simplifying assumptions. This yields closed-form expressions for expected task accuracy and inference cost under both sequential and parallel strategies, enabling principled analysis of compute–performance tradeoffs. While DS3 is initially introduced as a theoretical framework, its structure is later validated through empirical alignment with observed trends. We then compare CoT and ToT inference strategies, examining how their relative effectiveness varies with task difficulty and demonstrating how context length constraints influence their comparative efficiency.
Section~\ref{sec:unification} presents the results on unifying inference with training. Section~\ref{subsec:traininggraph} extends a hierarchical skill–text tripartite formulation to represent the model’s pretrained skill graph, allowing for joint analysis of training and inference from first principles. This construction recovers empirically observed phenomena under CoT prompting, including linear accuracy gains with logarithmic increases in training compute, inference compute, and inference tokens. Section~\ref{subsec:twostep} shows how the framework interfaces with empirical two-step forecasting methods that map pretraining loss to downstream task accuracy. The suboptimality of limiting design to Chinchilla-style pretraining is illustrated using a hypothetical example, why inference can elicit emergent capabilities is explained, and inference scaling in the unbounded data regime is briefly considered.
Section~\ref{subsec:beta} assumes a distribution directly over task accuracies to analyze repeated sampling strategies such as pass@N and MV. These results match empirical trends and offer an interpretation of population-level inference behavior. A preliminary discussion of how one may back out a model’s intrinsic inference capabilities from such task-level empirical fits is also included.
Section~\ref{sec:implications} discusses broader implications of the framework. Section~\ref{subsec:energy} examines inference-specific energy considerations and conjectures possible strategies to improve sustainable design. Section~\ref{subsec:algorithmsandhardware} outlines future algorithmic and hardware considerations, including unconventional inference scaling methods. Section~\ref{subsec:forecasting} explores implications for capability forecasting and Section~\ref{subsec:safety} reflects briefly on a few security considerations. Section~\ref{sec:conclusion} concludes with open directions and final remarks.

\section{Directed Stochastic Skill Search}\label{sec:ds3}
We refer to the following general inference formalism, along with its concrete instantiations under various representational assumptions, as directed stochastic skill search (DS3). At its core, DS3 treats inference as a dynamic ensemble of sequential traces, where each trace $i$ represents a single walk through a latent skill graph governed by stochastic transitions and outcomes. Formally, each trace $i$ unfolds on a directed graph $\mathcal{G}^{(i)} = (\mathcal{V}^{(i)}, \mathcal{E}^{(i)})$, which may differ across traces. That is, the graph structure itself---its set of vertices and edges---is indexed by $i$, allowing heterogeneous inference topologies even within a shared ensemble. The vertex set $\mathcal{V}^{(i)} = \mathcal{S}^{(i)} \cup \{\textsc{branch}^{(i)}\} \cup \{\textsc{idle}^{(i)}\} \cup \{\textsc{stop}^{(i)}\}$ consists of a (finite or countably infinite) set of latent skills $\mathcal{S}^{(i)}$ learned during training, along with distinguished control nodes: $\textsc{stop}^{(i)}$, an absorbing vertex that terminates the trace; $\textsc{idle}^{(i)}$, which allows the trace to persist without traversal; and $\textsc{branch}^{(i)}$, which initiates the creation of one or more new traces. The edge set $\mathcal{E}^{(i)} \subseteq \mathcal{V}^{(i)} \times \mathcal{V}^{(i)}$ defines the allowable transitions between vertices, including any self-loops and transitions into the distinguished nodes.

The process begins from an initialization $\Theta_0^{(i)}$, which encodes prior information available before inference begins---such as the initial prompting, interaction history, or previously-generated context. At each discrete step $t \in \{1, 2, \dots\}$, a trace is located at node $X_t^{(i)}\in \mathcal{V}^{(i)}$ and emits an outcome $O_t^{(i)}$ from an outcome space $\mathcal{O}^{(i)}$, which may correspond to generated tokens, latent updates, or decisions. The corresponding sequences of visited nodes and emitted outcomes are denoted by $\mathbf{X}_t^{(i)} = (X_1^{(i)}, \dots, X_t^{(i)})$ and $\mathbf{O}_t^{(i)} = (O_1^{(i)}, \dots, O_t^{(i)})$, respectively. We define the history of trace $i$ at time $t$ as $\mathbf{H}_t^{(i)} = (\Theta_0^{(i)}, \mathbf{X}_t^{(i)}, \mathbf{O}_t^{(i)})$, which captures all information observed along that trace up to time $t$. The tuple of all such history across the ensemble of all $N_{t}$ traces instantiated up to time $t$ is denoted by $\mathcal{H}_t = (\mathbf{H}_t^{(i)})_{i=1}^{N_{t}}$. At any time $t$ when a trace reaches $\textsc{branch}^{(i)}$, it may fork into new traces that inherit the parent’s history, $\mathbf{H}_t^{(i)}$, at the moment of branching. In particular, for each child trace $i'$, the history is initialized as $\mathbf{H}_t^{(i')}=\mathbf{H}_t^{(i)}$, and the graph structure is cloned: $\mathcal{G}^{(i')}=\mathcal{G}^{(i)}$. Any trace that does not arise through branching is explicitly defined in the framework, with an associated initialization $\Theta_0^{(i)}$ and graph $\mathcal{G}^{(i)}$, even if it begins in $\textsc{idle}^{(i)}$. The tuple of all graphs used in an inference strategy is $\mathfrak{G} = \left( \mathcal{G}^{(i)} \right)_{i=1}^{N_t}$.

We define a dynamic external environment state $\mathfrak{E}_t$, which may influence the inference process in real time, adapting to factors such as user interventions or resource availability (e.g., compute, energy, memory). Inference proceeds with each trace evolving independently or according to a collective policy. The walk is governed by two stochastic components. First, a transition kernel \(\iota_t^{(i)}\) defines the probability of moving to the next node, conditioned on the history tuple and the environment:
\begin{equation}\label{eq:transition_full}
\Pr[X_{t}^{(i)} = v \mid \mathcal{H}_{t-1}, \mathfrak{E}_{t-1}] = \iota_t^{(i)}(v; \mathcal{H}_{t-1}, \mathfrak{E}_{t-1}), \quad v \in \mathcal{V}^{(i)}.
\end{equation}
Second, a conditional distribution \(p_t^{(i)}\) governs outcomes at each step:
\begin{equation}\label{eq:outcome_full}
O_t^{(i)} \sim p_t^{(i)}(\cdot \mid \mathcal{H}_{t}, \mathfrak{E}_{t}).
\end{equation}
Next, we explicitly define the dynamic evaluation and cost sets:
\begin{equation}\label{eq:costandeval_full}
\mathcal{Q}_t = \left\{ q_t^{(j')}(\mathcal{O}_t, \mathfrak{E}_t) \right\}_{j' \in \mathcal{J}'}, \quad \mathcal{C}_t = \left\{ c_t^{(j)}(\mathcal{H}_t, \mathfrak{E}_t) \right\}_{j \in \mathcal{J}}.
\end{equation}
Each index $j' \in \mathcal{J}'$ denotes an evaluation function $q_t^{(j')}$, applied to the tuple of outcomes $\mathcal{O}_t = \left( \mathbf{O}_t^{(i)} \right)_{i=1}^{N_t}$ across all traces. These functions provide semantic or task-level assessments of inference quality---such as correctness, coherence, or goal satisfaction---and may also vary with the environment $\mathfrak{E}_t$. This set can include verifiers, pass@k metrics, or other evaluative procedures. Each index $j \in \mathcal{J}$ corresponds to a cost function $c_t^{(j)}$, which may depend on individual trace histories or aggregate ensemble behavior via $\mathcal{H}_t$, as well as on the external environment $\mathfrak{E}_t$. These costs capture both local metrics, such as per-step compute or latency, and global ones like inter-trace contention or coordination, modulated by real-time constraints such as energy pricing or hardware efficiency.

The termination time $\mathcal{T} \in \mathbb{N}$ is the first time step at which all traces have reached their respective stop nodes:
\begin{equation}\label{eq:termination_full}
\mathcal{T} = \min \left\{ t \in \mathbb{N} \,\middle|\, X_t^{(i)} = \textsc{stop}^{(i)} \text{ for all } i \in \{1, \dots, N_t\} \right\}.
\end{equation}
Inference proceeds over steps $t \in \{1, \dots, \mathcal{T}\}$, after which no further transitions or outcomes occur.

\subsection{Idealized Setup}
\begin{figure}
    \centering
    \includegraphics[width=0.9\textwidth]{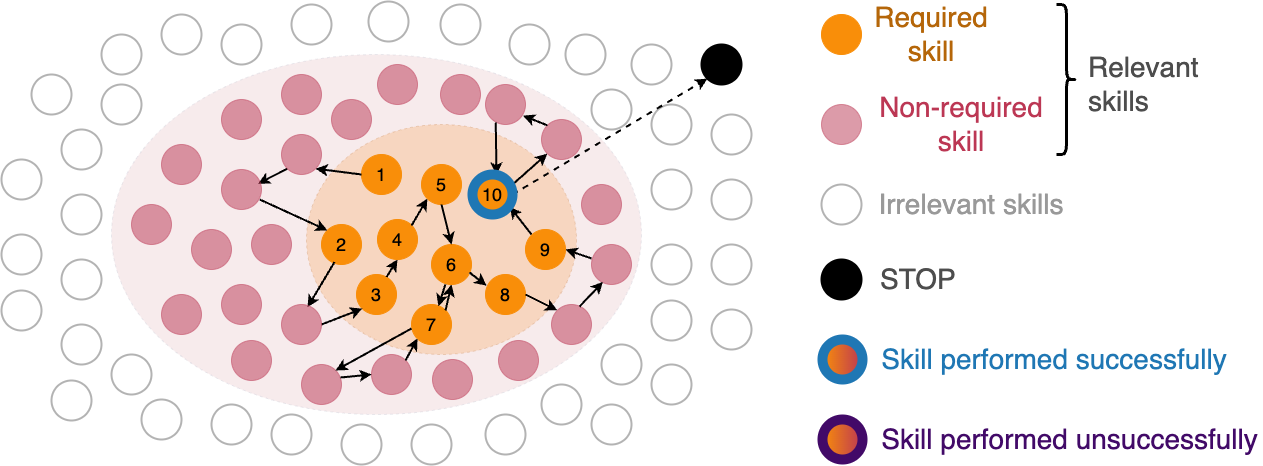}
    \caption{A directed stochastic skill search for a CoT policy, terminating upon completion of the final required skill. The path takes 22 skill steps, each producing \(\omega\) tokens, including detours through non-required and already visited skills. Required skills are shown in orange, with successful completions in blue. Non-required relevant skills are shown in pink and irrelevant skills in white.
    }
    \label{fig:SkillTraversal}
\end{figure}

We analyze four idealized inference strategies: CoT, an instance of ToT we refer to as ToT(1), Bo$N$, and MV. These idealized settings allow us to isolate the core dynamics of inference under a mathematically simple framework which can then be extended to a vast family of inference strategies. Figure~\ref{fig:SkillTraversal} illustrates a representative non-branching trajectory through the skill graph. To facilitate tractable analysis, we impose several structural simplifications on the general DS3 framework. All traces evolve over the same fully connected graph $\mathcal{G} = (\mathcal{V}, \mathcal{E})$ and act on a shared skill set $\mathcal{S}$, partitioned into three disjoint subsets: required skills $\mathcal{S}_{\mathcal{R}^+}$, relevant-but-not-required skills $\mathcal{S}_{\mathcal{R}^-}$, and irrelevant skills $\mathcal{S}_{\mathcal{I}}$. Task completion requires sequential execution of $m$ skills (possibly with repetitions) drawn from $\mathcal{S}_{\mathcal{R}^+}$\footnote{We assume the prompt explicitly elicits the required skills and no skills can be completed out of order. In practice---especially for complex tasks or weaker reasoning models---this may not occur reliably: the model might overlook essential steps, and creativity may benefit from engaging unconventional skills~\cite{Zhou2024architectural,Varshney2013creativity}. In our setup, expanding the relevant set degrades performance, and we assign no cost or benefit to invoking irrelevant skills. Some skills may be strictly learned or unlearned (e.g., internal knowledge retrieval), or inherently stochastic (e.g., generating randomness). Although randomness degrades performance here, stochasticity has been linked to improved creative exploration~\cite{Vempaty2017couponcollector}. Exploring the limits of reasoning and refining these assumptions is a direction for future work.\label{fn:randomness}}. The complete relevant set is denoted by $\mathcal{S}_{\mathcal{R}} = \mathcal{S}_{\mathcal{R}^+} \cup \mathcal{S}_{\mathcal{R}^-}$, with cardinality $M = |\mathcal{S}_{\mathcal{R}}|$.

We introduce a directionality coefficient $\delta \in [0, 1]$, where $\delta = 0$ corresponds to a fully undirected random walk and $\delta = 1$ corresponds to a perfectly directed walk toward the next required skill. Let $s^*$ denote the next required skill in the task sequence. We define $\mathbbm{1}_{\textsc{stop}}(v_{t-1}) = 1$ if one of the following conditions holds at step $t-1$: (i) all $m$ required skills have been completed in the correct order; (ii) the inference budget is exhausted\footnote{We do not account for early termination due to incorrect answers.}; or (iii) the current trace was not selected following a $\textsc{branch}$ and the outcome of the subsequent step. Otherwise, $\mathbbm{1}_{\textsc{stop}}(v_{t-1}) = 0$. We also define $\mathbbm{1}_{\textsc{branch}}(v_{t-1}) = 1$ if the trace at step $t-1$ initiates branching (i.e., under ToT(1)), and 0 otherwise. Transitions occur between vertices $v_{t-1},\, v_t \in \mathcal{S}_{\mathcal{R}} \cup {\{\textsc{branch}\}}\cup {\{\textsc{stop}\}}$, and we define the corresponding transition kernel as:
\begin{equation}\label{eq:transition}
    \iota(v_t \mid v_{t-1}) =
\begin{cases}
1, & \text{if } v_t = \textsc{stop},\, \mathbbm{1}_{\textsc{stop}}(v_{t-1}) = 1, \\
1, & \text{if } v_t = \textsc{branch},\, \mathbbm{1}_{\textsc{branch}}(v_{t-1}) = 1, \\
\delta + (1 - \delta)\frac{1}{M}, & \text{if } v_t = s^*,\, \mathbbm{1}_{\textsc{branch}}(v_{t-1}) = 0,\, \mathbbm{1}_{\textsc{stop}}(v_{t-1}) = 0, \\
(1 - \delta)\frac{1}{M}, & \text{if } v_t \in \mathcal{S}_{\mathcal{R}} \setminus \{s^*\},\, \mathbbm{1}_{\textsc{branch}}(v_{t-1}) = 0,\, \mathbbm{1}_{\textsc{stop}}(v_{t-1}) = 0, \\
0, & \text{otherwise.}
\end{cases}
\end{equation}
Then, $\hat{\iota} = \delta + (1 - \delta)\frac{1}{M} \in \left[\frac{1}{M}, 1\right]$, hereafter referred to as the directionality, quantifies the model’s ability to traverse its latent skill space effectively. A higher value of $\hat{\iota}$ indicates a stronger bias toward the correct next step in the skill sequence.

We define a simplified outcome space $\mathcal{O} = \{(0, \omega), (1, \omega), \varnothing\}$, where each skill application emits a constant $\omega$ tokens and returns a binary outcome: success $(1, \omega)$ or failure $(0, \omega)$ in applying the next required skill $s^*$. The symbol $\varnothing$ denotes a null outcome emitted at control nodes such as $\textsc{branch}$ or $\textsc{stop}$. The corresponding observation kernel is defined as:
\begin{equation}\label{eq:outcome}
    p(O_t \mid v_t) =
    \begin{cases}
        p, & \text{if } O_t = (1, \omega),\, v_t =s^*, \\
        1 - p, & \text{if } O_t = (0, \omega),\, v_t=s^*, \\
        1, & \text{if } O_t = (0, \omega),\, v_t \in \mathcal{S}_{\mathcal{R}}, \\
        1, & \text{if } O_t = \varnothing , \, v_t \in \{\textsc{stop}, \textsc{branch}\}, \\
        0, & \text{otherwise}.
    \end{cases}
\end{equation}
Here, $p \in [0,1]$ denotes the fixed probability of successfully applying any visited skill. We assume $p$ is uniform across all skills in $\mathcal{S}_{\mathcal{R}}$\footref{fn:randomness} and independent of the trace history. The perfect memory assumption benefits inference strategies with longer sequential reasoning chains by ensuring that intermediate failures or detours do not erase prior progress. In contrast, when incorporating penalties for longer chains (e.g., through context length constraints), our assumption of strictly sequential completion can disproportionately hinder sequential strategies: a required skill contributes to task completion only if it is executed at the correct position in the sequence.
However, we will revisit the impacts of memory and context length when comparing sequential and parallel strategies. For clarity of analysis, we also do not impose any maximum context length in this idealized framework.

We consider compute as the sole cost metric, defined as the cumulative sum of per-trace compute costs. We omit batching effects, memory usage, and hardware variability for simplicity. We further assume that verification is handled by oracle verifiers, which always select the best available trace. This removes the need to represent verification effort explicitly: evaluation reduces to scanning all traces and selecting the one with the highest cumulative outcome score. Let $q_t(\mathbf{O}_t)$ denote the number of successful outcomes by time $t$, defined as: $q_t(\mathbf{O}_{t}) = \sum_{t'=1}^{t} \mathbbm{1}\{O_t=(1,\omega)\}$. We denote the prompt token budget by $\Omega_\text{prompt}$ and take it as the initialization $\Theta_0^{(i)} = \Omega_\text{prompt}$ for all $N_1$ active traces at $t = 1$. The context length at time $t$ for trace $i$ is $n_{\text{ctx}}^{(i)}(\Omega_0,\mathbf{O}_{t-1}^{(i)})=\Omega_0+\sum_{t'=1}^{t-1}\omega\cdot\mathbbm{1}\{O_t^{(i)} \in \{(0,\omega),(1,\omega)\}$, which accumulates over all emitted tokens in the trace history. As final components of the setup, we define the evaluation set and the total inference cost:
\begin{equation}\label{eq:costandeval}
    \mathcal{Q}_t = \left\{ q_t^{(i)} := q_t(\mathbf{O}_t^{(i)}) \right\}_{i=1}^{N_t}, \quad C_{\text{inf}}=\sum_{t=1}^{\mathcal{T}}\sum_{i=1}^{N_t}\sum_{u=1}^{\omega}C_{\text{forward}}\left( n_{\text{ctx}}^{(i)}(\Omega_\text{prompt},\mathbf{O}_{u-1}^{(i)}) \right).
\end{equation}
Here, each $q_t^{(i)}$ denotes the cumulative number of required skills correctly applied by trace $i$ up to time $t$, and $C_{\text{inf}}$ is the total inference cost, aggregated over all time steps, active traces, and token-level forward passes. The per-step context length $n_{\text{ctx}}^{(i)}$ is passed as input to the forward cost function $C_{\text{forward}}$ (e.g., \eqref{eq:forwardcost}).

Table~\ref{tab:inferencestrategies} summarizes the accuracy and compute costs for the inference strategies examined in this work.

\begin{table}[ht]
\centering
\begin{threeparttable}\label{tab:inferencestrategies}
\caption{Comparison of Various Inference Strategies\\Chain‐of‐Thought (CoT), Tree‐of‐Thought(1) (ToT(1)), Best-of-$N$ (Bo$N$), and Majority Voting (MV)}
\begin{tabular}{lccccc}
\toprule
 & Task Success Probability ($\psi$)
 & $C_{\mathrm{inf}}^{\mathrm{param}}$ 
 & $C_{\mathrm{inf}}^{\mathrm{attn}}$ 
 & $C_{\mathrm{inf}}^{\mathrm{eval}}$ 
 & $\mathbb{E}[\Omega'\,|\,\Omega]$\\
\midrule

CoT
& $\displaystyle I_{\hat{\iota}p}\bigl(m,\;\tfrac{\Omega}{\omega}-m+1\bigr)$
& $\displaystyle 2N\,\Omega$
& $\displaystyle 2\,n_{\mathrm{layer}}\,d_{\mathrm{attn}}
   \Bigl(\Omega_{\mathrm{prompt}}^2 + \Omega_{\mathrm{prompt}}\,\Omega
   + \tfrac{1}{2}\Omega^2\Bigr)$
& No
& $\min \{\omega\tfrac{m}{\hat{\iota}p}, \Omega\}$\\

ToT(1)
& $\displaystyle I_{\,1-(1-\hat{\iota}p)^b}\bigl(m,\;\tfrac{\Omega}{b\omega}-m+1\bigr)$
& $\displaystyle 2N\,\Omega$
& $\displaystyle 2\,n_{\mathrm{layer}}\,d_{\mathrm{attn}}
   \Bigl(\Omega_{\mathrm{prompt}}^2 + \,\Omega_{\mathrm{prompt}}\,\Omega
   + \tfrac{1}{2b}\Omega^2\Bigr)$
& Yes
& $\min \{\omega\tfrac{bm}{1-(1-\hat{\iota}p)^b}, \Omega\}$ \\

Bo$N$
& $\displaystyle 1-\left(1-I_{\hat{\iota}p}\bigl(m,\;\tfrac{\Omega}{k\omega}-m+1\bigr)\right)^k $
& $\displaystyle 2N\,\Omega$
& $\displaystyle 2\,n_{\mathrm{layer}}\,d_{\mathrm{attn}}
   \Bigl(k\Omega_{\mathrm{prompt}}^2 + \,\Omega_{\mathrm{prompt}}\,\Omega
   + \tfrac{1}{2k}\Omega^2\Bigr)$
& Yes
& $\min \{\omega\tfrac{km}{\hat{\iota}p}, \Omega\}$ \\

MV
& $\displaystyle \mathbb{E}\bigl[ \tfrac{\mathbbm{1}\{Y_0=\max_{j}Y_j\}}{1+\sum_{j=1}^{\mathcal{Y}}\mathbbm{1}\{Y_0=Y_j\}} \bigr] $
& $\displaystyle 2N\,\Omega$
& $\displaystyle 2\,n_{\mathrm{layer}}\,d_{\mathrm{attn}}
   \Bigl(k\Omega_{\mathrm{0}}^2 + \Omega_{\mathrm{0}}\,\Omega
   + \tfrac{1}{2k}\Omega^2\Bigr)$
& Yes
& $\min \{\omega\tfrac{km}{\hat{\iota}p}, \Omega\}$ \\

\bottomrule
\end{tabular}
\begin{tablenotes}
\small
\item[a.] We report the pass@k for Bo$N$ and MV such that the $k$ trials each of length $\Omega/k$ independently process the prompt and proceed through their traces. We use a CoT base policy.
\item[b.] MV is given over a distribution which includes incorrect answers such that $Y \sim \mathrm{Multinomial}\left(k;\, \Pr[Y_0],\, \Pr[Y_1], \ldots, \Pr[Y_{|\mathcal{Y}|}] \right)$ where the first entry is the probability of the correct answer $\Pr[Y_0]=I_{\hat{\iota}p}(m,\frac{\Omega}{k\omega}-m+1)$ and the remaining $\Pr[Y_j]$ describe the mass assigned to incorrect outputs in the answer alphabet.
\item[c.] While we do not explicitly study the costs of evaluation in the idealized setups, methods using parallel inference have an evaluation cost. For ToT(1) and Bo$N$, this could be done by accumulating the logits of the tokens and comparing traces before branching ($C_{\text{inf}}^{\text{eval}}\approx\displaystyle \Omega + \Omega/\omega$) or selecting the final output ($C_{\text{inf}}^{\text{eval}}\approx\displaystyle \Omega + k$), respectively. In practice, a separate verifier network is often used which can incur a more significant cost. For MV, this is a comparison across the branches ($C_{\text{inf}}^{\text{eval}}\approx\displaystyle k$), which is negligible.
\item[d.] For a tighter approximation of the expected tokens used given a max token budget, $\mathbb{E}[\Omega'|\Omega]$, refer to the piecewise linear-Gaussian approximation, \eqref{eq:lineargaussianapprox}.

\end{tablenotes}
\end{threeparttable}
\end{table}

\subsection{CoT}\label{subsec:CoT}

The CoT inference strategy proceeds as a single sequential trace, thereby forming the conceptual basis for all other inference strategies. At each step, the probability of transitioning to the correct next skill is denoted by $\hat{\iota}$, and the probability of correctly executing that skill---conditioned on reaching the correct node---is $p$. The effective per-step probability of successfully selecting and executing the required skill is thus $\hat{\iota} p$. Let $X$ denote the number of steps required to first succeed in this combined action. Then $X$ is a geometric random variable with success parameter $\hat{\iota} p$,
\begin{equation} \label{eq:geometric_hit}
    \Pr(X = t) = \hat{\iota} p (1 - \hat{\iota} p)^{t - 1}, \quad t \geq 1.
\end{equation}

To complete a task consisting of $m$ sequential required skills, each skill must be independently located and successfully executed, requiring $m$ independent geometric successes. The total number of steps $T$ needed to achieve all $m$ successes thus follows a negative binomial distribution $T \sim \mathrm{NB}(m, \hat{\iota} p)$. The probability of completing the task within a given step budget, $T_{\text{max}}=\omega\, \Omega_{\text{budget}}$, where $\Omega_{\text{budget}}$ is the total token budget and $\omega$ is the number of tokens consumed per step, is given by the cumulative distribution function (CDF) of the negative binomial:
\begin{equation}\label{eq:geometric_cdf}
    \psi^{\text{CoT}}=\Pr(T \leq T_{\text{max}}) = I_{\hat{\iota} p}(m, T_{\text{max}} - m + 1),
\end{equation}
where $I_x(a,b)$ denotes the regularized incomplete beta function. The CDF in \eqref{eq:geometric_hit} corresponds to the special case \( m = 1 \):
\begin{equation}\label{eq:oneskillprob}
\Pr(T < T_{\max} \mid m=1)
= \sum_{t=1}^{T_{\max}} \hat{\iota} p \, (1 - \hat{\iota} p)^{t - 1}
= 1 - (1 - \hat{\iota} p)^{T_{\max}}.
\end{equation}
Another special case is when \( T_{\text{max}} = m \), the task succeeds only if all \( m \) steps are correct, yielding:
\begin{equation}\label{eq:budgetstepequality}
    \Pr(T \leq m) = (\hat{\iota} p)^{m}.
\end{equation}

The expected number of steps to complete the task under a capped step budget $T_{\text{max}}$ is:
\begin{equation}\label{eq:sequentialcostfull}
    \mathbb{E}[T\,|\,T_{\text{max}}] = \sum_{t=m}^{T_{\text{max}}} t \left[ I_{\hat{\iota} p}(m,t - m + 1) - I_{\hat{\iota} p}(m,t - m) \right] + T_{\text{max}} \left[1 - I_{\hat{\iota} p}(m,T_{\text{max}} - m + 1)\right].
\end{equation}
In the asymptotic limit $T_{\text{max}} \rightarrow \infty$, the expected number of steps approaches the uncapped mean of \( m \) independent geometric trials:
\begin{equation}
    \lim_{T_{\text{max}} \rightarrow \infty} \mathbb{E}[T\,|\,T_{\text{max}}] = \frac{m}{\hat{\iota} p}.
\end{equation}
When the budget is small relative to the uncapped mean, the task is unlikely to complete, and the inference process consumes the entire step budget:
\begin{equation}
    \mathbb{E}[T \,|\, T_{\text{max}}] = T_{\text{max}} \quad \text{whenever } T_{\text{max}} \ll \frac{m}{\hat{\iota} p}.
\end{equation}

\begin{figure}
  \centering
  \includegraphics[width=1\linewidth]{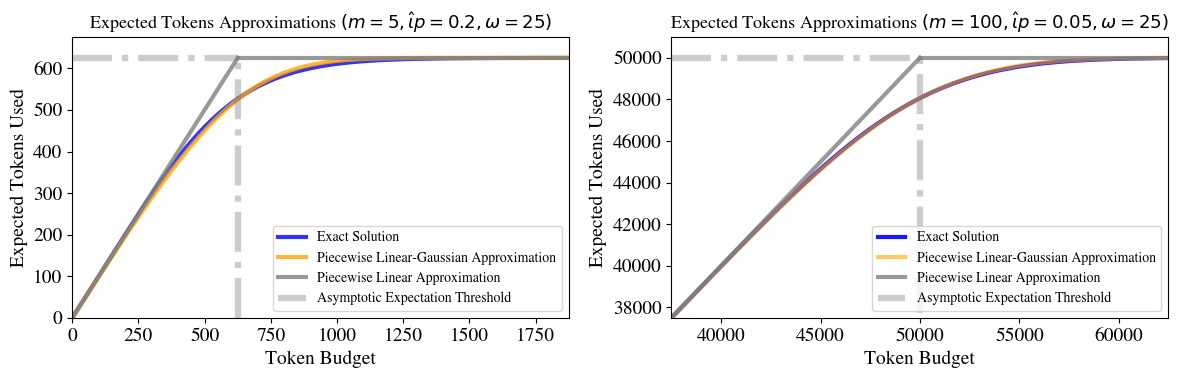}
  \caption{CoT expected token usage under capped inference budgets. Comparison of the exact solution (blue), a piecewise linear-Gaussian approximation (orange), and a simple piecewise linear approximation (gray), for two task difficulty regimes showing approximations preserve key trends across scale. Left: \(m = 5\), \(\hat{\iota} p = 0.2\), \(\omega = 25\); Right: \(m = 100\), \(\hat{\iota} p = 0.05\), \(\omega = 25\). The asymptotic expectation threshold \(\mu = \omega m / (\hat{\iota} p)\) is marked by a gray dashed line.
}
  \label{fig:expected_steps_comparison}
\end{figure}

Due to the computational challenge of evaluating the full closed-form summation in \eqref{eq:sequentialcostfull} at large inference budgets, it is useful to introduce tractable approximations. The expected cost exhibits simple limiting behavior in both the low- and high-budget regimes, and the distribution over required steps is well-approximated by linear behavior in the tails as illustrated in Fig.~\ref{fig:expected_steps_comparison}. In the central region---corresponding to moderate deviations from the mean under a standardized \( z \)-scoring---a Gaussian approximation to the negative binomial provides a smooth interpolation. Substituting \( T_{\text{max}} = \Omega_{\text{budget}} / \omega \) and defining the expected token output as \( \Omega_{\text{CoT}} := \omega \, \mathbb{E}[T \mid T_{\text{max}}] \) yields the following expression:
\begin{equation}\label{eq:lineargaussianapprox}
    \Omega_{\text{CoT}} \approx
\begin{cases}
\Omega_{\text{budget}}, & z < -z^*,\\[6pt]
\omega\frac{m}{\hat{\iota} p}\,\Phi(z) + \Omega_{\text{budget}}\left[1 - \Phi(z)\right] - \omega \frac{\sqrt{m(1 - \hat{\iota} p)}}{\hat{\iota} p}\varphi(z), & -z^* \le z \le z^*,\\[6pt]
\omega\frac{m}{\hat{\iota} p}, & z > z^*,
\end{cases}
\quad\text{where}\quad z = \frac{\Omega_{\text{budget}} - \omega\, m/(\hat{\iota} p)}{\omega\sqrt{m(1 - \hat{\iota} p)}/(\hat{\iota} p)}.
\end{equation}
Here, \( \Phi(\cdot) \) and \( \varphi(\cdot) \) denote the standard normal CDF and PDF, respectively. In Fig.~\ref{fig:expected_steps_comparison}, we set \( z^* = 2 \) to define the central region in which the Gaussian approximation is applied; outside this interval, the tails are handled explicitly by their limiting behavior.

For analytical convenience, we adopt the following piecewise-linear form:
\begin{equation}
    \Omega_{\text{CoT}} \approx \min\{\omega\frac{m}{\hat{\iota}p},\Omega_{\text{budget}}\}.
\end{equation}
which captures the same qualitative behavior---linear growth followed by saturation---while remaining continuous, monotonic, and tractable. This surrogate is used in all main results.

Once the form of \( \Omega_{\text{CoT}} \) is specified---whether exact or approximate---it can be substituted directly into the cost function. The total inference cost, as defined in \eqref{eq:costandeval}, is:
\begin{equation}\label{eq:CoTTotalCost}
    C_{\text{inf}}^{\text{CoT}} = 2 N \Omega_{\text{CoT}} + 2\,n_{\mathrm{layer}}\,d_{\mathrm{attn}}
   \Bigl(\Omega_{\mathrm{prompt}}^2 + \Omega_{\mathrm{prompt}}\,\Omega_{\text{CoT}}
   + \tfrac{1}{2}\Omega_{\text{CoT}}^2\Bigr).
\end{equation}

\subsection{ToT(1)}
To consider a hybrid sequential-parallel strategy, we define a simplified version of ToT---ToT(1). In this setting, branching occurs immediately after a skill node is visited, as specified in \eqref{eq:transition}. The oracle evaluator in \eqref{eq:costandeval} determines which trace, if any, transitions to the \textsc{branch} node, while all remaining traces transition to the \textsc{stop} node. If exactly one trace reaches a successful outcome, that trace is selected; otherwise, all but one are terminated arbitrarily.

The effective success probability improves under branching due to parallel exploration. Given \( b \) concurrent branches, the probability of successfully locating and executing a required skill in a single step becomes \( 1 - (1 - \hat{\iota} p)^b \). Consequently, the probability of completing a task that requires \( m \) sequential skills within a total skill step budget of \( T_{\text{max}} \) across all nodes is:
\begin{equation}\label{eq:totprob}
\psi^{\text{ToT(1)}}=P(T \le \lfloor T_{\text{max}}/b \rfloor) = I_{1-(1-\hat{\iota}p)^b}(m, \lfloor T_{\text{max}}/b\rfloor - m + 1).
\end{equation}
This accounts for the fact that each step in ToT(1) involves evaluating $b$ nodes, so the effective number of sequential skill steps available is $\lfloor T_{\text{max}}/b\rfloor$. Using the same piecewise-linear approximation for expected steps as in the CoT analysis, and noting that ToT(1) is effectively incurs the cost of \( b \) concurrent branches each of length $\lfloor T_{\text{max}}/b\rfloor$, the expected number of output tokens is:
\begin{equation}
\Omega_{\text{ToT(1)}} \approx \min \{\omega\tfrac{bm}{1-(1-\hat{\iota}p)^b}, \Omega_{\text{budget}}\},
\end{equation}
and the corresponding expected inference compute is:
\begin{equation}\label{eq:ToTTotalCost}
C_{\text{inf}}^{\text{ToT(1)}} = 2N\,\Omega_{\text{ToT(1)}}+2\,n_{\mathrm{layer}}\,d_{\mathrm{attn}}
   \Bigl(\Omega_{\mathrm{prompt}}^2 + \,\Omega_{\mathrm{prompt}}\,\Omega_{\text{ToT(1)}}
   + \tfrac{1}{2b}\Omega_{\text{ToT(1)}}^2\Bigr).
\end{equation}

\subsection{BoN}
For a Bo$N$ strategy, performance depends on both the base policy and, as with ToT, a verifier. We assume this verifier is an oracle, making the evaluation equivalent to measuring pass@$k$ coverage. To disambiguate from the model parameter size $N$, we denote the number of parallel outputs as $k$. The probability of task success is the probability that at least one of the $k$ independent parallel traces successfully completes the task. Then for a CoT base policy with a total step budget of $T_{\text{max}}$, evenly divided across the $k$ traces, the task success probability is::
\begin{equation}
\psi^{\text{Bo$N$; CoT}}=P(T \le \lfloor T_{\text{max}}/k \rfloor) = 1-(1-I_{\hat{\iota}p}(m, \lfloor T_{\text{max}}/k\rfloor - m+1))^{k}.
\end{equation}
The expected inference cost is \(k\) times the cost of a single CoT trace with budget \(T_{\text{max}}/k\), due to independence. Then, the expected number of output tokens under a total token budget \(\Omega_{\text{budget}}\) is:
\begin{equation}
    \Omega_{\text{Bo$N$;\,CoT}}\approx\min \{\omega\tfrac{km}{\hat{\iota}p}, \Omega_{\text{budget}}\},
\end{equation}
and the inference compute cost is:
\begin{equation}   C_{\text{inf}}^{\text{Bo$N$};\,\text{CoT}}=2N\,\Omega_{\text{Bo$N$;\,CoT}} +2\,n_{\mathrm{layer}}\,d_{\mathrm{attn}}
   \Bigl(k\Omega_{\mathrm{prompt}}^2 + \,\Omega_{\mathrm{prompt}}\,\Omega_{\text{Bo$N$;\,CoT}}
   + \frac{1}{2k}\Omega_{\text{Bo$N$;\,CoT}}^2\Bigr).
\end{equation}

Since Bo$N$ represents a parallel strategy, it can be combined with other base policies or even mixtures\footnote{For $k$ independent base policies $\pi_1,\ldots,\pi_k$ (possibly with repeats), each allocated a token budget $\Omega_{\pi_j}$ inducing success probability $\psi^{\pi_j} = \psi^{\pi_j}(\Omega_{\pi_j})$ and cost $C_{\text{inf}}^{\pi_j}=C_{\text{inf}}^{\pi_j}(\Omega_{\pi_j})$, the Bo$N$ success probability is $\psi^{\text{Bo}N; \pi_1,\ldots,\pi_k} = 1 - \prod_{j=1}^k (1 - \psi^{\pi_j})$, with total cost $C_{\text{inf}}^{\text{Bo}N; \pi_1,\ldots,\pi_k}=\sum_{j=1}^k C_{\text{inf}}^{\pi_j}$.}. For a ToT(1) base policy (omitting floor notation for simplicity), the success probability under total token budget \(\Omega_{\text{budget}}\) is:
\begin{equation}
    \psi^{\text{Bo$N$; ToT(1)}} = 1 - \left( 1 - I_{\,1 - (1 - \hat{\iota} p)^b} \left( m,\, \tfrac{\Omega_{\text{budget}}}{k b \omega} - m + 1 \right) \right)^k,
\end{equation}
with expected token usage approximated as:
\begin{equation}
    \Omega_{\text{Bo$N$;\,ToT(1)}} \approx \min \left\{ \omega \tfrac{k b m}{1 - (1 - \hat{\iota} p)^b},\, \Omega_{\text{budget}} \right\}.
\end{equation}
The expected cost is therefore:
\begin{equation}
    C_{\text{inf}}^{\text{Bo}N;\text{ ToT(1)}} = 2N\,\Omega_{\text{Bo}N;\text{ ToT(1)}}+2\,n_{\mathrm{layer}}\,d_{\mathrm{attn}}
   \Bigl(k\Omega_{\mathrm{prompt}}^2 + \,\Omega_{\mathrm{prompt}}\,\Omega_{\text{Bo}N;\text{ ToT(1)}}
   + \tfrac{1}{2bk}\Omega_{\text{Bo}N;\text{ ToT(1)}}^2\Bigr).
\end{equation}

\subsection{MV}\label{subsec:MV}
In majority voting (MV), also referred to as consensus voting, task success occurs when the correct output receives more votes than any single incorrect alternative. Ties are resolved by randomly selecting among the top-voted options. Let \(\mathcal{Y}\) denote the set of possible outputs, and suppose the output random variable \(Y\) follows a multinomial distribution:
\[
Y \sim \mathrm{Multinomial}\left(k;\, \Pr[Y_0],\, \Pr[Y_1],\, \ldots,\, \Pr[Y_{|\mathcal{Y}|}] \right),
\]
where \(k\) is the number of samples, \(\Pr[Y_0]\) is the probability of generating the correct answer under the base policy, and the remaining \(\Pr[Y_j]\) correspond to incorrect outputs in the answer alphabet. The expected token usage and inference cost are the same as in Bo$N$. The task success probability is given by\footnote{This expectation can be written in closed form for any fixed output distribution, as shown in Section~\ref{subsec:beta}.}:
\begin{equation}
\psi^{\text{MV}} = \mathbb{E}\left[ \frac{\mathbbm{1}\{Y_0 = \max_{j} Y_j\}}{1 + \sum_{j=1}^{|\mathcal{Y}|} \mathbbm{1}\{Y_0 = Y_j\}} \right].
\end{equation}

The preceding analysis establishes a foundation for comparing how structural inference strategies---ranging from sequential to parallel approaches---interface with compute constraints and success probabilities. While we focus on a representative subset of policies, the framework readily extends to a broader space of inference strategies, enabling principled analysis of their computational costs and effectiveness.

\subsection{Balance Between Sequential and Parallel Reasoning}\label{subsec:cotvstot}
\begin{figure}[ht]
    \centering
    \includegraphics[width=0.7\textwidth]{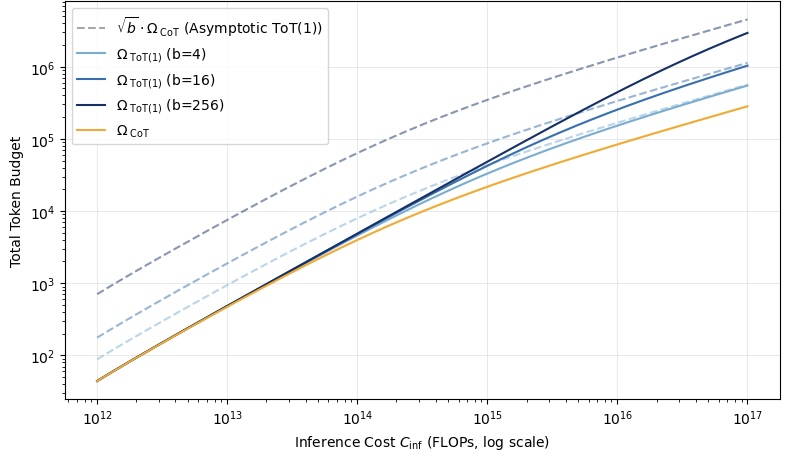}
    \caption{Total token budgets for CoT ($\Omega_{\text{CoT}}$) (orange) and ToT(1) ($\Omega_{\text{ToT(1)}}$) for various branching factors \(b\) (blues). Dashed lines indicate the asymptotic ToT(1) total token count which converges to $\sqrt{b}\cdot \Omega_{\text{ToT(1)}}$. The horizontal axis reflects total compute, which includes parameter, attention, and additional verification (which we take to be logit accumulation) costs. Architectural assumptions are provided in Appendix~\ref{subsec:cottotnumerical}.}
    \label{fig:totaltokenbudget}
\end{figure}

We now examine a case study that illustrates how the structure of inference impacts performance, contrasting CoT with ToT(1) under a fixed compute constraint. Assuming a constant per-step success probability $\hat{\iota}p$ independent of step index $t$, a fully sequential method is preferred when the total number of generated tokens is held fixed. However, the cost structure differs: because attention scales quadratically with context length, parallel strategies can often generate a larger number of tokens for the same total compute\footnote{While not analyzed in detail, real-world parallelism may improve throughput but also introduces overhead from memory reuse, caching, and selection dynamics.}. In the large inference budget regime, the quadratic attention terms dominate the compute-cost expressions (Eqs. \eqref{eq:CoTTotalCost} and \eqref{eq:ToTTotalCost}). Neglecting lower-order terms and equating the two costs therefore yields:
\begin{equation}
\frac{1}{2}\Omega_{\text{CoT}}^2 \approx \frac{1}{2b}\Omega_{\text{ToT(1)}}^2 \quad \Rightarrow \quad \Omega_{\text{ToT(1)}} \approx \sqrt{b} \, \Omega_{\text{CoT}}.
\end{equation}

This scaling relationship demonstrates that ToT(1) can generate more tokens under a fixed compute budget, increasing with the branching factor. We illustrate this tradeoff in Fig.~\ref{fig:totaltokenbudget}.

\begin{figure}[ht]
    \centering
    \includegraphics[width=1\textwidth]{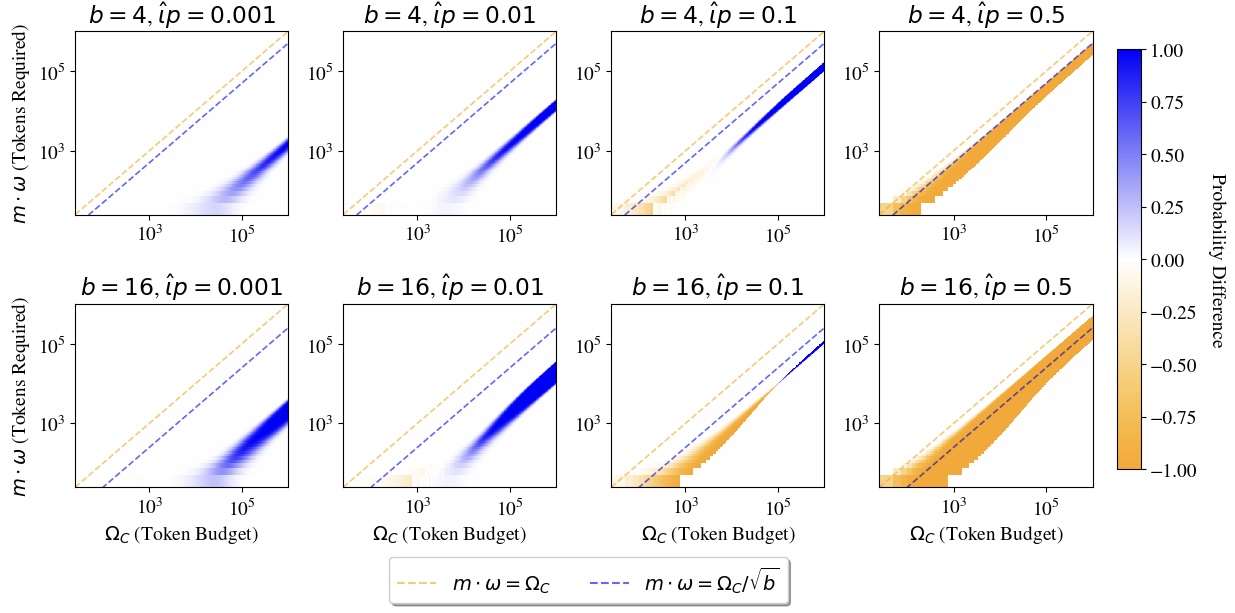}
    \caption{Differences in success probability between CoT and ToT(1) with branching factors 4 and 16. Blue regions indicate ToT(1) advantage; orange regions indicate CoT advantage. Dashed lines represent the token budget required given perfect success rates, which is lower for ToT(1) due to branching. White regions denote near-equal performance (right of the efficiency frontier) or invalid comparisons (left of feasibility bounds). ToT(1) is preferred at low success probabilities \(\hat{\iota}p\), CoT at high probabilities, and the better strategy at intermediate values depends on the number of required skills and branching factor. Architectural assumptions are detailed in Appendix~\ref{subsec:cottotnumerical}.}
    \label{fig:totcotdiff}
\end{figure}

As shown in the forthcoming Thm.~\ref{thm:task-capability}, the preferred inference strategy varies with the task–capability tradeoff, jointly determined by the number of required skills and the difficulty of executing each step. This dependence is visualized in Fig.~\ref{fig:totcotdiff}, where regions of advantage shift with both branching factor and per-step success probability.

\newpage
\noindent Setup:

Let $\psi^{(b)}$ denote the task–success probability under a branching inference policy ToT(1) with branching factor $b$ (\eqref{eq:totprob}). Then for total token budget $\Omega$ and per-step token cost $\omega$:
\[
\psi^{(b)}(\hat{\iota}p,m) = I_{\,1-(1-\hat{\iota}p)^{b}}\!\left(m,\,\tfrac{\Omega}{\sqrt{b}\,\omega}-m+1\right), \qquad b \ge 1,
\]
where $\hat{\iota}p$ is the per-step skill success probability, and $m$ is the number of required sequential skills. The special case $b=1$ corresponds to a purely sequential CoT strategy.

\begin{thm}[Task–Capability Tradeoff between Sequential and Branching Inference]
\label{thm:task-capability}

The preferred inference strategy given by $b$ depends on $\hat{\iota}p$.
\end{thm}

\begin{proof}\label{proof:branchingtaskcapability}
\begin{enumerate}
    \item High-Capability-to-Task Regime ($\hat{\iota}p \to 1^{-}$)
    
    For fixed $m$, we have:
    \[
    \lim_{\hat{\iota}p \to 1^{-}} \psi^{(b)}(\hat{\iota}p,m) = I_{1}\!\left(m,\,\tfrac{\Omega}{\sqrt{b}\,\omega} - m + 1\right) =
    \begin{cases}
        1, & \Omega \ge \sqrt{b}\,m\,\omega, \\
        0, & \text{otherwise}.
    \end{cases}
    \]
    Thus, CoT ($b=1$) minimizes the required number of output tokens and is optimal in the high-capability regime where the model reliably completes each skill step.
    
    \item Low-Capability–to-Task Regime ($\hat{\iota}p \to 0^+$)
    
    Let $\hat{\iota}p = \varepsilon \ll 1$, and assume $\Omega / (\sqrt{b}\,\omega) \gg m$ so that the regularized incomplete beta function lies in the small-$x$, large-$n$ regime:
    \[
    \psi^{(b)}(\varepsilon,m) = I_{\,1 - (1 - \varepsilon)^b}\left(m,\,\tfrac{\Omega}{\sqrt{b}\,\omega} - m + 1\right).
    \]
    
    Hence, the ratio of success probabilities satisfies:
    \[
    \frac{\psi^{(b)}}{\psi^{(1)}} \approx b^{\,m/2} > 1,
    \]
    where the asymptotic derivation is given in Appendix~\ref{ap:proof-low}.
    
    Therefore, under limited model capability (small $\hat{\iota}p$), the branching strategy ToT(1)($b$) provides a multiplicative gain in success probability over CoT, proportional to $b^{m/2}$. This advantage holds whenever the compute budget allows each branch to complete at least $m$ steps.
\end{enumerate}
\end{proof}

\begin{conj}[Task–Capability Tradeoff for Branching Inference: Intermediate Regime and Large-\(m\) Dominance]
\label{conj:intermediate}

Motivated by numerical analysis as demonstrated in Fig.~\ref{fig:totcotdiff}, we conjecture that there exists an optimal branching strategy for ToT(1), including CoT at $b=1$, which jointly depends on the per-step success probability \(\hat{\iota}p\), task length \(m\), branching factor \(b\), and token budget \(\Omega_{\text{budget}}\).

Fix any branching factor \(b\geq 1\) and constant \(c>\sqrt{b}\) (such that the shared budget scales as \(\Omega_{\text{budget}}=c\,m\omega\)).
Then:
\begin{enumerate}
\item For every \(m>1\) the function
      \(\Delta(\varepsilon,m)=\psi^{(b)}-\psi^{(1)}\) is continuous in
      \(\varepsilon\in[0,1]\) and changes sign, guaranteeing a 
      crossover \(\varepsilon_{\!*}(m)\in(0,1)\).
\item Define
      \(
        \varepsilon^{\dagger}=\sup\{\varepsilon:\,
        f_T(\varepsilon)<f_C(\varepsilon)\}\in(0,1).
      \)
      For every fixed \(\varepsilon<\varepsilon^{\dagger}\) there exists
      \(m_0(\varepsilon)\) such that
      \(\psi^{(b)}>\psi^{(1)}\) for all \(m\ge m_0(\varepsilon)\); in
      fact \(\psi^{(b)}-\psi^{(1)}\to1\) as \(m\to\infty\).
\end{enumerate}

CoT is preferred for \(\varepsilon>\varepsilon_{*}(m)\); ToT(1)\((b)\) is preferred for \(\varepsilon<\varepsilon_{*}(m)\), and wins unconditionally once \(m\) is large relative to the fixed \(\varepsilon<\varepsilon^{\dagger}\).
\end{conj}

In practice, even under a fixed token budget and differing compute, sequential approaches are not always preferable to parallel ones \cite{Snell2024, Wu2025}. This phenomenon can be understood within our framework by introducing fatigue or forgetting. Consider a pedagogical example with complete fatigue: for any chain that fails to yield a positive outcome at the first step (i.e., with probability $r_1 = \hat{\iota}_1 p_1$), the success probability immediately collapses to zero on all subsequent steps, $r_{t>1} = 0$, due to catastrophic error. In this case, without branching, the probability of success is exactly $r_1$, since no further steps can compensate for failure. For strategies with $b$ parallel traces, however, the probability becomes $1 - (1 - r_1)^b > r_1$ for all $b > 1$ and $0 < r_1 < 1$, so the preferred strategy is to fully parallelize the step budget. The relative effectiveness of sequential versus parallel inference varies across real systems, shaped by temporal dynamics in reasoning capabilities and task-specific dependencies, highlighting the value of adaptive strategies for improving both efficiency and robustness---particularly when task difficulty (as determined by the $(m, \hat{\iota}p)$ pair) interacts with fatigue or degradation.

Indeed, it has been found more broadly that simpler strategies often outperform more complex reasoning---particularly on easier tasks or those susceptible to bias \cite{Wu2025moreisless, Lee2025compresscot, Snell2024, Wu2025}. Similar effects are observed in humans: individuals with higher numeracy may over-reason and reinforce prior beliefs \cite{Kahan2017}. In LLMs, bias can arise from verifiers \cite{DeepSeekAI2025, Cobbe2021} or from internal representations \cite{Blodgett2020, Bender2021, Mehrabi2022, Zhou2024political}, which we conjecture may lead to information cascades \cite{Bikhchandani1992} that converge on incorrect outcomes. CoT reasoning has been shown to be more bias-prone and to underperform shorter chains in various settings \cite{Turpin2023, Wu2025moreisless, Zheng2025cursecot}. Related observations have been made using rejection sampling methods that favor shorter completions \cite{Muennighoff2025}, though this may reflect the tendency for correct reasoning to terminate early, while incorrect paths are longer and more exploratory, ultimately consuming more of the token budget.

Complemented by experiments, Wu et al.\ introduced a stylized formulation to demonstrate that longer reasoning chains can sometimes degrade performance \cite{Wu2025moreisless}. In their formulation, task success is decomposed into two multiplicative terms: the probability of posing correct sub-questions and the probability of answering them correctly. Specifically, the overall success probability over $T$ steps and task complexity $\ell$ is approximated by
$$
(1 - \sigma' \cdot \ell)^T (1 - E' \cdot \ell/T)^T \alpha,
$$
where $\sigma' \cdot l$ reflects sub-question error, $E'\cdot \ell/T$ captures sub-answer error, and $\alpha$ is a shortcut term representing end-to-end task solvability without decomposition. While highly simplified, this formulation captures an essential tradeoff: deeper decompositions mitigate per-step difficulty but compound error across steps. A similar pattern emerges in a limiting case of the idealized DS3 framework when the inference budget is set equal to the number of steps, $T_{\max} = m$, and skill success probability becomes
$$
\hat{\iota} p = (1 - \sigma' \cdot \ell)(1 - E' \cdot \ell/m) \alpha^{1/m},
$$ 
exhibiting the same tension between solution depth and cumulative risk. This motivates a hierarchical view of task completion, where a model may either solve a problem using a small number of complex skills or a longer sequence of simpler ones. These results underscore the importance of evaluating inference strategies not in isolation, but in the context of how models are trained. The DS3 framework offers an analytical foundation for comparing reasoning strategies---not only in terms of their runtime behavior, but also in how structural priors are shaped during training---a perspective we now turn to.

\section{Unifying Training and Inference Scaling in LLMs}\label{sec:unification}
Gaining a deeper understanding of inference strategies and optimal resource allocation begins with examining the structure of the pretrained latent skill graph that DS3 operates on. To this end, we unify DS3 with a first-principles hierarchical tripartite graph framework and, in parallel, incorporate empirically-grounded methods for characterizing LLM training.

Analyzing a single task in isolation is generally insufficient to capture broader performance trends. We therefore define a task distribution $\phi(\ell, m)$, where $m \in \mathfrak{M}$ denotes the number of required skills and $\ell \in \mathfrak{L}$ parameterizes skill difficulty, following the formulation in \cite{Nayak2025}. Given per-task accuracy $\psi_{\ell,m}$, the expected accuracy over the task distribution is:
\begin{equation}
\Psi \coloneqq \mathbb{E}_{\phi}[\psi_{\ell,m}] =\sum_{\ell \in \mathfrak{L}, m \in \mathfrak{M}} \phi(\ell,m) \cdot \psi_{\ell,m}.
\end{equation}
Likewise the expected inference compute is the weighted sum of individual inference costs:
\begin{equation}
C_{\text{inf}} := \mathbb{E}_\phi[C_{inf}] = \sum_{\ell \in \mathfrak{L}, m \in \mathfrak{M}} \phi(\ell,m) \cdot C_{\text{inf}}^{(\ell,m)}.
\end{equation}

\subsection{Hierarchical Skill-Text Tripartite Graph Framework for LLM Training}\label{subsec:traininggraph}

\begin{figure}
    \centering    \includegraphics[width=0.8\textwidth]{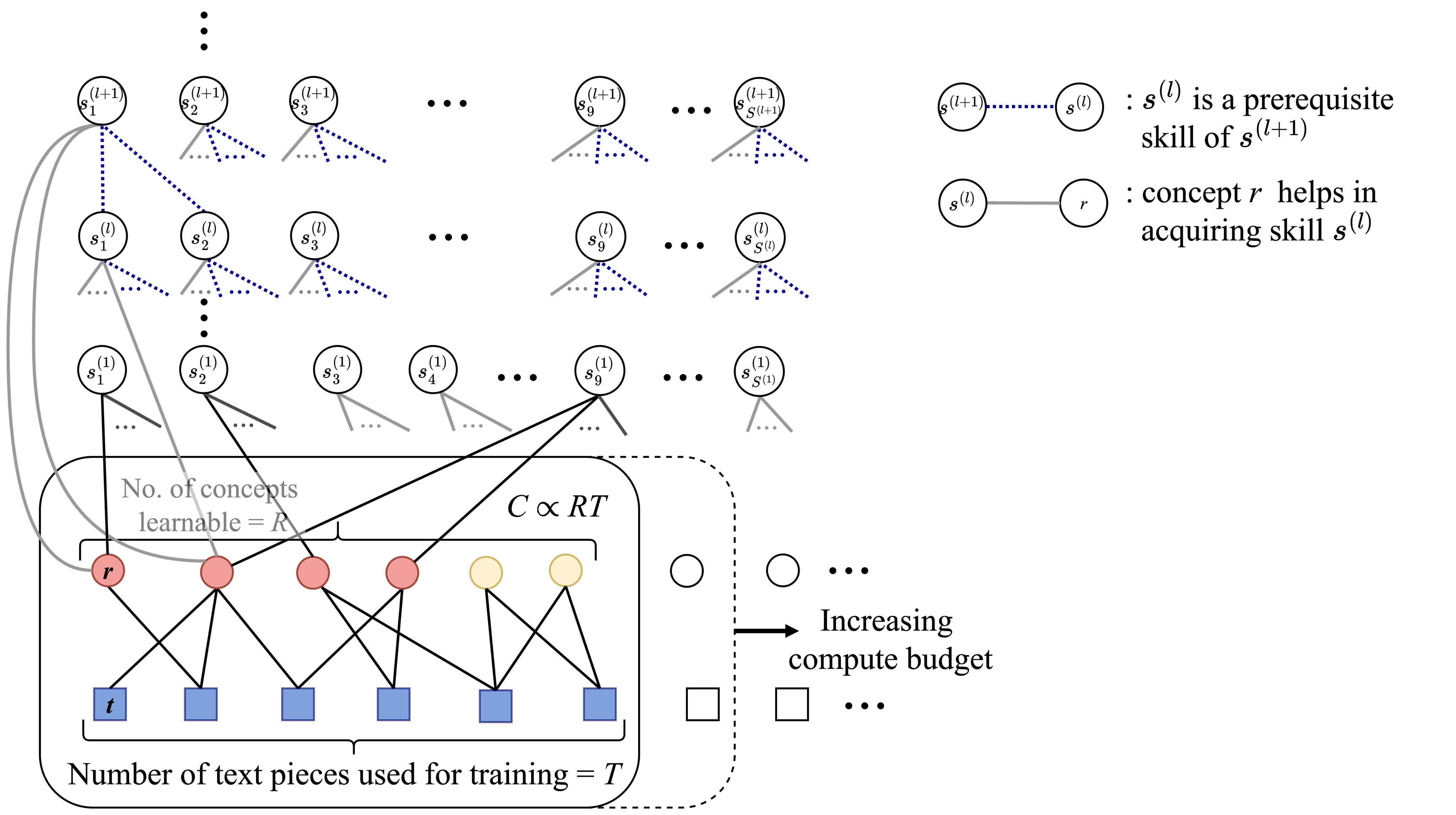}
    \caption{Tripartite graph framework and hierarchical skill graph, adapted from \cite{Nayak2025}. The model is structured as acquiring concepts and skills from text pieces. The lower subgraph depicts a bipartite graph between training text pieces (blue = included, white = unused) and concepts (pink = learned, yellow = learnable, white = unlearnable), with the number of learnable concepts determined by model size. The middle layer connects learned concepts to skills. The upper structure is a hierarchical skill graph, where skills at level \(\ell + 1\) depend on prerequisite skills from level \(\ell\). Training flows bottom-up---from text to concepts, from concepts to skills, and recursively across skill levels---and expands in both depth and width: more text and model parameters allow the model to learn more concepts and acquire a broader set of skills at each level.}
    \label{fig:SkillGraph}
\end{figure}

Consider the hierarchical skill-text tripartite graph framework for LLM training of \cite{Nayak2025} that is shown in Fig.~\ref{fig:SkillGraph}, and that itself builds on \cite{Liao2025semantic,Arora2023, Yu2023skillmix}.  We extend it to support flexible problem-solving strategies, including upward and downward traversal through the skill hierarchy. Moreover, we slightly modify the framework's functional form to better align with empirical trends---specifically, the observation that overtraining a fixed-size model leads to improved but diminishing performance gains. This refinement is especially relevant given inference-time preferences that increasingly favor overtraining. Crucially, these adjustments do not alter the framework's core conclusions, including the Chinchilla-style scaling laws for compute-optimal training, the emergence of capabilities, or the plateauing behavior observed in large-scale models shown in \cite{Nayak2025}.

The prerequisite-based hierarchy introduced in \cite{Nayak2025} models hierarchical skill acquisition by assuming that each skill at level $\ell$ requires $\sigma_\ell$ prerequisite skills from level $\ell-1$. This structure plays a central role in defining how complex capabilities are constructed from simpler components during training. We extend this formulation by allowing tasks to be equivalently solved using skills from higher or lower levels via systematic decomposition (downward traversal) or compression (upward traversal) of skills. For a task that nominally requires $m$ skills at level $\ell$, the number of equivalent skills required at another level $\ell'$ is given by:
\begin{equation}\label{eq:mlprime}
m_{\ell'} =
\begin{cases}
\left\lceil m/\prod_{k=\ell+1}^{\ell'} \sigma_k \right\rceil & \text{if } \ell' > \ell, \\
m & \text{if } \ell' = \ell, \\
\left\lceil m \cdot \prod_{k=l'+1}^{\ell} \sigma_k \right\rceil & \text{if } \ell' < \ell,
\end{cases}
\quad \forall  \ell' \in \mathbb{Z}^{+},
\end{equation}
where we take $\sigma_\ell = \tfrac{1}{2} \ln(\ell)$ to capture an increasing number of prerequisites required at higher layers. This bidirectional traversal capability provides a flexible structural foundation for analyzing the coupling between training and inference strategies. For consistency, we scale the size of the relevant skill set at level $\ell'$ according to the multiplicative rule:
\begin{equation}
    M_{\ell'} = |\mathcal{S}_\mathcal{R}^{(\ell')}| = \left\lceil m_{\ell'}(1+\beta) \right\rceil,
\end{equation}
ensuring that the size of the relevant skill set scales proportionally with the number of skills required at each level. This reflects the intuition that lower layers involve more granular operations and thus require a larger set of candidate skills, while higher layers use more abstract representations and correspondingly smaller sets.

In \cite{Nayak2025}, the composition of skills is determined by their co-occurrence in the training data, mediated by shared concepts and text pieces. However, this formulation fails to capture the empirical trend that performance degrades monotonically as one deviates from the Chinchilla-optimal token-to-parameter ratio $\kappa$. In some cases, the formulation even predicts improved accuracy away from the Chinchilla-optimal $\kappa$, contradicting observed behavior. To address this, we modify the mechanism of skill acquisition to depend directly on the number of learned concepts, without requiring co-occurrence between skills through shared intermediaries. This adjustment ensures that the framework reflects empirical results: for a fixed compute budget, performance declines consistently as one moves away from the Chinchilla-optimal $\kappa$ (see Fig.~\ref{fig:sum_psl_vs_kappa}).

The graph-based framework consists of $R$ concepts, $T$ text pieces, and a hierarchical partitioning of the complete skill set $\mathcal{S}$ into subgraphs $\{\mathcal{S}^{(\ell)}\}_{\ell=1}^L$, where each layer $\ell$ contains $S^{(\ell)}$ skills. Following the knowledge quantization hypothesis \cite{Michaud2023}, the number of learnable concepts $R$ is taken to scale proportionally with the number of parameters $N$, with $N = \zeta R$. Similarly, the number of text pieces $T$ is proportional to the dataset size $D$ (in tokens), with $D = \tau T$. A bipartite graph connects text pieces and concepts, and a second bipartite graph connects concepts to skills, forming a stacked tripartite structure. Given a fixed training compute budget, the goal is to maximize the number of concepts learned from the training data. Drawing from finite blocklength information theory, the probability that a concept is successfully learned from text pieces is given by:
\begin{equation}
    \bar{P_b} = 2 \nu^{*} Q\left(\sqrt{2 R}\frac{(\epsilon^* - \tfrac{1}{2})}{\alpha}\right),
    \label{eqn:ldpc}
\end{equation}
where $\nu^*$, $\epsilon^*$, and $\alpha$ are constants determined by the degree distributions of the text and concept nodes in the bipartite graph, and $Q(\cdot)$ is the complementary Gaussian CDF (see \cite{Nayak2025}). By optimizing over the allocation of parameter count $N$ and dataset size $D$ under a fixed training compute budget $C_{\text{tr}}$, equal scaling of $N$ and $D$ is recovered, and the Chinchilla rule \cite{Hoffmann2022} emerges as training-compute optimal through a finitary analysis \cite{Nayak2025}\footnote{Note that \eqref{eqn:ldpc}, and thus the associated acquisition behavior, does not apply in the overparameterized regime, where the finite blocklength approximation fails. This breakdown is evident in Fig.~\ref{fig:sum_psl_vs_kappa}, which shows a sharp drop in concept learning beyond the Chinchilla-optimal $\kappa$. However, this regime lies outside the scope of current interest.}.

Each concept $r$ is independently connected to a skill $s \in \mathcal{S}^{(\ell)}$ with probability $\xi_\ell = \exp(-\mathfrak{c} \ell / L)$, where $\mathfrak{c} > 0$ is a constant and $L$ is the total number of skill levels. This construction reflects the intuition that advanced skills are associated with fewer concepts than basic ones, making them correspondingly harder to acquire and compose. A skill $s$ at level $\ell$ is considered acquired only if at least $\eta$ of its associated concepts are learned during training.

During pre-training, an LLM acquires concepts from text data, enabling it to build up a hierarchy of skills from basic to advanced. The probability that a skill at level $\ell$ is acquired can thus be lower-bounded as:
\begin{align}
    \Pr\left(s \in \mathcal{S}^{(\ell)} \text{ is acquired}\right)
    &= \Pr\left(\text{at least } \eta \text{ concepts associated with } s \text{ are learned}\right) \nonumber \\
    &\geq \Pr\left(\text{at least } \eta \text{ out of } d^{(\ell)}_\rho \text{ associated concepts are learned}\right),
\end{align}
where $d^{(\ell)}_\rho$ denotes the number of concepts connected to skill $s$, which follows a binomial distribution with mean $\rho R \cdot \xi_\ell$, assuming each concept is learned independently with probability $\rho$.

We upper-bound the number of concepts associated with a skill $s \in \mathcal{S}^{(\ell)}$ using a Chernoff bound. The expected degree of a skill node (i.e., the expected number of connected concepts) is $\xi_\ell R / S^{(\ell)}$, since each of the $R$ concepts connects to each of the $S^{(\ell)}$ skills at level $\ell$ independently with probability $\xi_\ell / S^{(\ell)}$. Applying a Chernoff upper tail bound, we obtain:
\begin{equation}
    d^{(\ell)}_\rho \leq \frac{\xi_\ell R}{S^{(\ell)}} \left( 1 - \sqrt{ \frac{2 S^{(\ell)} \ln(S^{(\ell)} / \rho)}{R \xi_\ell} } \right) := d^{(\ell)}_{\rho, UB},
    \label{eqn:d_rho_UB}
\end{equation}
where $d^{(\ell)}_{\rho, \text{UB}}$ serves as a high-probability upper bound on the number of concepts associated with skill $s$, and $\rho$ is the failure probability. To ensure that the bound is strictly satisfied, we conservatively set $d^{(\ell)}_\rho = 0.99~d^{(\ell)}_{\rho, UB}$. Let $X$ be the number of associated concepts that are successfully learned. Assuming each concept is learned with probability $\bar{P_b}$, the probability that skill $s$ is acquired (i.e., that at least $\eta$ of its associated concepts are learned) satisfies:
\begin{equation}
    \Pr(X \geq \eta)
    \geq  \Pr( \geq\eta \text{ out of } d^{(\ell)}_\rho \text{ concepts are learned})\mbox{.}
\end{equation}

Applying the Chernoff bound from \cite{Nayak2025} to the guaranteed subset of size $d^{(\ell)}_\rho$, we obtain the following lower bound on the probability that a skill is acquired:
\begin{equation}
    \Pr( \geq\eta \text{ out of } d^{(\ell)}_\rho \text{ concepts are learned})
    \geq p_s(\ell) := 
    \begin{cases}
        1-\exp\left(-d^{(\ell)}_\rho D_{KL}\left(\frac{\eta}{d^{(l)}_\rho} \middle\| \bar{P_b}\right)\right) & \text{if } \eta \leq d^{(\ell)}_\rho \bar{P_b}, \\
        \frac{1}{\sqrt{2 d^{(\ell)}_\rho}}\exp\left(-d^{(\ell)}_\rho D_{KL}\left(\frac{\eta}{d^{(\ell)}_\rho} \middle\| \bar{P_b} \right) \right) & \text{if } d^{(\ell)}_\rho \bar{P_b} < \eta < d^{(\ell)}_\rho,\\
        1 & \text{if } \eta \geq d^{(\ell)}_\rho,
    \end{cases}
\end{equation}
which gives a lower bound on the probability that a skill at level $\ell$ is acquired.

Concepts are extracted from text, and skills are acquired hierarchically from these concepts, forming a tripartite structure over text, concepts, and skills. The concept-to-skill mapping is bipartite, with no intra-layer edges. Separately, within each level $\mathcal{S}^{(\ell)}$, we define skill-to-skill connectivity to capture composability. This is distinct from the bipartite structure and applies only within individual levels. Two skills in level $\ell$ can be composed directly if both are acquired and thus all of their respective $\sigma_\ell$ prerequisite skills are also acquired. The corresponding probability of composability is given by:
\begin{equation}
    p_\ell = \Pr(s \in \mathcal{S}^{(\ell)}\text{ is acquired})^2 \Pr(\text{all prerequisites of }s \in\mathcal{S}^{(\ell)}\text{ are acquired})^2 \geq p_s(\ell)^2 \gamma_{\ell-1}^{2 \sigma_\ell}.
\end{equation}
Here, the parameter $\gamma_\ell$ denotes the probability that a skill at level $\ell$ belongs to the giant connected component (GCC) of the corresponding skill layer subgraph. This captures whether skills at that level are globally composable within $\mathcal{S}^{(\ell)}$. Since edges between skills are formed independently, the skill layer subgraph is an Erd\"{o}s-R\'{e}nyi random graph. Then,
\begin{equation}
\gamma_\ell = 1+ \frac{1}{p_\ell S^{(\ell)}}W_0 \left( -p_\ell S^{(\ell)}\exp{\left( -p_\ell S^{(\ell)} \right)} \right),
\end{equation}
where $W_0(\cdot)$ denotes the principal branch of the Lambert $W$ function (see \cite{Nayak2025}). A GCC---and thus the potential for skill composability---emerges in the skill graph at level $\ell$ when the expected degree satisfies $p_\ell S^{(\ell)} > 1$. Below this threshold, the graph consists only of small disconnected clusters; above it, a macroscopic structure forms that enables long-range connectivity across the skill layer. For a given level $\ell$, both $\gamma_\ell$ and $p_\ell$ can be computed recursively, with initial conditions $\gamma_0 = 1$ and $\sigma_1 = 0$. While task-time inference engages only a relevant subset $\mathcal{S}_\mathcal{R}^{(\ell)} \subseteq \mathcal{S}^{(\ell)}$, we assume these are drawn from within the GCC, such that local reasoning inherits global composability. This reflects the premise that training-time exposure enables acquisition of mutually reachable skills, even if only a subset is invoked at inference time.

Then the success probability for a task at level $\ell$ requiring $m$ skills, under a CoT inference strategy with step budget $T_{\text{max}} = \Omega_{\text{budget}}/\omega$, is given by:
\begin{equation}
\psi_{\ell,m}^{\text{CoT}}=\max_{\ell'\in\{1,...,L\}}\gamma_{\ell'}^{m_{\ell'}} \, I_{\hat{\iota}p_s^{(\ell')}}(m_{\ell'},\Omega_{\text{budget}}/\omega-m_{\ell'}+1),
\end{equation}
where $m_{\ell'}$ is defined via the multiple solution strategy rule from the nominal task pair $(\ell, m)$, see \eqref{eq:mlprime}. All $m_\ell'$ skills must be composable within the GCC (with probability $\gamma_{\ell'}^{m_{\ell'}}$) and be successfully located and executed during inference.

Averaging over the task distribution $\phi(\ell, m)$, the expected success probability is:
\begin{equation}
    \Psi^{\text{CoT}}=\sum_{\ell \in \mathfrak{L}, m \in \mathfrak{M}} \phi(\ell,m) \max_{\ell'\in \{1,...,L\}} \gamma_{\ell'}^{m_{\ell'}} \, I_{\hat{\iota}p_s^{(\ell')}}(m_{\ell'},\Omega_{\text{budget}}/\omega-m_{\ell'}+1),
\end{equation}
where $m_{\ell'}$ is defined via the multiple solution strategy rule from the nominal task pair $(\ell, m)$ as in \eqref{eq:mlprime}.

\subsubsection{Accuracy Scaling with Training and Inference Compute}\label{subsec:accuracyscaling}

\begin{figure}[htbp]
    \centering
    \includegraphics[width=0.9\textwidth]{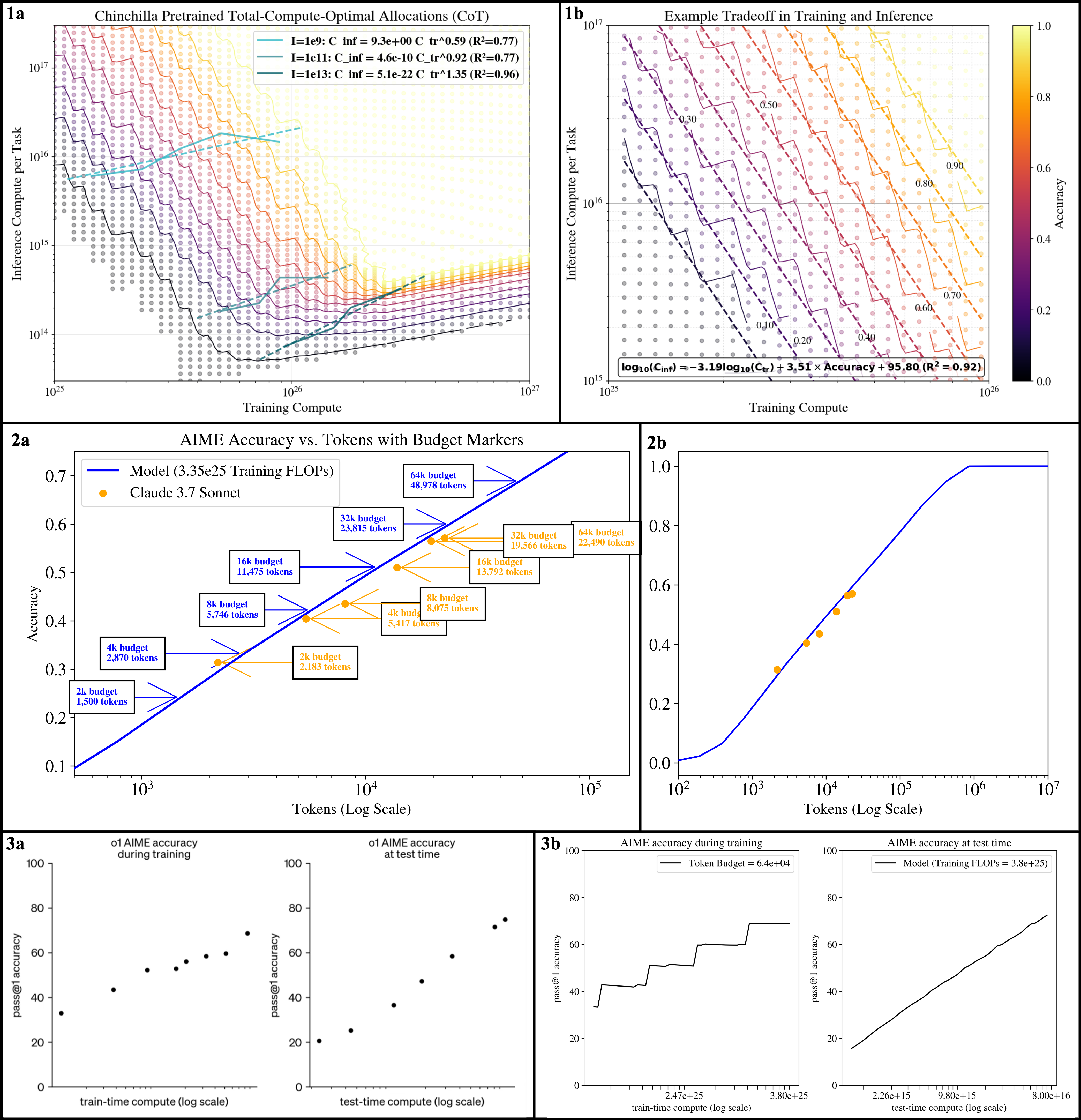}
    \caption{
    1.a. Contour plot of model accuracy under Chinchilla-style pretraining, showing total-compute-optimal allocations for a CoT policy. Parameters are aligned with Claude 3.7 Sonnet \cite{Anthropic2025ExtendedThinking}. Attention costs are not included in inference compute.
    1.b. Approximate log-linear trade-offs between inference and pretraining compute across performance isocontours.  
    2.a. Accuracy scaling of Claude 3.7 on AIME versus log tokens used per task, adapted from \cite{Anthropic2025ExtendedThinking}, using EpochAI-estimated training FLOPs \cite{EpochAI2024Data}.  
    2.b. Same as 2.a, without token budget overlays, extended to broader token regions. 
    3.a. Scaling behavior of OpenAI’s o1 on AIME in both training and inference compute \cite{OpenAI2024o1}.  
    3.b. Reproduction using larger training data estimates \cite{EpochAI2024Data} and 64k token inference budget.
    }\label{fig:AIME}
\end{figure}

To highlight the alignment of our framework with empirical observations, we simulate task performance using a rounded and interpretable set of parameters chosen to reflect plausible values without overfitting. The resulting trends are compared against publicly discussed results\footnote{While Claude 3.7 and o1 may employ mixture-of-experts or non-Chinchilla training regimes, the core log-linear scaling behavior remains. The model class can be adjusted accordingly without altering the qualitative conclusions. Model training FLOPs are estimated \cite{EpochAI2024Data}, and token-to-parameter ratios are assumed to follow Chinchilla-style pretraining.} for Anthropic’s Claude 3.7 Sonnet model \cite{Anthropic2025ExtendedThinking} and OpenAI’s o1 model \cite{OpenAI2024reasoning} on the American Invitational Mathematics Examination (AIME).

Optimal scaling behavior shifts with the number of downstream tasks. As shown in Fig.~\ref{fig:AIME}.1.a, models should allocate proportionally more compute to pretraining when amortized over a larger task load. This shift also alters the total-compute-optimal scaling laws. The plotted regimes reflect realistic, rounded parameters (see Appendix~\ref{ap:claudefit}). Since these are general-purpose LLMs, the precise allocation ultimately depends on the task mix. Fig.~\ref{fig:AIME}.1.b illustrates the fundamental tradeoff: compute can be spent on training or inference, and performance iso-contours are approximately log-linear in this space---consistent with prior findings in the game of Hex \cite{Jones2021} and suggested more broadly for LLMs as well \cite{EpochAI2023}.

Figures~\ref{fig:AIME}.2.a–3.b show that accuracy scales linearly with the logarithm of training compute, inference compute, and tokens used, reinforcing the empirical regularities predicted by our framework.  We show that the first-principles structure of DS3 is not only interpretable and theoretically grounded, but also show in Sections~\ref{subsec:twostep} and~\ref{subsec:beta} that it can unify training and inference scaling through empirical methodologies.
\subsection{Downstream Empirical Fitting Method}\label{subsec:twostep}
We begin with the widely-used empirical model of pretraining loss in \eqref{eqn:pretrainingloss} \cite{Hoffmann2022}, which serves as the foundation for estimating a model’s learned competence. To connect this loss to downstream task performance, we adopt a two-step approach. First, we map pretraining loss to skill-level proficiency via sigmoid transformation \eqref{eq:sig}, following recent work on scaling-to-task transfer \cite{Meta2024llama3, Xiao2024densinglawllms, Owen2024downstreamtasksigmoid, Ruan2024observationalscalinglaws}.

While downstream accuracy over a task distribution need not directly reflect individual task or skill performance, we treat this mapping as a functional ansatz. The qualitative behaviors we investigate rely on two assumptions: (i) that the mapping from pretraining loss to individual skill success probability $\hat{\iota}p$ is monotonic over some regime (i.e., improved training enhances the model’s ability to compose and execute skills) and (ii) that the DS3 formalism meaningfully captures the structure of inference scaling as demonstrated in Section~\ref{subsec:accuracyscaling}\footnote{We focus on the regime where overtraining smaller models continues to improve downstream performance \cite{Meta2024llama3, Devries2023, Gadre2024, Sardana2024}. However, our framework does not preclude analysis of later regimes where overtraining leads to saturation or degradation in transferability \cite{Abnar2021upstreamtodownstreamsaturation, Lin2022inversescalingtruth, Liu2022betterdownstream, Parrish2022biggermorebias, Suzgun2022bigbenchCoT, Wei2023inversescaling, Diaz2024scalinglawsevaluations, Liu2025notjustscalinglaws, Springer2025overtrainedfinetuning, Chen2025rethinkingfinetuning}, which we leave to future work.}.

For simplicity, we omit modeling multiple solution strategies or prerequisite-based structures, as these lack a natural expression in the empirical setting without introducing additional assumptions or empirical justifications. We also exclude attention-related inference costs. We then simulate inference using the DS3 formalism under the skill proficiency profile induced by the pretraining loss.

\subsubsection{Efficiency Benefits of Explicit Token Scaling}
\begin{figure}[!htbp]
    \centering    \includegraphics[width=\textwidth]{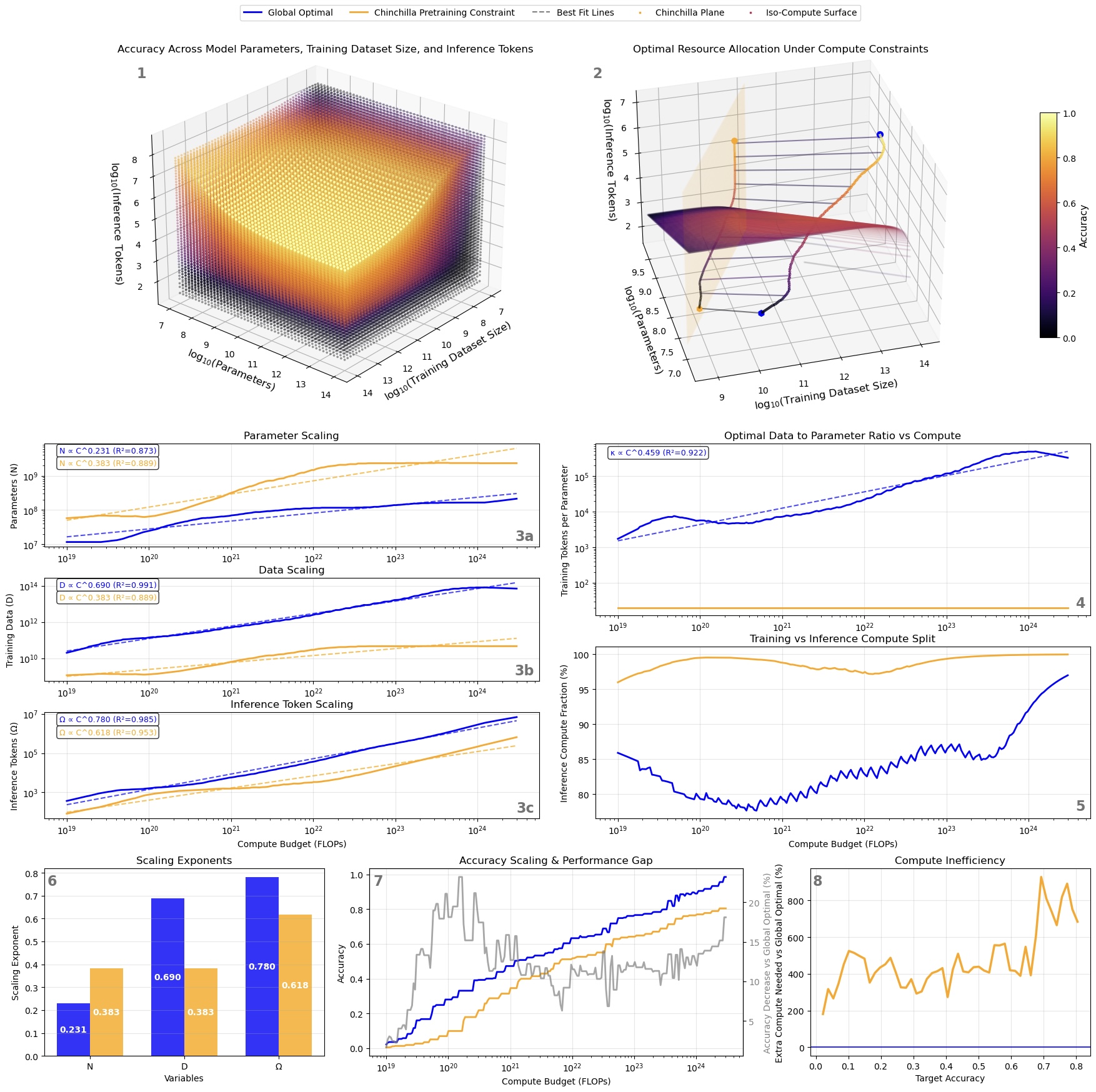}
    \caption{Hypothetical comparison of global vs Chinchilla pretrained using two-step training and DS3 inference theory---see Appendix~\ref{ap:chinchilla_params} for detailed setup and description. (1) Log-space accuracy surface used for optimization. (2) Optimal trajectories on a constant-compute surface: unconstrained global optimum (blue), Chinchilla pretrained and plane (orange, $\kappa=20$). (3a-c) Scaling of parameters $N$, training tokens $D$ and inference tokens $\Omega$ with approximate power-law fits. (4) Optimal data-per-parameter ratio $\kappa$ rises with compute, exceeding the Chinchilla constant. (5) Chinchilla spends a larger share of FLOPs at inference. (6) Fitted power-law exponents show steeper growth in $D$ and $\Omega$. (7) Chinchilla trails the global path in accuracy. (8) Matching accuracy costs significantly more compute, underscoring the inefficiency of the pretraining constraint when inference load and scaling is included.}
    \label{fig:OptimalvsChinchilla}
\end{figure}

Maximizing performance relative to resource expenditure is a central imperative for sustainable AI and efficient system design. This motivates a closer examination of where inefficiencies emerge in standard model scaling practices. As illustrated in the synthetic scenario of Fig.~\ref{fig:OptimalvsChinchilla}, constraining to the two-dimensional plane defined by Chinchilla-style training compute optimality can lead to substantial inefficiency. Our analysis indicates that similar limitations arise in strategies that optimize parameter count and dataset size under a fixed inference token budget, as in \cite{Sardana2024}, since such methods confine scaling to a different two-dimensional slice of the broader design space (see Fig.~\ref{fig:beyondchinchillacomp}). In contrast, allowing joint optimization over model size $N$, dataset size $D$, and inference token budget $\Omega$ enables access to significantly more efficient scaling regimes.

We define the optimal configuration $(N^*(C), D_{\text{tr}}^*(C), \Omega^*(C))$ as the solution to the following constrained maximization problem:
\begin{equation}\label{eq:max_accuracy}
\begin{aligned}
    (N^*(C),\ D_{\text{tr}}^*(C),\ \Omega^*(C)) = \mathop{\arg\max}_{\substack{
        N,\ D_{\text{tr}},\ \Omega \\
        \text{s.t. } 6N D_{\text{tr}} + 2N I \Omega < C \\
        \text{and other constraints}
    }}\ 
    \mathbb{E}_{\phi(\mathcal{L}_0, m)}\left[
        I_{p\bigl(\mathcal{L}(N, D_{\text{tr}}), \mathcal{L}_0\bigr)}\left(m,\ \frac{\Omega}{\omega} - m + 1\right)
    \right]
\end{aligned}
\end{equation}
where $p$ is defined in \eqref{eq:sig}, $\mathcal{L}$ in \eqref{eqn:pretrainingloss}, $I$ is the number of tasks, and $\phi(\mathcal{L}_0,m)$ is the distribution of task difficulties.

\begin{figure}
    \centering    \includegraphics[width=\textwidth]{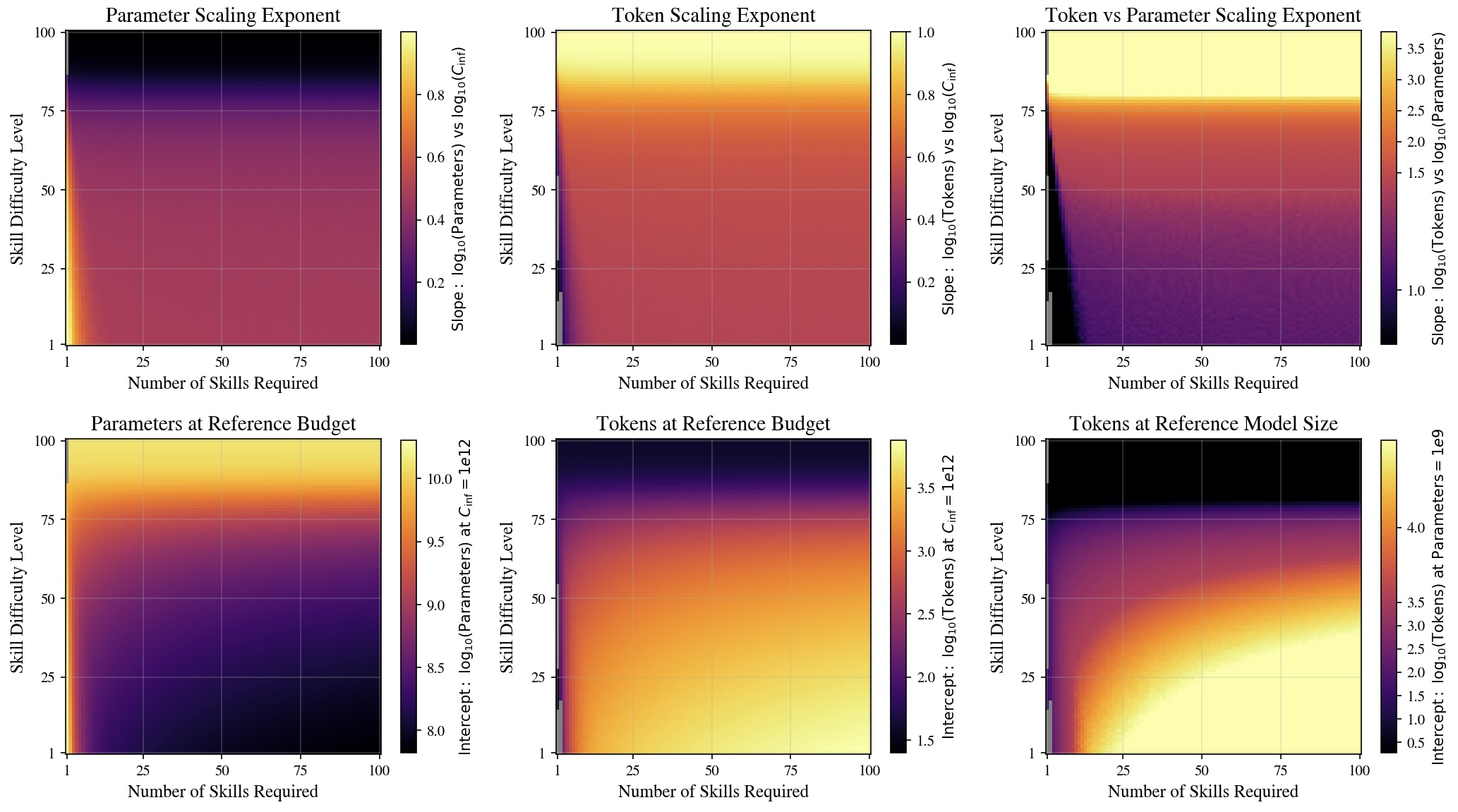}
    \caption{
        Inference-compute-optimal scaling laws across a grid of task difficulties $(\ell(\mathcal{L}_0), m)$ in the infinite training data regime, evaluated under a fixed total inference compute budget $C_{\text{inf}} = 2NI\Omega$ (without attention cost). $\ell(\mathcal{L}_0)$ is a rescaling of the $\mathcal{L}_0$ sweep to provide interpretable task difficulties. Each panel shows a heat map over number of skills and skill difficulty for scaling exponents and reference values. Grayed-out regions are insufficiently valid fits, see Appendix~\ref{ap:infinitedata} for more information.
    } 
    \label{fig:task_difficulty}
\end{figure}

For inference-compute-optimal scaling, dataset size does not affect the cost. Therefore, we consider the case where data is effectively unbounded, to analyze the relationship between model parameters and inference tokens. We analyze this inference-compute-optimal regime across several orders of magnitude, focusing on individual tasks as defined by their difficulty $(\mathcal{L}_0,m)$, as illustrated in Fig.~\ref{fig:task_difficulty} for a range of task difficulties and in Fig.~\ref{fig:inference_scaling_single_task} for a single $(\mathcal{L}_0,m)$ pair. This analysis suggests that tasks requiring more complex skills necessitate high base model capability, and therefore larger model size, but benefit primarily from scaling in inference tokens. In contrast, tasks involving simpler or moderately difficult skills exhibit scaling in both parameters and tokens, with base requirements that vary according to task complexity.

\subsection{Can Reasoning Elicit Emergence?}\label{subsec:reasoning-emergence}
Scaling the depth of reasoning seemingly lowers the model size at which new capabilities emerge. This has been empirically observed through CoT style prompting improving flat or even inverse scaling\footnote{We leave completely flat and inverse scaling for future work but conjecture that this is due to similar properties as discussed in this section. Scaling model size beyond a point may hurt performance by adding relevant but not-required skills without improving reasoning ability.} in model size to positive scaling
\cite{Suzgun2022bigbenchCoT,Wei2023inversescaling}.

\begin{thm}[Reasoning Depth Amplifies Model Scaling Benefits]\label{thm:slope-amplification}

Let $S(N,t) = 1 - (1 - r(N))^t$ denote the probability of successfully completing a skill in $t$ reasoning steps, where $r(N) \in (0,1)$ is a strictly increasing function (i.e., $dr/dN>0$) representing the per-step success probability of a model with $N$ non-embedding parameters. Define the sensitivity of accuracy to model size as the $N$-slope:
\[
\Delta(N,t) \coloneqq \frac{\partial S}{\partial N} = t (1 - r(N))^{t - 1} \frac{dr}{dN}.
\]

Consider how $\Delta(N,t)$ changes with reasoning depth $t$ at a model size $N$, focusing on the regime when $r(N)$ is small and this slope $\Delta(N,1)$ is shallow, corresponding to model sizes where accuracy has plateaued at low reasoning depth. Then increasing $t$ steepens this slope, improving scaling.
\end{thm}

\begin{proof}\label{proof:emergentscaling}
Write $A(t,r)\coloneqq t(1-r)^{t-1}$; then $\Delta(N,t)=A(t,r)\,dr/dN$.
Throughout we keep $r\in(0,1)$ fixed (i.e., not a function of reasoning depth) and treat $A$ as a function of $t$.

Let $r \coloneqq r(N)$ and define the critical reasoning depth
\[
t^*(r) \coloneqq \frac{1}{r} - 1.
\]
Then:
\begin{enumerate}
\item \textbf{Monotonicity in $t$.}\;
      For $r\in(0,\tfrac12)$,
      \[ \Delta(N,t+1)\;\ge\;\Delta(N,t), \qquad t\le t^{\star}(r),\]
      with strict inequality whenever $t<t^{\star}(r)$. Hence $\Delta(N,t)$ increases with $t$ up to $t^{\star}(r)$. Compute the discrete ratio
      \[
      \frac{A(t+1,r)}{A(t,r)}=\frac{t+1}{t}\,\bigl(1-r\bigr).
      \]
      Hence $A(t+1,r)\ge A(t,r)$ if and only if $(1-r)(t+1)/t\ge 1$ and thus $1/r-1\ge t$, i.e., $t^{\star}(r) \ge t$.
      \vspace{.5em}
\item \textbf{Maximum Amplification Factor.}\;
      Define
      \[
        g(r)\;\coloneqq\;\bigl(t^{\star}(r)\bigr)\bigl(1-r\bigr)^{t^{\star}(r)-1}
               \;=\;\Bigl(\frac1r-1\Bigr)(1-r)^{\frac1r-2},
        \qquad r\in(0,\tfrac12),
      \]
      and set $g(r)=1$ for $r\ge\tfrac12$. The maximum of $A(t,r)$ over $t\in \mathbb{N}$ is attained at $t^{\star}(r)$ whenever $r<\tfrac12$; for $r\ge\tfrac12$ it is already maximal at $t=1$. Substituting $t^{\star}(r)$ gives the stated $g(r)$. Then
      \[
        \max_{t\ge1}\Delta(N,t)
        \;=\;g\bigl(r(N)\bigr)\,\frac{dr}{dN},
        \qquad
        \frac{\max_{t\ge1}\Delta(N,t)}{\Delta(N,1)}
        \;=\;{g(r)}.
      \]
\item \textbf{Properties of $g(r)$.}\;
      The map $g:(0,\tfrac12]\to[1,\infty)$ is strictly decreasing and
      \[
        \lim_{r\to0^{+}}g(r)=+\infty,
        \qquad
        g\bigl(\tfrac12\bigr)=1.
      \]
      For $r\in(0,\tfrac12)$,
      \[
      \log g(r)=\log \bigl(\tfrac1r-1\bigr)+\Bigl(\tfrac1r-2\Bigr)\log(1-r).
      \]
      Differentiating,
      \[
      \frac{d \log g(r)}{dr}=-\frac1{r(1-r)}-\frac{\log(1-r)}{r^{2}}-\Bigl(\tfrac1r-2\Bigr)\frac1{1-r}=-\frac{\log(1-r)+r}{r^{2}(1-r)}<0,
      \]
      because $\log(1-r)<-r$ for $r\in(0,1)$.  Thus $\log g(r)$ (and hence $g$) is strictly decreasing. The limit $\lim_{r\to0^{+}}g(r)=+\infty$ is proved in Appendix~\ref{ap:g_limit}, while $g(\tfrac12)=1$ is immediate.
      Thus, for increasingly difficult skills
      (small $r$), reasoning depth can boost the $N$-slope by an arbitrarily large factor, whereas for easier skills ($r\ge\tfrac12$) no boost is possible in the slope (although the accuracy is still increased).
\end{enumerate}
\end{proof}

\subsubsection{Interpretation and Mechanisms}
We identify three conceptual mechanisms through which reasoning depth enhances performance when increasing model size has plateaued, only the first of which is formally established in Thm.~\ref{thm:slope-amplification}.
\begin{itemize}
\item[(i)]
\emph{Model size–to–skill capability mapping:} Thm.~\ref{thm:slope-amplification} demonstrates that, for a model size $N$ where performance has plateaued, increasing reasoning depth $t$ not only increases task success probability $S(N,t)$, but also steepens its derivative with respect to $N$---that is, accuracy improves more sharply with model size under deeper reasoning.
\item[(ii)] \emph{Skill-count bottleneck:} When a task requires $m$ skill applications that are currently available to the model, no number of reasoning steps $t < m$ suffices for task completion. Therefore, the model must scale in size to acquire new skills, yet this may take a significant increase in training. Increasing $t$ beyond $m$ removes this bottleneck, enabling continued scaling through improving current skills required for the task.
\item[(iii)] \emph{Decomposition:} Reasoning steps may also enable the model to decompose complex tasks into sequences of simpler sub-skills, effectively reducing per-step difficulty. This mechanism allows structured inference to compensate for limited base competence, while facilitating continued performance gains with depth, analogous to mechanism (ii) but operating through internal task restructuring.
\end{itemize}

These mechanisms independently reproduce the qualitative pattern in which increased reasoning depth leads to earlier gains in capability. When applied over a distribution of task difficulties, the framework predicts two characteristic behaviors: for sets of tasks with a broad and continuous difficulty distribution, performance improves smoothly as the number of reasoning steps increases; for tasks with discretely spaced or clustered difficulty levels, performance may exhibit sharp, step-like emergence curves. Under a task-centric interpretation of emergence \cite{Nayak2025}, this suggests that reasoning can actively elicit emergent behavior, shifting capability thresholds to smaller scales. In both settings, additional reasoning can enable smaller models to solve tasks that would otherwise remain unsolvable. To push performance further, one can also leverage aggregation techniques that operate over multiple reasoning samples, motivating a closer study of Bo$N$ and MV behavior across heterogeneous task populations.
\subsection{BoN Coverage and MV Saturation}\label{subsec:beta}

In a canonical two-task setting, let \(\phi_1 = \phi_2 = \tfrac12\) denote the prior probabilities of encountering tasks 1 and 2, respectively, and let \(\psi_1 > 0.5\) and \(0<\psi_2 < 0.5\) denote the corresponding probabilities of success on each task. Let the output alphabet $\mathcal{Y}$ have size two (i.e., one correct and one incorrect option). Then the pass@$k$ probability---corresponding to Bo$N$ with an oracle verifier (hereafter simply Bo$N$)---is given by:
\begin{equation}
    \Psi^{\mathrm{Bo}N}(k) =1 - \frac{1}{2}(1-\psi_1)^k - \frac{1}{2}(1-\psi_2)^k.
\end{equation}
This expression reflects the probability that at least one correct answer appears among $k$ independent samples. It increases monotonically with $k$ and saturates at 1 as $k \to \infty$.

By contrast, MV accuracy depends not only on whether the correct answer appears among the outputs, but on whether it constitutes the majority. For a given task $j$ with base success probability $\psi_j$, and assuming a binary output alphabet $\mathcal{Y}$ as before, the number of correct outputs among $k$ independent samples follows a binomial distribution: $Y_j \sim \mathrm{Binomial}(k, \psi_j)$. Under MV with random tie-breaking, the expected success probability is:
\begin{equation}
\psi_j^{\mathrm{MV}}(k) = \sum_{k' = \lceil k/2 \rceil + 1}^{k} \binom{k}{k'} \psi_j^{k'} (1 - \psi_j)^{k - k'}+
\frac{1}{2} \binom{k}{k/2} \psi_j^{k/2} (1 - \psi_j)^{k/2} \cdot \mathbbm{1}_{k \text{ even}}.
\end{equation}
The second term accounts for the case of an exact tie, resolved by random selection. In the large-$k$ limit, the binomial distribution concentrates, and the expected MV accuracy approaches a step function: $\psi_j^{\mathrm{MV}}(k) \to 1$ if $\psi_j > \tfrac{1}{2}$, and $\psi_j^{\mathrm{MV}}(k) \to 0$ if $\psi_j < \tfrac{1}{2}$. Consequently, for a uniform mixture over two tasks, one with $\psi_1 > \tfrac{1}{2}$ and the other with $\psi_2 < \tfrac{1}{2}$, the aggregate MV accuracy saturates at $\Psi^{\mathrm{MV}}(k) \to \tfrac{1}{2}$.

For mixed populations of tasks with varying base success rates, this distinction leads to divergent behavior: Bo$N$ continues to benefit from increased sampling, while MV accuracy can saturate early. The more general MV policy defined in Section~\ref{subsec:MV} accommodates arbitrary output alphabets and distributions over incorrect answers. Suppose task difficulty is parameterized by $\ell$ and $m$, and that the corresponding base success probability, dependent on both model capabilities and inference strategies, is given by $\psi(\ell, m)$. Then the Bo$N$ coverage becomes:
\begin{equation}
\Psi^{\mathrm{Bo}N}(k) = 1 - \int\sum_m (1 - \psi(\ell, m))^k\, \phi(\ell, m)\, d\ell\,
\end{equation}
where $\phi(\ell, m)$ denotes the joint density over task difficulty parameters $\ell$ and $m$.

Equivalently, we may define the induced distribution over base success probabilities, denoted $f(\psi)$, via the pushforward of $\phi(\ell, m)$ under the mapping $\psi = \psi(l, m)$, yielding:
\begin{equation}
\Psi^{\mathrm{Bo}N}(k) = 1 - \int_0^1 (1 - \psi)^k f(\psi)\, d\psi.
\end{equation}

\begin{figure}[ht]
    \centering    \includegraphics[width=\textwidth]{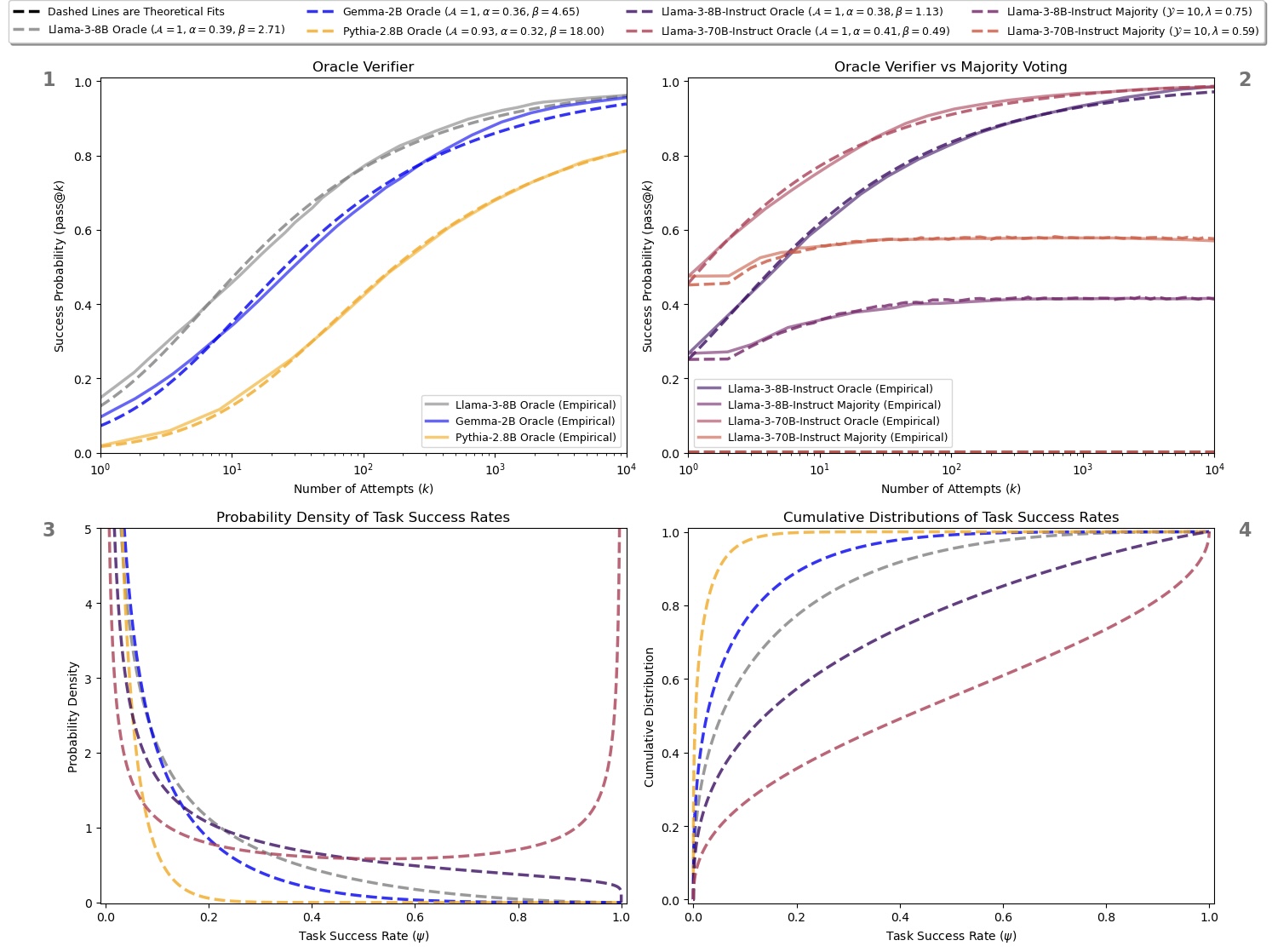}
    \caption{Comparison of oracle verifier and MV accuracy using empirical pass@$k$ and Beta-distributed task success rates. Solid lines are empirical data from \cite{Brown2024repeatedsamplingsscaling} for various models on the MATH benchmark \cite{Hendrycks2021MATH}. (1) Pass@$k$ curves for the Bo$N$ with oracle verifier on three base models; dashed lines the Beta-fit prediction similar to \cite{Levi2024simple}. (2) Same oracle curves (solid) versus majority-vote accuracy (dashed) for the two instruction-tuned Llama-3 models, showing early MV saturation. (3) PDFs of task-level success rates implied by the Beta fits in 1 and 2, highlighting heavy mass near low success rates and, for Llama-3-70B (peach), a secondary peak near high success rates. (4) Corresponding CDFs, illustrating the increasing spread of distribution with model capabilities.}
    \label{fig:passatN}
\end{figure}

A recent theoretical formulation for pass@$k$ proposed by Levi \cite{Levi2024simple} treats the error rate as a Beta-distributed random variable, which asymptotically yields a power-law decay in failure probability. We formally derive this asymptotic behavior in Appendix~\ref{ap:asymptoticBoN}. Motivated by this observation, we represent the induced success probability distribution $f(\psi)$ as a mixture of a Beta distribution and a point mass at $\psi = 0$, representing tasks the model cannot solve\footnote{This can be understood most naturally in our framework as not having the skills learned in the skill graph. Another possible interpretation is that the model does not have enough reasoning steps in the budget to perform the task given the available skills.}. Specifically, let $\mathcal{A} \in [0,1]$ denote the fraction of tasks that are at least partially solvable, and $1 - \mathcal{A}$ the fraction of unsolvable tasks. Then the density becomes:
\begin{equation}
f(\psi) = (1 - \mathcal{A}) \delta(\psi) + \mathcal{A} \frac{\psi^{\alpha - 1}(1 - \psi)^{\beta - 1}}{B(\alpha, \beta)},
\end{equation}
where $\delta(\cdot)$ is the Dirac delta function centered at $\psi = 0$, and $B(\alpha, \beta)$ is the Beta function. Then, the pass@$k$ coverage becomes:
\begin{equation}
\begin{aligned}
\Psi^{\mathrm{Bo}N}(k)
&=1 - \int_{0}^{1}(1-\psi)^k\,f(\psi)\,d\psi,\\
&=\mathcal{A}\left(1 - \frac{B(\alpha,\beta+k)}{B(\alpha,\beta)}\right),
\end{aligned}
\end{equation}
where the second line follows from substituting the Beta form of $f(\psi)$. Swapping the Beta distribution parameters recovers the formulation proposed by Levi \cite{Levi2024simple}, shown to closely match empirical pass@$k$ curves reported by Brown et al. \cite{Brown2024repeatedsamplingsscaling} on the MATH benchmark \cite{Hendrycks2021MATH} for the Pythia-2.8B \cite{Biderman2023pythia}, Llama-3-8B \cite{Meta2024llama3}, and Gemma-2B \cite{Gemmateam2024gemma} models.

We perform our own fit under the same framework\footnote{We report slightly different best-fit parameters likely due to differences in fitting procedure or dataset preprocessing.}, and visualize the resulting curves in Fig.~\ref{fig:passatN}, including the predicted CDF and PDF of task-level success probabilities. Additionally, we extend the analysis to two other models: Llama-3-8B-Instruct and Llama-3-70B-Instruct \cite{Meta2024llama3}. For all models, we observe a pronounced mass of tasks with low success probability. For Llama-3-70B-Instruct, we find a secondary peak near $\psi=1$ indicating a subset of tasks that are likely trivial for the model. This bimodality may reflect the increased base capabilities of the larger model, which solves some tasks with near certainty.

We further ask whether the same framework used for Bo$N$ can also capture MV scaling behavior as a function of $k$. Taking the expectation of the closed-form MV accuracy expression, under the assumption of random tie-breaking and a fixed number of output classes, yields:
\begin{equation}\label{eq:mvfull}
\Psi^{\mathrm{MV}}(k,\mathcal Y)
= \int_0^1 \sum_{\substack{\sum_{j'=0}^{|\mathcal{Y}|}y_{j'}=k}} 
\frac{ \mathbbm{1}\{y_0 = \max_{j\in \{0,\dots,|\mathcal Y|\}} y_j\} }
     { \sum\nolimits_{j=0}^{|\mathcal Y|} \mathbbm{1}[y_j = y_0] }
\cdot \frac{k!}{\prod_{j'=0}^{|\mathcal{Y}|}(y_{j'}!)} \,
\psi^{\,y_0}
\prod_{j=1}^{|\mathcal Y|} q_j(\psi)^{y_j} \,
f(\psi) \, d\psi,
\end{equation}
where $q_j(\psi)$ is the distribution over incorrect outputs.

Assuming that the distribution over incorrect outputs is independent of the base success rate $\psi$, such that $q_j(\psi) = (1 - \psi)c_j$, the resulting expression simplifies when integrating over the Beta form of $f(\psi)$. This leads to a fully specified formula for $\Psi^{\mathrm{MV}}(k, \mathcal{Y})$, parameterized by $(\alpha, \beta, \mathcal{A}, \{c_j\})$:
\begin{equation}
\Psi^{\mathrm{MV}}(k,\mathcal Y)
 =\mathcal{A} \sum_{\substack{\sum_{j'=0}^{|\mathcal Y|} y_{j'} = k}} 
\frac{ \mathbbm{1} \left\{y_0 = \max_{j \in \{0,\dots,|\mathcal Y|\}} y_j \right\} }
     { \sum\nolimits_{j=0}^{|\mathcal Y|} \mathbbm{1}\left\{y_j = y_0 \right\} }
\cdot \frac{k!}{\prod_{j'=0}^{|\mathcal Y|} y_{j'}!} \,
\prod_{j=1}^{|\mathcal Y|} c_j^{y_j} \cdot 
\frac{B(\alpha + y_0, \beta + k - y_0)}{B(\alpha, \beta)}.
\end{equation}
The derivations and supplemental details for \eqref{eq:mvfull}--\eqref{eq:pushforward_map} are provided in Appendix~\ref{ap:mv}.

We apply this formulation to the Llama-3-8B-Instruct and Llama-3-70B-Instruct models \cite{Meta2024llama3}, using pass@$k$ curves from Brown et al. \cite{Brown2024repeatedsamplingsscaling} to first fit the Bo$N$ parameters $(\alpha, \beta, \mathcal{A})$. These fitted parameters are then carried over directly into the MV model without additional tuning. Since the asymptotic MV accuracy is governed by the most frequent incorrect response, we calibrate the  $c_j$ values to match this limit. In particular, we set the dominant incorrect output rate to reproduce the empirically observed MV saturation point. This leads to the constraint:
\begin{equation}\label{eq:cstar}
    \max_{j\in\{1,...,|\mathcal{Y}-1|\}} c_j=\frac{
    I^{-1}\left(1-\frac{P_{\infty}^{\mathrm{MV}}}{\mathcal A};\alpha,\beta\right)
    }{1-I^{-1}\left(1-\frac{P_{\infty}^{\mathrm{MV}}}{\mathcal A};\alpha,\beta\right)},
\end{equation}
where $I^{-1}(y;\alpha,\beta)$, is the inverse of the regularized incomplete beta function, and $P_\infty^{\mathrm{MV}}$ measures MV saturation accuracy.

\begin{figure}[ht]
    \centering
\includegraphics[width=\textwidth]{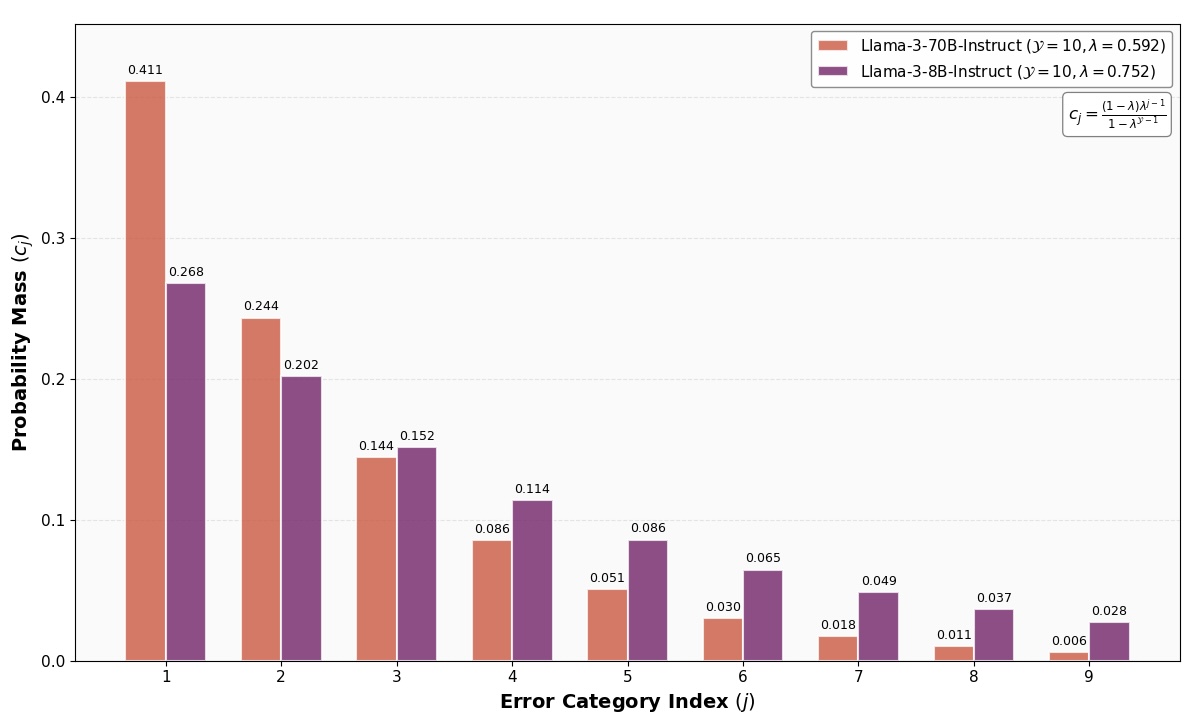}
    \caption{Error distributions used in MV accuracy fits of Fig.~\ref{fig:passatN}, assuming a truncated geometric form over the $|\mathcal{Y}| - 1 = 9$ incorrect response ranks. Each curve shows the categorical error probability $c_j \propto (1 - \lambda)\lambda^{j - 1}$, where $\lambda$ controls how sharply errors concentrate near the top-ranked incorrect responses. For Llama-3-70B-Instruct ($\lambda = 0.592$, peach), most error mass lies in the top-ranked incorrect option; for Llama-3-8B-Instruct ($\lambda = 0.752$, purple), errors are spread more broadly across ranks.
}\label{fig:mv_geometric}
\end{figure}

We then take the distribution over incorrect outputs to be an ordered, truncated geometric distribution:
\begin{equation}
    c_j = \frac{(1-\lambda)\lambda^{j-1}}{1-\lambda^{|\mathcal{Y}|-1}},\quad j\in \{1,...,|\mathcal{Y}|-1\},
\end{equation}
where each $c_j$ is weighted by the task-dependent probability of producing an incorrect answer, ($1-\psi$). This form captures a decaying preference over incorrect outputs, a plausible behavior for language models, which tend to concentrate probability mass on a few high-likelihood completions even when incorrect \cite{Chen2025rethinkingfinetuning}.

Given a fixed alphabet size $|\mathcal{Y}|=10$, we solve for the decay parameter $\lambda$ using the inferred maximum incorrect output probability $c_1$ from \eqref{eq:cstar}. The resulting distribution over incorrect outputs is shown in Fig.~\ref{fig:mv_geometric}. One possible interpretation is that larger models, which tend to achieve lower pretraining loss, not only increase confidence in correct predictions but also exhibit sharper preferences among incorrect ones---an effect reminiscent of overtrained models, as discussed in \cite{Chen2025rethinkingfinetuning}. The final fit to MV accuracy is shown in Fig.~\ref{fig:passatN}.

\begin{figure}[ht]
    \centering
    \includegraphics[width=\textwidth]{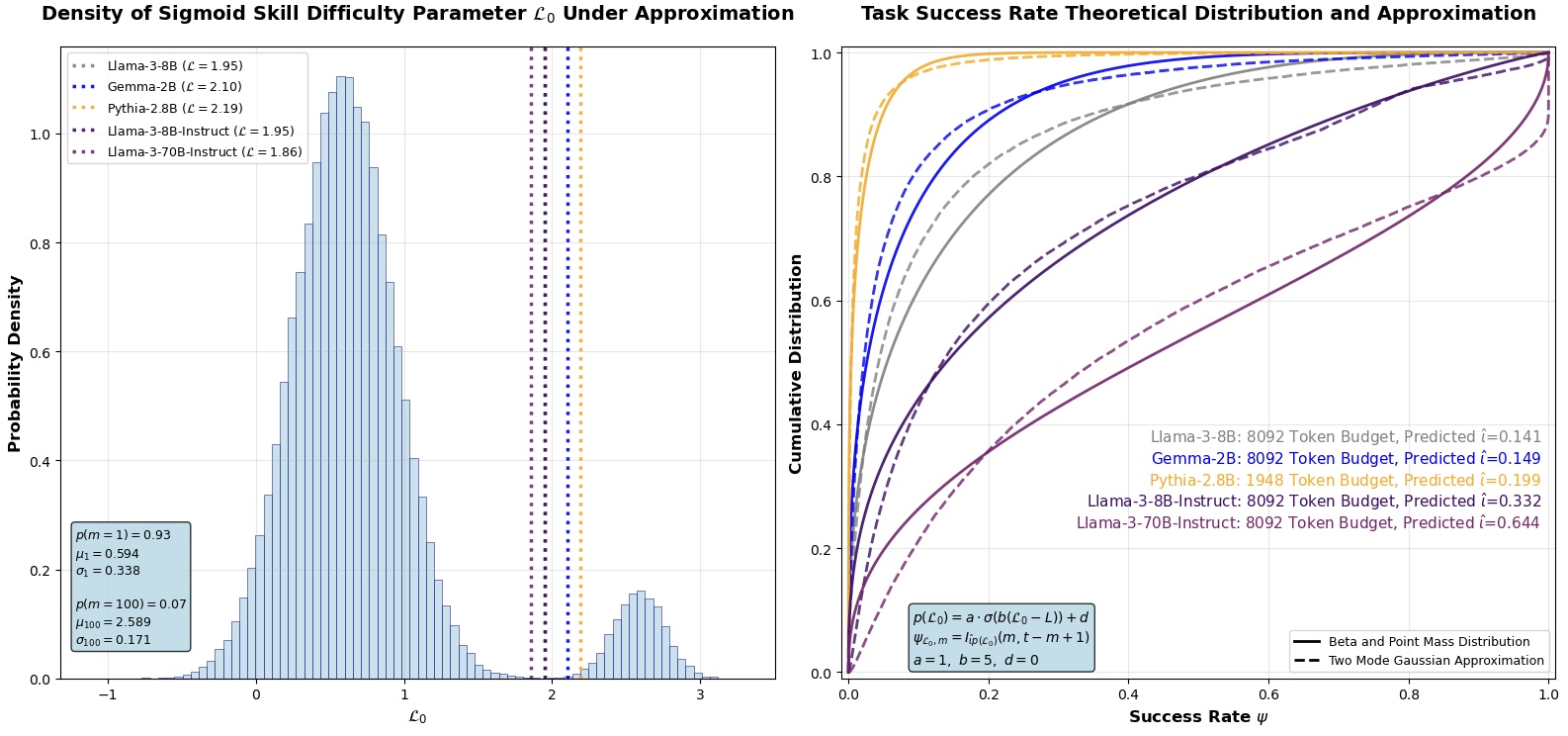}
    \caption{Example of inferred skill distributions and approximate search directionality from pass@$k$ fits. A two-mode Gaussian over task difficulty $\mathcal{L}_0$ (left; histogram with model-specific loss $\mathcal{L}$ markers) is jointly optimized with the search-directionality parameter $\hat{\iota}$. The resulting push-forward CDFs (right, dashed) closely match the Beta+point-mass fits (solid) from Fig.~\ref{fig:passatN}. In-plot labels report the inferred $\hat{\iota}$ values for each model at its token budget.}
    \label{fig:searchdirectionality}
\end{figure}

We further demonstrate, as a proof of concept, that a distribution over task-level parameters can be pushed forward through the two-step sigmoid ansatz to estimate the model-specific search directionality parameter $\hat{\iota}$. Specifically, we assume each task is characterized by a pair $(\mathcal{L}_0, m)$, where $m \in \{1, \dots, 100\}$ denotes the required number of skills, and $\mathcal{L}_0 \sim \mathcal{N}(\mu_m, \sigma_m^2)$ represents the skill difficulty. The resulting success probability $\psi$ is computed via the parametric mapping:
\begin{equation}\label{eq:pushforward_map}
(\mathcal{L}_0, m, \, \hat{\iota}) \mapsto
\psi = I_{\hat{\iota} \cdot \sigma(b(\mathcal{L}_0 - \mathcal{L}))}(m, \Omega_{\max}/\omega - m + 1),
\end{equation}
where $\sigma(z) = 1 / (1 + e^{-z})$ is the logistic sigmoid. This map induces a pushforward distribution $f_\Psi(\psi)$, which we compare against empirical model performance.

The Gaussian mixture parameters $(\mu_m, \sigma_m^2)$ and the model-specific directionality coefficients $\underline{\hat{\iota}} = (\hat{\iota}_1, \hat{\iota}_2, \dots)$ are jointly optimized via maximum likelihood to fit the observed pass@$k$ curves. Despite strong simplifying assumptions, including a two-mode Gaussian task prior and the exclusion of alternative solution strategies, the estimated values of $\hat{\iota}$ are qualitatively consistent with expectations, assigning higher directionality to instruction-tuned models. Results are shown in Fig.~\ref{fig:searchdirectionality}.

As this construction is inherently sensitive to the assumed task distribution and has not yet undergone thorough robustness analysis, we leave a more detailed investigation for future work.

\section{Implications}\label{sec:implications}
\subsection{Distinct Energy Profiles}\label{subsec:energy}
The computational economics of LLMs exhibit fundamental asymmetries between training and inference phases. Contrary to simplified equivalency assumptions as in \eqref{eq:totalcompute}, three inference tokens do not correspond to one training token in terms of computational cost, leading to markedly different resource allocation requirements.

\subsubsection{Usage and Deployment Environments}
Energy prices vary significantly based on time of day, season, generation source, and geographic location, reflecting dynamic demand patterns, resource availability, and regional infrastructure constraints \cite{Borenstein2005, EIAPricing}. Whereas training time can be modulated to be in off-peak energy usage times, inference is generally desired at the time of query for most use cases.

Due to relative predictability of usage, in particular in education \cite{Handa2025}, it may be possible to perform some probabilistic pre-computing and store energy in information as proposed in information batteries \cite{Switzer2022InformationBatteries}. Furthermore, one may choose to tranche intelligence based probabilistically on how much inference-compute scaling is done, governed according to renewable energy availability forming a risk-limited dispatch of (inference-scaled) intelligence \cite{Varshney2014RiskLimitedDispatchKnowledge}, much like energy itself \cite{Varaiya2011RiskLimitedDispatchEnergy}. This can further enable energy savings in lower-priority or deferred inference tiers, where relaxed latency constraints allow for hardware-level energy-saving strategies such as power capping \cite{Zhou2023SustainableSupercomputing}. In line with our findings that optimal compute strategies vary significantly with task difficulty and target accuracy, intelligent routing \cite{Ong2025routellm, Maddireddy2025locoml} to appropriately sized LLMs could yield substantial additional energy savings---a principle reflected in how providers like OpenAI and Anthropic apply differential pricing \cite{OpenAI2024GPT4Pricing} and usage limits \cite{Anthropic2025UsageLimits} for larger models and extended reasoning.

These dynamics are further influenced by the deployment environment---for instance, edge inference on resource-constrained devices and latency requirements, such as drones \cite{Chen2024typeflyflyingdroneslarge}, often involves hardware with limited computational capabilities and energy resources, leading to distinct optimization tradeoffs \cite{Rey2025DroneEdge}.

\subsubsection{Algorithms and Hardware}\label{subsec:algorithmsandhardware}

The discrepancy between training and inference manifests in multiple algorithmic and hardware dimensions. For example, hardware utilization rates during inference often fall significantly below training utilization levels \cite{Pope2023, Korthikanti2023, Sardana2024}, post-training quantization substantially reduces inference computational demands \cite{Frantar2023, Sardana2024}, and processing costs differ notably between input and output token generation \cite{Sardana2024}.

GPU architectures present varying degrees of suitability across the AI computational spectrum. While training workloads prioritize raw computational throughput, inference operations demand a more nuanced balance among several parameters: computational capacity, memory bandwidth, interconnect performance, network throughput, and total memory availability. While memory and interconnect bandwidth have long served as a fundamental bottleneck in conventional computing and AI, the multiplicative factor of these bandwidths are about half that of computational capacity FLOPS (FLOPs per second) forming a compute overhang \cite{Gholami2024}---only exacerbated during inference.

Mixture-of-Experts (MoE) networks reduce feedforward compute by activating only a small subset of expert subnetworks for each token’s layer-wise representation, while leaving attention-related compute growth largely unchanged \cite{Shazeer2017moe, Fedus2022switch, DeepSeekAI2024, DeepSeekAI2025}. This can yield notable inference savings in large models or short contexts, though the gains may diminish on longer sequences where attention and memory costs dominate. During training, expert routing incurs additional coordination overhead, and the overall effect of MoE architectures on higher-level reasoning performance remains nuanced \cite{Jelassi2025mixture, Wang2025two}.

The efficiency profile of inference workloads is substantially influenced by batching strategies. Expanded batch sizes facilitate improved amortization of memory access latencies and enable more efficient computational scheduling, thereby enhancing overall throughput \cite{Pope2023}. However, this theoretical optimization encounters practical limitations in production environments, particularly those requiring low-latency responses or serving online requests, where large batch aggregation proves impractical. This constraint often results in substantial underutilization of available hardware capacity.

Recent advances in model architecture have increasingly targeted memory efficiency, particularly focusing on key-value (KV) cache optimizations. The KV cache represents a critical memory structure that stores the key and value tensors computed for every token in the context during transformer-based inference. The per token cost scales linearly with context length, and thus the sum cost scales quadratically---each additional token requires storing corresponding KV pairs across all attention layers. But algorithmic innovations such as multi-query attention \cite{Shazeer2019} and multi-head latent attention \cite{DeepSeekAI2024} can further reduce this cost. For models with billions of parameters, this creates substantial memory pressure, often hundreds of gigabytes just for the KV cache \cite{DeepSeekAI2024}.

This direct relationship between context length and memory requirements explains why extending context windows places disproportionate demands on memory bandwidth and capacity rather than computational throughput. Efficient KV cache implementations can reduce memory bandwidth consumption by optimizing access patterns, compressing stored tensors, or selectively pruning less relevant entries, enabling longer contexts without proportional cost increases \cite{DeepSeekAI2024}. Specialized hardware approaches like wafer-scale computing architectures further address these bandwidth limitations through massive on-chip memory and optimized KV cache management \cite{He2025}. These developments highlight an emerging paradigm shift toward model architectures specifically designed to minimize memory and bandwidth consumption during deployment---requirements which alters AI policy considerations (e.g., export control frameworks \cite{Somala2025}). While still primarily in research stages of development, spiking neural networks \cite{Xing2025SpikeLLM}, neuromorphic hardware \cite{Merolla2014TrueNorth, Davies2018Loihi}, and even computing with living neurons \cite{Habibollahi2022, Cai2023, EllisMohr2024Directed} provide other pathways to further reducing inference costs.

The choice of inference methodology significantly impacts resource utilization patterns. Sequential reasoning approaches, exemplified by CoT prompting \cite{Wei2023CoT}, increase context length requirements and KV cache demands. Recent work has aimed to help reduce this cost by compressing the reasoning CoT \cite{Cheng2024compressedcot, Shen2025efficientreasoninghiddenthinking, Xia2025cotcompression, Kang2025c3ot, Wu2025moreisless, Lee2025compresscot}. Relatedly, it is unclear how much information is lost in the summarized version of extended thinking which could be useful context in multi-query chats. This raises the question of whether intermediate reasoning steps contain auxiliary signals useful in multi-turn settings. 

Conversely, parallel approaches such as MV and self-consistency sampling \cite{Wang2022self} require multiple simultaneous processing paths, but benefit from batch-level optimizations. Hybrid methodologies like ToT reasoning \cite{Yao2023ToT} attempt to balance these considerations by dynamically pruning reasoning pathways \cite{Wu2025} to maintain manageable memory and computational requirements. Furthermore, inference may include an explicit memory of techniques and heuristics useful for problem solving \cite{Suzgun2025}. In retrieval-augmented generation settings, inference scaling can achieve near-linear performance gains when compute is optimally divided between retrieval and generation components \cite{Yue2024inference}. The choice of inference methodology significantly impacts resource utilization patterns.

As context capabilities extend into the million-token range \cite{DeepMind2024, Meta2025Llama4}, inference costs no longer maintain token-level linearity but instead scale with accumulated KV cache size, as well as associated memory demands. Recent research has proposed revised scaling laws that explicitly incorporate context length as a parameter, revealing significant divergence from Chinchilla optimality. When considering the internal transformer architecture, using fewer query and KV heads while increasing the model size can reduce compute and memory \cite{Chen2025}. These costs become increasingly important as context lengths grow with inference scaling. Altering architectural parameters such as the model width and depth alters latency as well as compute cost. Recent scaling laws have been proposed to consider such architectural parameters, guiding tradeoffs which are of particular importance considering growing context lengths \cite{Bian2025scaling}. Commercial pricing structures reflect these technical realities, with providers implementing context-dependent pricing---OpenAI doubles per-token costs between 8K and 32K context windows \cite{OpenAI2024GPT4Pricing}, while Anthropic explicitly scales usage quotas based on context utilization \cite{Anthropic2025UsageLimits}.

However, this type of autoregressive reasoning over language tokens is not the only form demonstrated in LLMs. It has been suggested that human language is optimized for communication rather than reasoning \cite{Fedorenko2024}. Reasoning tasks across a variety of domains appear to operate independently of classical language-processing regions, according to neuroimaging findings \cite{Monti2007, Monti2009, Fedorenko2011, Monti2012, Amalric2019}. This, among other algorithmic advantages, has prompted researchers to study latent space reasoning \cite{Hao2024, Geiping2025}, where models operate over continuous latent representations rather than projecting onto a single next token. These approaches may reduce redundant layer-wise processing over multiple tokens and alleviate context-length constraints. It remains an open question whether blending autoregressive and latent-space inference enhances reasoning, and what tradeoffs arise as one leans more heavily on one mode over the other.

Likewise, multi-modal models integrating information processing on images, video, or audio can reason through generation and manipulation in these information spaces (e.g., images \cite{Li2025, OpenAI2025ImageThinking}). Concurrently, generation may be improved by inference scaling \cite{Kim2022, Ma2025imageinferencescaling}.

Diffusion language models (dLMs) represent another promising scaling direction. Unlike standard autoregressive decoding, dLMs perform iterative denoising steps, which enables more global control and potential improvements in coherence \cite{Li2022diffusionlm, Han2023dlm, Lou2024discretedlm, Gong2024scalingdlm, Nie2025dLLMs}. While dLMs incur higher computational costs (in terms of FLOPs) than autoregressive models due to repeated forward passes and cache recomputation, they may offer lower latency if the number of parallelizable diffusion steps is smaller than the target output length. The overall efficiency depends heavily on hardware utilization and implementation strategy. Within the autoregressive regime, techniques such as speculative decoding \cite{Leviathan2023} aim to reduce inference latency by decreasing the number of full-model forward passes required during generation.

Inference scaling also can be done by employing multiple agents \cite{Ye2024MultiAgentInfScaling, Jin2025multiagenttesttime, Qian2025scalingmultiagent}. Due to the bi-directional nature of collaboration, non-parallel techniques for multi-agent scaling is dependent on training and inference and thus may need to be understood in a unified manner.

\subsection{Capability Forecasting}\label{subsec:forecasting}
The predictability of loss with training compute has allowed for compute to serve as a reasonable proxy for AI capabilities due to its quantifiability, detectability, excludability, and concentrated supply chain (i.e., chip manufacturing; cf. uranium enrichment for nuclear risk or synthetic nucleic acids for biological risk) \cite{Sastry2024, Shavit2023}. However, there are drawbacks to using compute as a regulatory metric.

There is no universal agreement on how compute should be empirically measured or regulated for a given model. Companies do not necessarily disclose or even track the exact compute used in training their models. Moreover, indirect proxies---such as hardware availability, training duration, or total energy consumption---may be insufficient or misleading. Developers could claim their compute was spread across multiple models or that large portions of compute were used inefficiently or did not contribute meaningfully to the final trained system. This could lead to either a reasonable lack of knowledge when actions are impermissible or an apprehension even when actions are permissible regarding investment, development, or use. Future regulatory advancements may address some of these concerns through a combination of AI governance mechanisms (e.g., hardware enabled governance \cite{Kulp2024}). However, many open questions remain about the feasibility and enforcement of such approaches \cite{Reuel2024}.

More broadly, emerging factors complicate oversight. Decentralized training distributes the computational load across multiple nodes, which is especially relevant given that reinforcement learning---currently the dominant method for scaling inference---requires fewer parameter updates between devices. This approach already has been demonstrated at scale and open-sourced through a fully asynchronous, globally distributed collaborative training across a permissionless network of a 32-billion-parameter model. The training-to-inference compute ratio was reported to be approximately 1:4, and it is anticipated that this ratio will shift even further toward inference with further scaling \cite{PrimeIntellect2025decentralizedreasoning}. Here, we do not specifically consider reasoning training, yet this is another aspect in which energy costs are shifting toward inference.

High inference demands may necessitate access to cloud infrastructure, presenting a potential point of leverage for know-your-customer requirements and other governance mechanisms. However, as models become more efficient and portable through techniques like distillation, there is a growing shift toward edge deployment, where inference can occur on local or privately controlled hardware. This diffusion reduces the visibility and enforceability of centralized oversight, raising challenges for monitoring high-capability model use across diverse contexts. Distillation, which contributes to enhancing the viability of inference scaling \cite{Busbridge2025}, decouples deployment capabilities from the original training run by allowing smaller models to inherit and retain performance through compression. This reduces transparency around where and how performance was acquired. Additionally, a broader supply chain---including cloud compute providers and model weight distributors---complicates attribution and oversight, as contributions to final capabilities may be distributed and difficult to trace, cf.~\cite{VarshneyKS219}. Although open-source models can be incrementally fine-tuned to exceed their initial capabilities, significant performance increases have previously come with increasing training on an order-of-magnitude (OOM) basis---meaning one would need enough data and compute to build their own model regardless. Yet, the additional training done, if any, to introduce inference scaling can be extremely sample- and compute-efficient while still leading to significant performance increases \cite{Muennighoff2025, Wang2025RLoneexample}. Separately, narrow or finetuned models can achieve higher capabilities in specialized domains\footnote{For this reason, among possible contamination and overfitting problems \cite{Zhang2024benchmarkissues, Cohen2025benchmarkissues}, benchmarks commonly used to quantify model performance, while providing useful data, are less fundamental and perhaps less useful in oversight and governance. On a related note, benchmarks are also being redesigned to better consider the improved reasoning capabilities of frontier models (e.g., the BIG-Bench Extra Hard \cite{Kazemi2025bigbenchextrahard} problems are significantly longer and are likely to require much lengthier solutions for correct answers than its predecessor BIG-Bench Hard \cite{Suzgun2022bigbenchCoT}).}. Therefore, regulations have targeted specific types of data (e.g., biological sequence data) \cite{FederalRegister2024Investments}.

Existing regulations, such as certain U.S. outbound investment rules \cite{FederalRegister2024Investments} and the EU AI Act \cite{EUAI2024Article51}, explicitly use training compute thresholds (e.g., \(10^{25}\) FLOPs or integer operations). The assumption is that high FLOPs for training implies high capabilities and risks. However, by not explicitly including inference compute, policymakers risk underestimating a model’s real-world impact. Due to the tradeoff between inference and training compute for performance, it seems prudent to include inference in any measure used to quantify performance. One may leverage this broad-strokes tradeoff in response to a growing desire for an effective compute measure \cite{Heim2024training}. A straightforward, albeit imperfect, way to better approximate capabilities is to take the weighted sum of the logarithms of the training and inference compute:
\[
C_{\text{OOM}} = a\log(C_{\text{tr}}) + b\log(C_{\text{inf}}) \mbox{.}
\]
Adopting a metric that includes inference compute provides a fuller view of a model’s potential. It can help policymakers avoid underestimating capabilities and ensure regulations remain effective as the technology evolves.

\subsection{Security}\label{subsec:safety}

Inference-time compute scaling introduces a range of considerations for both capabilities and security. These advances offer significant benefits, but also may increase the potential for misuse in areas such as cybersecurity or biosecurity \cite{Openai2025o3o4, Anthropic2024claude4}.

Inference-time capabilities, particularly in planning and code synthesis, are likely to accelerate progress in automated research \cite{He2021automl, Zhang2025selfimprovementllm, DeepMind2025alphaevolve}, including through methods such as synthetic data generation \cite{Peng2024, DeepSeekAI2025} and improved prompt engineering \cite{Zhou2022humanlevelprompt}.

OpenAI has reported that ChatGPT-o1’s inference implementation and mitigation techniques contribute to alignment by producing human-readable reasoning traces and increasing opportunities to detect misuse\footnote{This kind of visibility would not be directly available in models that rely primarily on latent-space inference \cite{Hao2024, Geiping2025}.}. However, alignment does not necessarily follow from interpretability or capability. Early signs of deceptive reasoning in ChatGPT-o1 \cite{OpenAI2024o1, Meinke2025} raise questions about how inference scaling may affect security.

In security-relevant deployments, inference efficiency can be critical; latency constraints may favor a single lightweight model, or a set of them, over a more powerful but slower alternative. How to trade off capability, interpretability, and responsiveness remains unresolved.

A structured representation of security-relevant behavior can serve as a basis for analyzing how security outcomes are shaped by, and scale with, inference-time capabilities. We believe that our framework can support analysis of inference-time security scaling by representing such behavior in terms of acquired skills or transitions within a structured reasoning process.

\section{Conclusion}\label{sec:conclusion}
As LLMs continue to advance, understanding and strategically allocating compute across training and inference has become critical for maximizing performance and efficiency. In this work, we have presented a unified framework for compute scaling that integrates both training and inference within a shared mathematical structure. By framing reasoning as a directed stochastic search over a latent skill graph, we demonstrated how key empirical phenomena naturally arise from the interaction between inference dynamics, training structure, and task difficulty. Our analysis considers inference as a foundational axis of model capability---one that increasingly calls for codesign spanning training and inference so as to align algorithms, system architectures, and hardware platforms across the full lifecycle of model development and deployment.

As inference becomes an increasingly dominant use of energy, aligning system performance with dynamic energy availability and usage patterns requires new principles of system-level design. These principles extend beyond energy-aware infrastructure to include algorithmic and hardware innovations that reshape the efficiency profile of inference and open new frontiers in deployment strategy. A clearer understanding of compute dynamics can help guide the design of models, infrastructure, and policies that are more efficient and better aligned with societal needs and constraints. These dynamics give rise to qualitatively distinct optimization regimes, each with its own characteristic scaling behavior.

As in biological systems, increasing scale alters a system’s functional demands \cite{Haldane1926}. When organisms grow larger, they often shift ecological roles and constraints---what was optimal at one scale becomes suboptimal at another. Scaling modifies not only resource tradeoffs, but also the broader integration of artificial intelligence into social, economic, and technical ecosystems, creating a feedback loop in which new capabilities reshape the underlying conditions that govern further scaling.

\section*{Acknowledgment}
We thank Xinbo Wu and Michael Nercessian for discussions on aspects of the DS3 framework. We are grateful to Alex Wadell, Akhil Bhimaraju, and Razan Baltaji for valuable conversations on context length considerations and unconventional inference scaling. We also thank Rainer Engelken for thoughtful feedback.

\bibliographystyle{IEEEtran}
\bibliography{abrv,conf_abrv,arem_lib}

\newcommand{\SortNoop}[1]{}
\begin{thebibliography}{100}
\providecommand{\url}[1]{#1}
\csname url@samestyle\endcsname
\providecommand{\newblock}{\relax}
\providecommand{\bibinfo}[2]{#2}
\providecommand{\BIBentrySTDinterwordspacing}{\spaceskip=0pt\relax}
\providecommand{\BIBentryALTinterwordstretchfactor}{4}
\providecommand{\BIBentryALTinterwordspacing}{\spaceskip=\fontdimen2\font plus
\BIBentryALTinterwordstretchfactor\fontdimen3\font minus \fontdimen4\font\relax}
\providecommand{\BIBforeignlanguage}[2]{{%
\expandafter\ifx\csname l@#1\endcsname\relax
\typeout{** WARNING: IEEEtran.bst: No hyphenation pattern has been}%
\typeout{** loaded for the language `#1'. Using the pattern for}%
\typeout{** the default language instead.}%
\else
\language=\csname l@#1\endcsname
\fi
#2}}
\providecommand{\BIBdecl}{\relax}
\BIBdecl

\bibitem{Kaplan2020}
J.~Kaplan, S.~McCandlish, T.~Henighan, T.~B. Brown, B.~Chess, R.~Child, S.~Gray, A.~Radford, J.~Wu, and D.~Amodei, ``Scaling laws for neural language models,'' {arXiv:2001.08361 [cs.LG]}, 2020.

\bibitem{Hoffmann2022}
J.~Hoffmann, S.~Borgeaud, A.~Mensch, E.~Buchatskaya, T.~Cai, E.~Rutherford, D.~de~Las~Casas, L.~A. Hendricks, J.~Welbl, A.~Clark, T.~Hennigan, E.~Noland, K.~Millican, G.~van~den Driessche, B.~Damoc, A.~Guy, S.~Osindero, K.~Simonyan, E.~Elsen, J.~W. Rae, O.~Vinyals, and L.~Sifre, ``Training compute-optimal large language models,'' {arXiv:2203.15556 [cs.CL]}, 2022.

\bibitem{Finzi2025computeoptimalllmsprovablygeneralize}
M.~Finzi, S.~Kapoor, D.~Granziol, A.~Gu, C.~D. Sa, J.~Z. Kolter, and A.~G. Wilson, ``Compute-optimal {LLMs} provably generalize better with scale,'' {arXiv:2504.15208 [cs.LG]}, 2025.

\bibitem{EpochAI2024computegrows}
\BIBentryALTinterwordspacing
J.~Sevilla and E.~Roldán, ``Training compute of frontier {AI} models grows by 4-5x per year,'' 2024, accessed: 2025-04-26. [Online]. Available: \url{https://epoch.ai/blog/training-compute-of-frontier-ai-models-grows-by-4-5x-per-year}
\BIBentrySTDinterwordspacing

\bibitem{Pilz2024}
K.~Pilz, L.~Heim, and N.~Brown, ``Increased compute efficiency and the diffusion of {AI} capabilities,'' {arXiv:2311.15377 [cs.CY]}, 2024.

\bibitem{Hernandez2020}
D.~Hernandez and T.~B. Brown, ``Measuring the algorithmic efficiency of neural networks,'' {arXiv:2005.04305 [cs.LG]}, 2020.

\bibitem{Ho2024}
A.~Ho, T.~Besiroglu, E.~Erdil, D.~Owen, R.~Rahman, Z.~C. Guo, D.~Atkinson, N.~Thompson, and J.~Sevilla, ``Algorithmic progress in language models,'' arXiv:2403.05812 [cs.CL], 2024.

\bibitem{Anthropic2025ExtendedThinking}
{Anthropic}, ``{Claude’s extended thinking},'' \url{https://www.anthropic.com/news/visible-extended-thinking}, Feb. 2025, {A}nthropic blog post.

\bibitem{DeepSeekAI2025}
DeepSeek-AI, D.~Guo, D.~Yang, H.~Zhang, J.~Song, R.~Zhang, R.~Xu, Q.~Zhu, S.~Ma, P.~Wang, X.~Bi, X.~Zhang, X.~Yu, Y.~Wu, Z.~F. Wu, Z.~Gou, Z.~Shao, Z.~Li, Z.~Gao, A.~Liu, B.~Xue, B.~Wang, B.~Wu, B.~Feng, C.~Lu, C.~Zhao, C.~Deng, C.~Zhang, C.~Ruan, D.~Dai, D.~Chen, D.~Ji, E.~Li, F.~Lin, F.~Dai, F.~Luo, G.~Hao, G.~Chen, G.~Li, H.~Zhang, H.~Bao, H.~Xu, H.~Wang, H.~Ding, H.~Xin, H.~Gao, H.~Qu, H.~Li, J.~Guo, J.~Li, J.~Wang, J.~Chen, J.~Yuan, J.~Qiu, J.~Li, J.~L. Cai, J.~Ni, J.~Liang, J.~Chen, K.~Dong, K.~Hu, K.~Gao, K.~Guan, K.~Huang, K.~Yu, L.~Wang, L.~Zhang, L.~Zhao, L.~Wang, L.~Zhang, L.~Xu, L.~Xia, M.~Zhang, M.~Zhang, M.~Tang, M.~Li, M.~Wang, M.~Li, N.~Tian, P.~Huang, P.~Zhang, Q.~Wang, Q.~Chen, Q.~Du, R.~Ge, R.~Zhang, R.~Pan, R.~Wang, R.~J. Chen, R.~L. Jin, R.~Chen, S.~Lu, S.~Zhou, S.~Chen, S.~Ye, S.~Wang, S.~Yu, S.~Zhou, S.~Pan, S.~S. Li, S.~Zhou, S.~Wu, S.~Ye, T.~Yun, T.~Pei, T.~Sun, T.~Wang, W.~Zeng, W.~Zhao, W.~Liu, W.~Liang, W.~Gao, W.~Yu, W.~Zhang, W.~L. Xiao, W.~An, X.~Liu, X.~Wang, X.~Chen, X.~Nie,
  X.~Cheng, X.~Liu, X.~Xie, X.~Liu, X.~Yang, X.~Li, X.~Su, X.~Lin, X.~Q. Li, X.~Jin, X.~Shen, X.~Chen, X.~Sun, X.~Wang, X.~Song, X.~Zhou, X.~Wang, X.~Shan, Y.~K. Li, Y.~Q. Wang, Y.~X. Wei, Y.~Zhang, Y.~Xu, Y.~Li, Y.~Zhao, Y.~Sun, Y.~Wang, Y.~Yu, Y.~Zhang, Y.~Shi, Y.~Xiong, Y.~He, Y.~Piao, Y.~Wang, Y.~Tan, Y.~Ma, Y.~Liu, Y.~Guo, Y.~Ou, Y.~Wang, Y.~Gong, Y.~Zou, Y.~He, Y.~Xiong, Y.~Luo, Y.~You, Y.~Liu, Y.~Zhou, Y.~X. Zhu, Y.~Xu, Y.~Huang, Y.~Li, Y.~Zheng, Y.~Zhu, Y.~Ma, Y.~Tang, Y.~Zha, Y.~Yan, Z.~Z. Ren, Z.~Ren, Z.~Sha, Z.~Fu, Z.~Xu, Z.~Xie, Z.~Zhang, Z.~Hao, Z.~Ma, Z.~Yan, Z.~Wu, Z.~Gu, Z.~Zhu, Z.~Liu, Z.~Li, Z.~Xie, Z.~Song, Z.~Pan, Z.~Huang, Z.~Xu, Z.~Zhang, and Z.~Zhang, ``{DeepSeek-R1}: Incentivizing reasoning capability in {LLMs} via reinforcement learning,'' {arXiv:2501.12948 [cs.CL]}, 2025.

\bibitem{Kavukcuoglu2025Gemini25}
K.~Kavukcuoglu, ``{Gemini 2.5: Our most intelligent AI model},'' \url{https://blog.google/technology/google-deepmind/gemini-model-thinking-updates-march-2025/}, Mar. 2025, {G}oogle DeepMind blog post.

\bibitem{IBM2025Granite32}
K.~Soule and D.~Bergmann, ``{IBM Granite 3.2: Reasoning, vision, forecasting and more},'' \url{https://www.ibm.com/new/announcements/ibm-granite-3-2-open-source-reasoning-and-vision}, Feb. 2025, accessed: 2025-05-06.

\bibitem{Microsoft2025phi}
M.~Abdin, S.~Agarwal, A.~Awadallah, V.~Balachandran, H.~Behl, L.~Chen, G.~de~Rosa, S.~Gunasekar, M.~Javaheripi, N.~Joshi, P.~Kauffmann, Y.~Lara, C.~C.~T. Mendes, A.~Mitra, B.~Nushi, D.~Papailiopoulos, O.~Saarikivi, S.~Shah, V.~Shrivastava, V.~Vineet, Y.~Wu, S.~Yousefi, and G.~Zheng, ``{Phi-4-reasoning technical report},'' {arXiv:2504.21318 [cs.AI]}, 2025.

\bibitem{OpenAI2024o1}
\BIBentryALTinterwordspacing
{OpenAI}, ``{OpenAI o1 System Card},'' OpenAI, Tech. Rep., December 2024. [Online]. Available: \url{https://openai.com/index/openai-o1-system-card/}
\BIBentrySTDinterwordspacing

\bibitem{OpenAI2025o3}
------, ``{Introducing {OpenAI}} o3 and o4-mini,'' \url{https://openai.com/index/introducing-o3-and-o4-mini/}, Apr. 2025, openAI blog post.

\bibitem{xAI2025Grok3}
{xAI}, ``{Grok 3 Beta — The Age of Reasoning Agents},'' \url{https://x.ai/news/grok-3}, Feb. 2025, xAI news post.

\bibitem{DeVries2023growingenergy}
A.~De~Vries, ``{The growing energy footprint of artificial intelligence},'' \emph{Joule}, vol.~7, no.~10, pp. 2191--2194, 2023.

\bibitem{Luccioni2023bloom}
A.~S. Luccioni, S.~Viguier, and A.-L. Ligozat, ``Estimating the carbon footprint of {BLOOM}, a {176B} parameter language model,'' \emph{Journal of Machine Learning Research}, vol.~24, no. 253, pp. 1--15, 2023.

\bibitem{Patterson2022}
D.~Patterson, J.~Gonzalez, U.~Hölzle, Q.~Le, C.~Liang, L.-M. Munguia, D.~Rothchild, D.~So, M.~Texier, and J.~Dean, ``The carbon footprint of machine learning training will plateau, then shrink,'' {arXiv:2204.05149 [cs.LG]}, 2022.

\bibitem{Wu2022sustainable}
C.-J. Wu, R.~Raghavendra, U.~Gupta, B.~Acun, N.~Ardalani, K.~Maeng, G.~Chang, F.~Aga, J.~Huang, C.~Bai, M.~Gschwind, A.~Gupta, M.~Ott, A.~Melnikov, S.~Candido, D.~Brooks, G.~Chauhan, B.~Lee, H.-H. Lee, B.~Akyildiz, M.~Balandat, J.~Spisak, R.~Jain, M.~Rabbat, and K.~Hazelwood, ``Sustainable {AI}: Environmental implications, challenges and opportunities,'' in \emph{Proceedings of Machine Learning and Systems}, vol.~4, 2022, pp. 795--813.

\bibitem{Barclays2024}
T.~O'Malley, R.~Sandler, and T.~Young, ``The next wave of {AI}: Demand and adoption,'' \url{https://www.ib.barclays/our-insights/3-point-perspective/the-next-wave-of-AI-demand-and-adoption.html}, 2024, accessed: 2025-04-25.

\bibitem{Luccioni2025Jevons}
A.~S. Luccioni, E.~Strubell, and K.~Crawford, ``From efficiency gains to rebound effects: The problem of {J}evons' {P}aradox in {AI}'s polarized environmental debate,'' {arXiv:2501.16548 [cs.CY]}, 2025.

\bibitem{Jevons1866}
W.~S. Jevons, \emph{The Coal Question; An Inquiry concerning the Progress of the Nation, and the Probable Exhaustion of our Coal-mines}.\hskip 1em plus 0.5em minus 0.4em\relax Macmillan, 1866.

\bibitem{Verge2024chatgpt}
\BIBentryALTinterwordspacing
E.~Roth, ``{1B} user messages sent on {ChatGPT} every day,'' The Verge, 2024, accessed: 2025-04-22. [Online]. Available: \url{https://www.theverge.com/2024/12/4/24313097/chatgpt-300-million-weekly-users}
\BIBentrySTDinterwordspacing

\bibitem{Verge2025chatgpt}
\BIBentryALTinterwordspacing
K.~Robison, ``{ChatGPT} added one million users in the last hour,'' The Verge, 2025, accessed: 2025-04-22. [Online]. Available: \url{https://www.theverge.com/openai/639960/chatgpt-added-one-million-users-in-the-last-hour}
\BIBentrySTDinterwordspacing

\bibitem{DemandSage2025chatgpt}
\BIBentryALTinterwordspacing
{DemandSage}, ``{ChatGPT} statistics and user trends (2025),'' 2025, accessed: 2025-04-22. [Online]. Available: \url{https://www.demandsage.com/chatgpt-statistics}
\BIBentrySTDinterwordspacing

\bibitem{Verdecchia2023greenAI}
R.~Verdecchia, J.~Sallou, and L.~Cruz, ``{A systematic review of Green AI},'' \emph{Wiley Interdisciplinary Reviews: Data Mining and Knowledge Discovery}, vol.~13, no.~4, p. e1507, 2023.

\bibitem{Campbell2002}
M.~Campbell, A.~J. Hoane, Jr., and F.-H. Hsu, ``{Deep Blue},'' \emph{Artificial Intelligence}, vol. 134, no.~1, pp. 57--83, 2002.

\bibitem{Silver2016AlphaGo}
D.~Silver, A.~Huang, C.~J. Maddison, A.~Guez, L.~Sifre, G.~Van Den~Driessche, J.~Schrittwieser, I.~Antonoglou, V.~Panneershelvam, M.~Lanctot \emph{et~al.}, ``Mastering the game of {Go} with deep neural networks and tree search,'' \emph{Nature}, vol. 529, no. 7587, pp. 484--489, 2016.

\bibitem{Silver2017AlphaGoZero}
D.~Silver, J.~Schrittwieser, K.~Simonyan, I.~Antonoglou, A.~Huang, A.~Guez, T.~Hubert, L.~Baker, M.~Lai, A.~Bolton \emph{et~al.}, ``Mastering the game of {Go} without human knowledge,'' \emph{Nature}, vol. 550, no. 7676, pp. 354--359, 2017.

\bibitem{Silver2017ChessShogiGo}
D.~Silver, T.~Hubert, J.~Schrittwieser, I.~Antonoglou, M.~Lai, A.~Guez, M.~Lanctot, L.~Sifre, D.~Kumaran, T.~Graepel, T.~Lillicrap, K.~Simonyan, and D.~Hassabis, ``Mastering {C}hess and {S}hogi by self-play with a general reinforcement learning algorithm,'' {arXiv:1712.01815 [cs.AI]}, 2017.

\bibitem{Jones2021}
A.~L. Jones, ``Scaling scaling laws with board games,'' {arXiv:2104.03113 [cs.LG]}, 2021.

\bibitem{Brown2017}
N.~Brown and T.~Sandholm, ``Safe and nested subgame solving for imperfect-information games,'' {arXiv:1705.02955 [cs.AI]}, 2017.

\bibitem{Meta2022}
{Meta Fundamental AI Research Diplomacy Team (FAIR)}, A.~Bakhtin, N.~Brown, E.~Dinan, G.~Farina, C.~Flaherty, D.~Fried, A.~Goff, J.~Gray, H.~Hu \emph{et~al.}, ``Human-level play in the game of {D}iplomacy by combining language models with strategic reasoning,'' \emph{Science}, vol. 378, no. 6624, pp. 1067--1074, 2022.

\bibitem{Wei2022Emergence}
J.~Wei, Y.~Tay, R.~Bommasani, C.~Raffel, B.~Zoph, S.~Borgeaud, D.~Yogatama, M.~Bosma, D.~Zhou, D.~Metzler, E.~H. Chi, T.~Hashimoto, O.~Vinyals, P.~Liang, J.~Dean, and W.~Fedus, ``Emergent abilities of large language models,'' {arXiv:2206.07682 [cs.CL]}, 2022.

\bibitem{Nayak2025}
A.~K. Nayak and L.~R. Varshney, ``An information theory of compute-optimal size scaling, emergence, and plateaus in language models,'' in \emph{Compression Workshop @ NeurIPS 2024}, 2024.

\bibitem{MMLULeaderboard}
\BIBentryALTinterwordspacing
{Papers with Code}, ``{Multi-task Language Understanding on MMLU Leaderboard},'' 2025, accessed: 2025-04-28. [Online]. Available: \url{https://paperswithcode.com/sota/multi-task-language-understanding-on-mmlu}
\BIBentrySTDinterwordspacing

\bibitem{Hendrycks2021MMLU}
D.~Hendrycks, C.~Burns, S.~Basart, A.~Zou, M.~Mazeika, D.~Song, and J.~Steinhardt, ``Measuring massive multitask language understanding,'' in \emph{Proceedings of the International Conference on Learning Representations (ICLR)}, 2021.

\bibitem{Schaeffer2023mirage}
R.~Schaeffer, B.~Miranda, and S.~Koyejo, ``Are emergent abilities of large language models a mirage?'' in \emph{Advances in Neural Information Processing Systems}, 2023, vol.~36, pp. 55\,565--55\,581.

\bibitem{Michaud2023}
E.~Michaud, Z.~Liu, U.~Girit, and M.~Tegmark, ``{The quantization model of neural scaling},'' in \emph{Advances in Neural Information Processing Systems}, 2023, vol.~36, pp. 28\,699--28\,722.

\bibitem{Ameisen2025Circuit}
\BIBentryALTinterwordspacing
E.~Ameisen, J.~Lindsey, A.~Pearce, W.~Gurnee, N.~L. Turner, B.~Chen, C.~Citro, D.~Abrahams, S.~Carter, B.~Hosmer, J.~Marcus, M.~Sklar, A.~Templeton, T.~Bricken, C.~McDougall, H.~Cunningham, T.~Henighan, A.~Jermyn, A.~Jones, A.~Persic, Z.~Qi, T.~Ben~Thompson, S.~Zimmerman, K.~Rivoire, T.~Conerly, C.~Olah, and J.~Batson, ``Circuit tracing: Revealing computational graphs in language models,'' \emph{Transformer Circuits Thread}, 2025. [Online]. Available: \url{https://transformer-circuits.pub/2025/attribution-graphs/methods.html}
\BIBentrySTDinterwordspacing

\bibitem{Bengio2009curriculum}
Y.~Bengio, J.~Louradour, R.~Collobert, and J.~Weston, ``Curriculum learning,'' in \emph{Proceedings of the International Conference on Machine Learning (ICML)}, 2009, pp. 41--48.

\bibitem{Arora2023}
S.~Arora and A.~Goyal, ``A theory for emergence of complex skills in language models,'' {arXiv:2307.15936 [cs.LG]}, 2023.

\bibitem{Liao2025semantic}
K.-Y. Liao, C.-S. Chang, and Y.-W.~P. Hong, ``{A mathematical theory for learning semantic languages by abstract learners},'' \emph{IEEE Journal on Selected Areas in Communications}, 2025.

\bibitem{Yu2023skillmix}
D.~Yu, S.~Kaur, A.~Gupta, J.~Brown-Cohen, A.~Goyal, and S.~Arora, ``{Skill-Mix}: a flexible and expandable family of evaluations for {AI} models,'' in \emph{Proceedings of the International Conference on Learning Representations (ICLR)}, 2024.

\bibitem{Gallistel2004plateaus}
C.~R. Gallistel, S.~Fairhurst, and P.~Balsam, ``{The learning curve: implications of a quantitative analysis},'' \emph{Proceedings of the National Academy of Sciences}, vol. 101, no.~36, pp. 13\,124--13\,131, 2004.

\bibitem{Gray2017plateausandleaps}
W.~D. Gray and J.~K. Lindstedt, ``{Plateaus, dips, and leaps: Where to look for inventions and discoveries during skilled performance},'' \emph{Cognitive Science}, vol.~41, no.~7, pp. 1838--1870, 2017.

\bibitem{Lu2025plateausandleaps}
M.~Lu, T.~Marghetis, and V.~C. Yang, ``{A first-principles mathematical model integrates the disparate timescales of human learning},'' \emph{npj Complexity}, vol.~2, no.~1, p.~15, 2025.

\bibitem{Sourlas1989spin}
N.~Sourlas, ``{Spin-glass models as error-correcting codes},'' \emph{Nature}, vol. 339, no. 6227, pp. 693--695, 1989.

\bibitem{Newell1994unified}
A.~Newell, \emph{Unified Theories of Cognition}.\hskip 1em plus 0.5em minus 0.4em\relax Harvard University Press, 1994.

\bibitem{Posfai2016network}
A.-L. Barab{\'a}si, \emph{Network Science}.\hskip 1em plus 0.5em minus 0.4em\relax Cambridge University Press, 2016.

\bibitem{Cortes1993scalinglaws}
C.~Cortes, L.~D. Jackel, S.~Solla, V.~Vapnik, and J.~Denker, ``Learning curves: Asymptotic values and rate of convergence,'' in \emph{Advances in Neural Information Processing Systems}, 1993, vol.~6, pp. 327--334.

\bibitem{Hestness2017scalinglaws}
J.~Hestness, S.~Narang, N.~Ardalani, G.~Diamos, H.~Jun, H.~Kianinejad, M.~M.~A. Patwary, Y.~Yang, and Y.~Zhou, ``Deep learning scaling is predictable, empirically,'' {arXiv:1712.00409 [cs.LG]}, 2017.

\bibitem{Rosenfeld2019scalinglaws}
J.~S. Rosenfeld, A.~Rosenfeld, Y.~Belinkov, and N.~Shavit, ``A constructive prediction of the generalization error across scales,'' {arXiv:1909.12673 [cs.LG]}, 2019.

\bibitem{Brown2020fewshotlearners}
T.~B. Brown, B.~Mann, N.~Ryder, M.~Subbiah, J.~Kaplan, P.~Dhariwal, A.~Neelakantan, P.~Shyam, G.~Sastry, A.~Askell, S.~Agarwal, A.~Herbert-Voss, G.~Krueger, T.~Henighan, R.~Child, A.~Ramesh, D.~M. Ziegler, J.~Wu, C.~Winter, C.~Hesse, M.~Chen, E.~Sigler, M.~Litwin, S.~Gray, B.~Chess, J.~Clark, C.~Berner, S.~McCandlish, A.~Radford, I.~Sutskever, and D.~Amodei, ``Language models are few-shot learners,'' {arXiv:2005.14165 [cs.CL]}, 2020.

\bibitem{Shannon1951LanguageEntropy}
C.~E. Shannon, ``Prediction and entropy of printed {E}nglish,'' \emph{Bell System Technical Journal}, vol.~30, no.~1, pp. 50--64, 1951.

\bibitem{Bahri2024explainingscaling}
Y.~Bahri, E.~Dyer, J.~Kaplan, J.~Lee, and U.~Sharma, ``Explaining neural scaling laws,'' \emph{Proceedings of the National Academy of Sciences}, vol. 121, no.~27, p. e2311878121, Jun. 2024.

\bibitem{Wu2024towards}
C.~Wu and R.~Tang, ``Towards a universal scaling law of {LLM} training and inference,'' ScienceOpen Preprints, 2024.

\bibitem{Meta2024llama3}
A.~Grattafiori, A.~Dubey, A.~Jauhri, A.~Pandey, A.~Kadian, A.~Al-Dahle, A.~Letman, A.~Mathur, A.~Schelten, A.~Vaughan, A.~Yang, A.~Fan, A.~Goyal, A.~Hartshorn, A.~Yang, A.~Mitra, A.~Sravankumar, A.~Korenev, A.~Hinsvark, A.~Rao, A.~Zhang, A.~Rodriguez, A.~Gregerson, A.~Spataru, B.~Roziere, B.~Biron, B.~Tang, B.~Chern, C.~Caucheteux, C.~Nayak, C.~Bi, C.~Marra, C.~McConnell, C.~Keller, C.~Touret, C.~Wu, C.~Wong, C.~C. Ferrer, C.~Nikolaidis, D.~Allonsius, D.~Song, D.~Pintz, D.~Livshits, D.~Wyatt, D.~Esiobu, D.~Choudhary, D.~Mahajan, D.~Garcia-Olano, D.~Perino, D.~Hupkes, E.~Lakomkin, E.~AlBadawy, E.~Lobanova, E.~Dinan, E.~M. Smith, F.~Radenovic, F.~Guzmán, F.~Zhang, G.~Synnaeve, G.~Lee, G.~L. Anderson, G.~Thattai, G.~Nail, G.~Mialon, G.~Pang, G.~Cucurell, H.~Nguyen, H.~Korevaar, H.~Xu, H.~Touvron, I.~Zarov, I.~A. Ibarra, I.~Kloumann, I.~Misra, I.~Evtimov, J.~Zhang, J.~Copet, J.~Lee, J.~Geffert, J.~Vranes, J.~Park, J.~Mahadeokar, J.~Shah, J.~van~der Linde, J.~Billock, J.~Hong, J.~Lee, J.~Fu, J.~Chi, J.~Huang,
  J.~Liu, J.~Wang, J.~Yu, J.~Bitton, J.~Spisak, J.~Park, J.~Rocca, J.~Johnstun, J.~Saxe, J.~Jia, K.~V. Alwala, K.~Prasad, K.~Upasani, K.~Plawiak, K.~Li, K.~Heafield, K.~Stone, K.~El-Arini, K.~Iyer, K.~Malik, K.~Chiu, K.~Bhalla, K.~Lakhotia, L.~Rantala-Yeary, L.~van~der Maaten, L.~Chen, L.~Tan, L.~Jenkins, L.~Martin, L.~Madaan, L.~Malo, L.~Blecher, L.~Landzaat, L.~de~Oliveira, M.~Muzzi, M.~Pasupuleti, M.~Singh, M.~Paluri, M.~Kardas, M.~Tsimpoukelli, M.~Oldham, M.~Rita, M.~Pavlova, M.~Kambadur, M.~Lewis, M.~Si, M.~K. Singh, M.~Hassan, N.~Goyal, N.~Torabi, N.~Bashlykov, N.~Bogoychev, N.~Chatterji, N.~Zhang, O.~Duchenne, O.~Çelebi, P.~Alrassy, P.~Zhang, P.~Li, P.~Vasic, P.~Weng, P.~Bhargava, P.~Dubal, P.~Krishnan, P.~S. Koura, P.~Xu, Q.~He, Q.~Dong, R.~Srinivasan, R.~Ganapathy, R.~Calderer, R.~S. Cabral, R.~Stojnic, R.~Raileanu, R.~Maheswari, R.~Girdhar, R.~Patel, R.~Sauvestre, R.~Polidoro, R.~Sumbaly, R.~Taylor, R.~Silva, R.~Hou, R.~Wang, S.~Hosseini, S.~Chennabasappa, S.~Singh, S.~Bell, S.~S. Kim, S.~Edunov,
  S.~Nie, S.~Narang, S.~Raparthy, S.~Shen, S.~Wan, S.~Bhosale, S.~Zhang, S.~Vandenhende, S.~Batra, S.~Whitman, S.~Sootla, S.~Collot, S.~Gururangan, S.~Borodinsky, T.~Herman, T.~Fowler, T.~Sheasha, T.~Georgiou, T.~Scialom, T.~Speckbacher, T.~Mihaylov, T.~Xiao, U.~Karn, V.~Goswami, V.~Gupta, V.~Ramanathan, V.~Kerkez, V.~Gonguet, V.~Do, V.~Vogeti, V.~Albiero, V.~Petrovic, W.~Chu, W.~Xiong, W.~Fu, W.~Meers, X.~Martinet, X.~Wang, X.~Wang, X.~E. Tan, X.~Xia, X.~Xie, X.~Jia, X.~Wang, Y.~Goldschlag, Y.~Gaur, Y.~Babaei, Y.~Wen, Y.~Song, Y.~Zhang, Y.~Li, Y.~Mao, Z.~D. Coudert, Z.~Yan, Z.~Chen, Z.~Papakipos, A.~Singh, A.~Srivastava, A.~Jain, A.~Kelsey, A.~Shajnfeld, A.~Gangidi, A.~Victoria, A.~Goldstand, A.~Menon, A.~Sharma, A.~Boesenberg, A.~Baevski, A.~Feinstein, A.~Kallet, A.~Sangani, A.~Teo, A.~Yunus, A.~Lupu, A.~Alvarado, A.~Caples, A.~Gu, A.~Ho, A.~Poulton, A.~Ryan, A.~Ramchandani, A.~Dong, A.~Franco, A.~Goyal, A.~Saraf, A.~Chowdhury, A.~Gabriel, A.~Bharambe, A.~Eisenman, A.~Yazdan, B.~James, B.~Maurer,
  B.~Leonhardi, B.~Huang, B.~Loyd, B.~D. Paola, B.~Paranjape, B.~Liu, B.~Wu, B.~Ni, B.~Hancock, B.~Wasti, B.~Spence, B.~Stojkovic, B.~Gamido, B.~Montalvo, C.~Parker, C.~Burton, C.~Mejia, C.~Liu, C.~Wang, C.~Kim, C.~Zhou, C.~Hu, C.-H. Chu, C.~Cai, C.~Tindal, C.~Feichtenhofer, C.~Gao, D.~Civin, D.~Beaty, D.~Kreymer, D.~Li, D.~Adkins, D.~Xu, D.~Testuggine, D.~David, D.~Parikh, D.~Liskovich, D.~Foss, D.~Wang, D.~Le, D.~Holland, E.~Dowling, E.~Jamil, E.~Montgomery, E.~Presani, E.~Hahn, E.~Wood, E.-T. Le, E.~Brinkman, E.~Arcaute, E.~Dunbar, E.~Smothers, F.~Sun, F.~Kreuk, F.~Tian, F.~Kokkinos, F.~Ozgenel, F.~Caggioni, F.~Kanayet, F.~Seide, G.~M. Florez, G.~Schwarz, G.~Badeer, G.~Swee, G.~Halpern, G.~Herman, G.~Sizov, Guangyi, Zhang, G.~Lakshminarayanan, H.~Inan, H.~Shojanazeri, H.~Zou, H.~Wang, H.~Zha, H.~Habeeb, H.~Rudolph, H.~Suk, H.~Aspegren, H.~Goldman, H.~Zhan, I.~Damlaj, I.~Molybog, I.~Tufanov, I.~Leontiadis, I.-E. Veliche, I.~Gat, J.~Weissman, J.~Geboski, J.~Kohli, J.~Lam, J.~Asher, J.-B. Gaya, J.~Marcus,
  J.~Tang, J.~Chan, J.~Zhen, J.~Reizenstein, J.~Teboul, J.~Zhong, J.~Jin, J.~Yang, J.~Cummings, J.~Carvill, J.~Shepard, J.~McPhie, J.~Torres, J.~Ginsburg, J.~Wang, K.~Wu, K.~H. U, K.~Saxena, K.~Khandelwal, K.~Zand, K.~Matosich, K.~Veeraraghavan, K.~Michelena, K.~Li, K.~Jagadeesh, K.~Huang, K.~Chawla, K.~Huang, L.~Chen, L.~Garg, L.~A, L.~Silva, L.~Bell, L.~Zhang, L.~Guo, L.~Yu, L.~Moshkovich, L.~Wehrstedt, M.~Khabsa, M.~Avalani, M.~Bhatt, M.~Mankus, M.~Hasson, M.~Lennie, M.~Reso, M.~Groshev, M.~Naumov, M.~Lathi, M.~Keneally, M.~Liu, M.~L. Seltzer, M.~Valko, M.~Restrepo, M.~Patel, M.~Vyatskov, M.~Samvelyan, M.~Clark, M.~Macey, M.~Wang, M.~J. Hermoso, M.~Metanat, M.~Rastegari, M.~Bansal, N.~Santhanam, N.~Parks, N.~White, N.~Bawa, N.~Singhal, N.~Egebo, N.~Usunier, N.~Mehta, N.~P. Laptev, N.~Dong, N.~Cheng, O.~Chernoguz, O.~Hart, O.~Salpekar, O.~Kalinli, P.~Kent, P.~Parekh, P.~Saab, P.~Balaji, P.~Rittner, P.~Bontrager, P.~Roux, P.~Dollar, P.~Zvyagina, P.~Ratanchandani, P.~Yuvraj, Q.~Liang, R.~Alao, R.~Rodriguez,
  R.~Ayub, R.~Murthy, R.~Nayani, R.~Mitra, R.~Parthasarathy, R.~Li, R.~Hogan, R.~Battey, R.~Wang, R.~Howes, R.~Rinott, S.~Mehta, S.~Siby, S.~J. Bondu, S.~Datta, S.~Chugh, S.~Hunt, S.~Dhillon, S.~Sidorov, S.~Pan, S.~Mahajan, S.~Verma, S.~Yamamoto, S.~Ramaswamy, S.~Lindsay, S.~Lindsay, S.~Feng, S.~Lin, S.~C. Zha, S.~Patil, S.~Shankar, S.~Zhang, S.~Zhang, S.~Wang, S.~Agarwal, S.~Sajuyigbe, S.~Chintala, S.~Max, S.~Chen, S.~Kehoe, S.~Satterfield, S.~Govindaprasad, S.~Gupta, S.~Deng, S.~Cho, S.~Virk, S.~Subramanian, S.~Choudhury, S.~Goldman, T.~Remez, T.~Glaser, T.~Best, T.~Koehler, T.~Robinson, T.~Li, T.~Zhang, T.~Matthews, T.~Chou, T.~Shaked, V.~Vontimitta, V.~Ajayi, V.~Montanez, V.~Mohan, V.~S. Kumar, V.~Mangla, V.~Ionescu, V.~Poenaru, V.~T. Mihailescu, V.~Ivanov, W.~Li, W.~Wang, W.~Jiang, W.~Bouaziz, W.~Constable, X.~Tang, X.~Wu, X.~Wang, X.~Wu, X.~Gao, Y.~Kleinman, Y.~Chen, Y.~Hu, Y.~Jia, Y.~Qi, Y.~Li, Y.~Zhang, Y.~Zhang, Y.~Adi, Y.~Nam, Yu, Wang, Y.~Zhao, Y.~Hao, Y.~Qian, Y.~Li, Y.~He, Z.~Rait, Z.~DeVito,
  Z.~Rosnbrick, Z.~Wen, Z.~Yang, Z.~Zhao, and Z.~Ma, ``The {L}lama 3 herd of models,'' {arXiv:2407.21783 [cs.AI]}, 2024.

\bibitem{Xiao2024densinglawllms}
C.~Xiao, J.~Cai, W.~Zhao, G.~Zeng, B.~Lin, J.~Zhou, Z.~Zheng, X.~Han, Z.~Liu, and M.~Sun, ``Densing law of {LLMs},'' {arXiv:2412.04315 [cs.AI]}, 2024.

\bibitem{Owen2024downstreamtasksigmoid}
D.~Owen, ``How predictable is language model benchmark performance?'' {arXiv:2401.04757 [cs.LG]}, 2024.

\bibitem{Ruan2024observationalscalinglaws}
Y.~Ruan, C.~J. Maddison, and T.~Hashimoto, ``Observational scaling laws and the predictability of language model performance,'' {arXiv:2405.10938 [cs.LG]}, 2024.

\bibitem{Gadre2024}
S.~Y. Gadre, G.~Smyrnis, V.~Shankar, S.~Gururangan, M.~Wortsman, R.~Shao, J.~Mercat, A.~Fang, J.~Li, S.~Keh, R.~Xin, M.~Nezhurina, I.~Vasiljevic, J.~Jitsev, L.~Soldaini, A.~G. Dimakis, G.~Ilharco, P.~W. Koh, S.~Song, T.~Kollar, Y.~Carmon, A.~Dave, R.~Heckel, N.~Muennighoff, and L.~Schmidt, ``{Language models scale reliably with over-training and on downstream tasks},'' {arXiv:2403.08540 [cs.CL]}, 2024.

\bibitem{Caballero2023brokenneuralscalinglaws}
E.~Caballero, K.~Gupta, I.~Rish, and D.~Krueger, ``Broken neural scaling laws,'' {arXiv:2210.14891 [cs.LG]}, 2023.

\bibitem{Isik2025scalinglawsdownstreamtask}
B.~Isik, N.~Ponomareva, H.~Hazimeh, D.~Paparas, S.~Vassilvitskii, and S.~Koyejo, ``Scaling laws for downstream task performance in machine translation,'' {arXiv:2402.04177 [cs.CL]}, 2025.

\bibitem{Abnar2021upstreamtodownstreamsaturation}
S.~Abnar, M.~Dehghani, B.~Neyshabur, and H.~Sedghi, ``Exploring the limits of large scale pre-training,'' {arXiv:2110.02095 [cs.LG]}, 2021.

\bibitem{Diaz2024scalinglawsevaluations}
F.~Diaz and M.~Madaio, ``Scaling laws do not scale,'' {arXiv:2307.03201 [cs.LG]}, 2024.

\bibitem{Liu2025notjustscalinglaws}
E.~Liu, A.~Bertsch, L.~Sutawika, L.~Tjuatja, P.~Fernandes, L.~Marinov, M.~Chen, S.~Singhal, C.~Lawrence, A.~Raghunathan, K.~Gashteovski, and G.~Neubig, ``Not-just-scaling laws: Towards a better understanding of the downstream impact of language model design decisions,'' {arXiv:2503.03862 [cs.CL]}, 2025.

\bibitem{Liu2022betterdownstream}
H.~Liu, S.~M. Xie, Z.~Li, and T.~Ma, ``Same pre-training loss, better downstream: Implicit bias matters for language models,'' {arXiv:2210.14199 [cs.LG]}, 2022.

\bibitem{Springer2025overtrainedfinetuning}
J.~M. Springer, S.~Goyal, K.~Wen, T.~Kumar, X.~Yue, S.~Malladi, G.~Neubig, and A.~Raghunathan, ``Overtrained language models are harder to fine-tune,'' {arXiv:2503.19206 [cs.CL]}, 2025.

\bibitem{Chen2025rethinkingfinetuning}
F.~Chen, A.~Raventos, N.~Cheng, S.~Ganguli, and S.~Druckmann, ``Rethinking fine-tuning when scaling test-time compute: Limiting confidence improves mathematical reasoning,'' {arXiv:2502.07154 [cs.LG]}, 2025.

\bibitem{Lin2022inversescalingtruth}
S.~Lin, J.~Hilton, and O.~Evans, ``{TruthfulQA}: Measuring how models mimic human falsehoods,'' {arXiv:2109.07958 [cs.CL]}, 2022.

\bibitem{Parrish2022biggermorebias}
A.~Parrish, A.~Chen, N.~Nangia, V.~Padmakumar, J.~Phang, J.~Thompson, P.~M. Htut, and S.~R. Bowman, ``{BBQ}: A hand-built bias benchmark for question answering,'' {arXiv:2110.08193 [cs.CL]}, 2022.

\bibitem{Wei2023inversescaling}
J.~Wei, N.~Kim, Y.~Tay, and Q.~V. Le, ``Inverse scaling can become {U}-shaped,'' {arXiv:2211.02011 [cs.CL]}, 2023.

\bibitem{Muennighoff2025}
N.~Muennighoff, Z.~Yang, W.~Shi, X.~L. Li, L.~Fei-Fei, H.~Hajishirzi, L.~Zettlemoyer, P.~Liang, E.~Candès, and T.~Hashimoto, ``{s1}: Simple test-time scaling,'' {arXiv:2501.19393 [cs.CL]}, 2025.

\bibitem{Wang2025RLoneexample}
Y.~Wang, Q.~Yang, Z.~Zeng, L.~Ren, L.~Liu, B.~Peng, H.~Cheng, X.~He, K.~Wang, J.~Gao, W.~Chen, S.~Wang, S.~S. Du, and Y.~Shen, ``Reinforcement learning for reasoning in large language models with one training example,'' {arXiv:2504.20571 [cs.LG]}, 2025.

\bibitem{Kojima2023}
T.~Kojima, S.~S. Gu, M.~Reid, Y.~Matsuo, and Y.~Iwasawa, ``Large language models are zero-shot reasoners,'' {arXiv:2205.11916 [cs.CL]}, 2023.

\bibitem{Suzgun2022bigbenchCoT}
M.~Suzgun, N.~Scales, N.~Schärli, S.~Gehrmann, Y.~Tay, H.~W. Chung, A.~Chowdhery, Q.~V. Le, E.~H. Chi, D.~Zhou, and J.~Wei, ``Challenging {BIG-Bench} tasks and whether chain-of-thought can solve them,'' {arXiv:2210.09261 [cs.CL]}, 2022.

\bibitem{Wang2024reasoninggraph}
X.~Wang, A.~Amayuelas, K.~Zhang, L.~Pan, W.~Chen, and W.~Y. Wang, ``Understanding reasoning ability of language models from the perspective of reasoning paths aggregation,'' arXiv:2402.03268 [cs.LG], 2024.

\bibitem{Sutskever2024NeurIPSTalk}
\BIBentryALTinterwordspacing
I.~Sutskever, ``Sequence to sequence learning with neural networks: What a decade,'' NeurIPS 2024 Test of Time Award Talk, Dec. 2024. [Online]. Available: \url{https://www.youtube.com/watch?v=HlGi4OOuZyw}
\BIBentrySTDinterwordspacing

\bibitem{Byrnes2023}
S.~Byrnes, ``{{AI doom from an LLM-plateau-ist perspective}},'' \url{https://www.lesswrong.com/posts/[insert-slug-if-available]}, 2023, online essay.

\bibitem{Ritter2024}
G.~Ritter and W.~Lu, ``The first wave of {AI} innovation is over. here’s what comes next,'' \emph{Fast Company}, Jul. 2024.

\bibitem{Brown2024TEDTalk}
\BIBentryALTinterwordspacing
N.~Brown, ``{AI} won’t plateau — if we give it time to think,'' TED Talk, Dec. 2024. [Online]. Available: \url{https://www.youtube.com/watch?v=MG9oqntiJKg}
\BIBentrySTDinterwordspacing

\bibitem{Muennighoff2023}
N.~Muennighoff, A.~Rush, B.~Barak, T.~Le~Scao, N.~Tazi, A.~Piktus, S.~Pyysalo, T.~Wolf, and C.~A. Raffel, ``{Scaling data-constrained language models},'' in \emph{Advances in Neural Information Processing Systems}, 2023, vol.~36, pp. 50\,358--50\,376.

\bibitem{Goyal2024datafiltering}
S.~Goyal, P.~Maini, Z.~C. Lipton, A.~Raghunathan, and J.~Z. Kolter, ``Scaling laws for data filtering -- data curation cannot be compute agnostic,'' {arXiv:2404.07177 [cs.LG]}, 2024.

\bibitem{Cottier2024}
B.~Cottier, R.~Rahman, L.~Fattorini, N.~Maslej, and D.~Owen, ``The rising costs of training frontier {AI} models,'' {arXiv:2405.21015 [cs.CY]}, 2024.

\bibitem{Turpin2023}
M.~Turpin, J.~Michael, E.~Perez, and S.~R. Bowman, ``Language models don't always say what they think: Unfaithful explanations in chain-of-thought prompting,'' {arXiv:2305.04388 [cs.CL]}, 2023.

\bibitem{Saparov2023}
A.~Saparov and H.~He, ``Language models are greedy reasoners: A systematic formal analysis of chain-of-thought,'' {arXiv:2210.01240 [cs.CL]}, 2023.

\bibitem{Xiao2024efficientstreaming}
G.~Xiao, Y.~Tian, B.~Chen, S.~Han, and M.~Lewis, ``Efficient streaming language models with attention sinks,'' {arXiv:2309.17453 [cs.CL]}, 2024.

\bibitem{Yan2025attentionCoT}
S.~Yan, C.~Shen, W.~Wang, L.~Xie, J.~Liu, and J.~Ye, ``Don't take things out of context: Attention intervention for enhancing chain-of-thought reasoning in large language models,'' {arXiv:2503.11154 [cs.CL]}, 2025.

\bibitem{Zheng2025cursecot}
T.~Zheng, Y.~Chen, C.~Li, C.~Li, Q.~Zong, H.~Shi, B.~Xu, Y.~Song, G.~Y. Wong, and S.~See, ``The curse of {CoT}: On the limitations of chain-of-thought in in-context learning,'' {arXiv:2504.05081 [cs.CL]}, 2025.

\bibitem{Wu2025moreisless}
Y.~Wu, Y.~Wang, T.~Du, S.~Jegelka, and Y.~Wang, ``When more is less: Understanding chain-of-thought length in {LLMs},'' {arXiv:2502.07266 [cs.AI]}, 2025.

\bibitem{Lightman2023}
H.~Lightman, V.~Kosaraju, Y.~Burda, H.~Edwards, B.~Baker, T.~Lee, J.~Leike, J.~Schulman, I.~Sutskever, and K.~Cobbe, ``Let's verify step by step,'' {arXiv:2305.20050 [cs.LG]}, 2023.

\bibitem{Cobbe2021}
K.~Cobbe, V.~Kosaraju, M.~Bavarian, M.~Chen, H.~Jun, L.~Kaiser, M.~Plappert, J.~Tworek, J.~Hilton, R.~Nakano, C.~Hesse, and J.~Schulman, ``Training verifiers to solve math word problems,'' {arXiv:2110.14168 [cs.LG]}, 2021.

\bibitem{Singhi2025solveverifycomputeoptimal}
N.~Singhi, H.~Bansal, A.~Hosseini, A.~Grover, K.-W. Chang, M.~Rohrbach, and A.~Rohrbach, ``When to solve, when to verify: Compute-optimal problem solving and generative verification for {LLM} reasoning,'' {arXiv:2504.01005 [cs.CL]}, 2025.

\bibitem{Wang2022self}
X.~Wang, J.~Wei, D.~Schuurmans, Q.~Le, E.~Chi, S.~Narang, A.~Chowdhery, and D.~Zhou, ``{Self-consistency improves chain of thought reasoning in language models},'' {arXiv:2203.11171 [cs.CL]}, 2022.

\bibitem{Lewkowycz2022}
A.~Lewkowycz, A.~Andreassen, D.~Dohan, E.~Dyer, H.~Michalewski, V.~Ramasesh, A.~Slone, C.~Anil, I.~Schlag, T.~Gutman-Solo, Y.~Wu, B.~Neyshabur, G.~Gur-Ari, and V.~Misra, ``Solving quantitative reasoning problems with language models,'' in \emph{Advances in Neural Information Processing Systems}, 2022, vol.~35, pp. 3843--3857.

\bibitem{Madaan2023}
A.~Madaan, N.~Tandon, P.~Gupta, S.~Hallinan, L.~Gao, S.~Wiegreffe, U.~Alon, N.~Dziri, S.~Prabhumoye, Y.~Yang, S.~Gupta, B.~P. Majumder, K.~Hermann, S.~Welleck, A.~Yazdanbakhsh, and P.~Clark, ``{Self-Refine}: Iterative refinement with self-feedback,'' {arXiv:2303.17651 [cs.CL]}, 2023.

\bibitem{Erol2025economicLLM}
M.~H. Erol, B.~El, M.~Suzgun, M.~Yuksekgonul, and J.~Zou, ``{Cost-of-Pass}: An economic framework for evaluating language models,'' {arXiv:2504.13359 [cs.AI]}, 2025.

\bibitem{Hernandez2021TransferScaling}
D.~Hernandez, J.~Kaplan, T.~Henighan, and S.~McCandlish, ``Scaling laws for transfer,'' {arXiv:2102.01293 [cs.LG]}, 2021.

\bibitem{Cherti2023constrastivescaling}
M.~Cherti, R.~Beaumont, R.~Wightman, M.~Wortsman, G.~Ilharco, C.~Gordon, C.~Schuhmann, L.~Schmidt, and J.~Jitsev, ``Reproducible scaling laws for contrastive language-image learning,'' in \emph{Proceedings of the 2023 IEEE/CVF Conference on Computer Vision and Pattern Recognition (CVPR)}, Jun. 2023, pp. 2818--2829.

\bibitem{Busbridge2025}
D.~Busbridge, A.~Shidani, F.~Weers, J.~Ramapuram, E.~Littwin, and R.~Webb, ``Distillation scaling laws,'' {arXiv:2502.08606 [cs.LG]}, 2025.

\bibitem{Krajewski2024scalinglawsMOE}
J.~Krajewski, J.~Ludziejewski, K.~Adamczewski, M.~Pióro, M.~Krutul, S.~Antoniak, K.~Ciebiera, K.~Król, T.~Odrzygóźdź, P.~Sankowski, M.~Cygan, and S.~Jaszczur, ``Scaling laws for fine-grained mixture of experts,'' {arXiv:2402.07871 [cs.LG]}, 2024.

\bibitem{Thompson1917}
D.~W. Thompson, \emph{On Growth and Form}.\hskip 1em plus 0.5em minus 0.4em\relax Cambridge University Press, 1917.

\bibitem{Haldane1926}
J.~B.~S. Haldane, ``{On being the right size},'' \emph{Harper’s Magazine}, vol. 152, pp. 424--427, 1926.

\bibitem{Manger2013}
P.~R. Manger, M.~A. Spocter, and N.~Patzke, ``{The evolutions of large brain size in mammals: the ‘over-700-gram club quartet'},'' \emph{Brain Behavior and Evolution}, vol.~82, no.~1, pp. 68--78, 2013.

\bibitem{Sardana2024}
N.~Sardana, J.~Portes, S.~Doubov, and J.~Frankle, ``Beyond {C}hinchilla-optimal: Accounting for inference in language model scaling laws,'' {arXiv:2401.00448 [cs.LG]}, 2024.

\bibitem{DeepSeekAI2024llmscaling}
DeepSeek-AI, X.~Bi, D.~Chen, G.~Chen, S.~Chen, D.~Dai, C.~Deng, H.~Ding, K.~Dong, Q.~Du, Z.~Fu, H.~Gao, K.~Gao, W.~Gao, R.~Ge, K.~Guan, D.~Guo, J.~Guo, G.~Hao, Z.~Hao, Y.~He, W.~Hu, P.~Huang, E.~Li, G.~Li, J.~Li, Y.~Li, Y.~K. Li, W.~Liang, F.~Lin, A.~X. Liu, B.~Liu, W.~Liu, X.~Liu, X.~Liu, Y.~Liu, H.~Lu, S.~Lu, F.~Luo, S.~Ma, X.~Nie, T.~Pei, Y.~Piao, J.~Qiu, H.~Qu, T.~Ren, Z.~Ren, C.~Ruan, Z.~Sha, Z.~Shao, J.~Song, X.~Su, J.~Sun, Y.~Sun, M.~Tang, B.~Wang, P.~Wang, S.~Wang, Y.~Wang, Y.~Wang, T.~Wu, Y.~Wu, X.~Xie, Z.~Xie, Z.~Xie, Y.~Xiong, H.~Xu, R.~X. Xu, Y.~Xu, D.~Yang, Y.~You, S.~Yu, X.~Yu, B.~Zhang, H.~Zhang, L.~Zhang, L.~Zhang, M.~Zhang, M.~Zhang, W.~Zhang, Y.~Zhang, C.~Zhao, Y.~Zhao, S.~Zhou, S.~Zhou, Q.~Zhu, and Y.~Zou, ``{DeepSeek LLM}: Scaling open-source language models with longtermism,'' {arXiv:2401.02954 [cs.CL]}, 2024.

\bibitem{Chen2025}
Y.~Chen, Y.~Wu, X.~Han, Z.~Liu, and M.~Sun, ``Cost-optimal grouped-query attention for long-context {LLMs},'' arXiv:2503.09579 [cs.CL], 2025.

\bibitem{Anil2023}
R.~Anil, A.~M. Dai, O.~Firat, M.~Johnson, D.~Lepikhin, A.~Passos, S.~Shakeri, E.~Taropa, P.~Bailey, Z.~Chen \emph{et~al.}, ``{PaLM} 2 technical report,'' arXiv:2305.10403 [cs.CL], 2023.

\bibitem{Besiroglu2024}
T.~Besiroglu, E.~Erdil, M.~Barnett, and J.~You, ``Chinchilla scaling: A replication attempt,'' {arXiv:2404.10102 [cs.AI]}, 2024.

\bibitem{Snell2024}
C.~Snell, J.~Lee, K.~Xu, and A.~Kumar, ``Scaling {LLM} test-time compute optimally can be more effective than scaling model parameters,'' {arXiv:2408.03314 [cs.LG]}, 2024.

\bibitem{Wu2025}
Y.~Wu, Z.~Sun, S.~Li, S.~Welleck, and Y.~Yang, ``Inference scaling laws: An empirical analysis of compute-optimal inference for problem-solving with language models,'' {arXiv:2408.00724 [cs.AI]}, 2025.

\bibitem{Devries2023}
\BIBentryALTinterwordspacing
H.~De~Vries, ``{Go smol or go home},'' 2023. [Online]. Available: \url{https://www.harmdevries.com/post/model-size-vs-compute-overhead/}
\BIBentrySTDinterwordspacing

\bibitem{OpenAI2024reasoning}
\BIBentryALTinterwordspacing
{OpenAI}, ``Learning to reason with {LLMs},'' Sep. 2024. [Online]. Available: \url{https://openai.com/index/learning-to-reason-with-llms/}
\BIBentrySTDinterwordspacing

\bibitem{EpochAI2023}
\BIBentryALTinterwordspacing
P.~Villalobos and D.~Atkinson, ``Trading off compute in training and inference,'' 2023. [Online]. Available: \url{https://epoch.ai/blog/trading-off-compute-in-training-and-inference}
\BIBentrySTDinterwordspacing

\bibitem{Brown2024repeatedsamplingsscaling}
B.~Brown, J.~Juravsky, R.~Ehrlich, R.~Clark, Q.~V. Le, C.~Ré, and A.~Mirhoseini, ``Large language monkeys: Scaling inference compute with repeated sampling,'' {arXiv:2407.21787 [cs.LG]}, 2024.

\bibitem{Polo2025slothscalinglawsllm}
F.~M. Polo, S.~Somerstep, L.~Choshen, Y.~Sun, and M.~Yurochkin, ``Sloth: scaling laws for {LLM} skills to predict multi-benchmark performance across families,'' {arXiv:2412.06540 [cs.LG]}, 2025.

\bibitem{Levi2024simple}
N.~Levi, ``A simple model of inference scaling laws,'' {arXiv:2410.16377 [stat.ML]}, 2024.

\bibitem{Zhou2024architectural}
D.~Zhou, S.~Patankar, D.~M. Lydon-Staley, P.~Zurn, M.~Gerlach, and D.~S. Bassett, ``{Architectural styles of curiosity in global Wikipedia mobile app readership},'' \emph{Science Advances}, vol.~10, no.~43, p. eadn3268, 2024.

\bibitem{Varshney2013creativity}
L.~R. Varshney, F.~Pinel, K.~R. Varshney, D.~Bhattacharjya, A.~Schoergendorfer, and Y.-M. Chee, ``A big data approach to computational creativity: The curious case of {C}hef {W}atson,'' \emph{IBM Journal of Research and Development}, vol.~63, no.~1, pp. 7:1--7:18, 2019.

\bibitem{Vempaty2017couponcollector}
\BIBentryALTinterwordspacing
A.~Vempaty, L.~R. Varshney, and P.~K. Varshney, ``A coupon-collector model of machine-aided discovery,'' in \emph{KDD Workshop on Data-Driven Discovery}, 2017. [Online]. Available: \url{https://arxiv.org/abs/1708.03833}
\BIBentrySTDinterwordspacing

\bibitem{Lee2025compresscot}
A.~Lee, E.~Che, and T.~Peng, ``How well do {LLMs} compress their own chain-of-thought? a token complexity approach,'' {arXiv:2503.01141 [cs.CL]}, 2025.

\bibitem{Kahan2017}
D.~M. Kahan, E.~Peters, E.~C. Dawson, and P.~Slovic, ``{Motivated numeracy and enlightened self-government},'' \emph{Behavioural Public Policy}, vol.~1, no.~1, pp. 54--86, 2017.

\bibitem{Blodgett2020}
S.~L. Blodgett, S.~Barocas, H.~Daum{\'e}~III, and H.~Wallach, ``Language (technology) is power: A critical survey of {\textquotedblleft}bias{\textquotedblright} in {NLP},'' in \emph{Proceedings of the 58th Annual Meeting of the Association for Computational Linguistics}, D.~Jurafsky, J.~Chai, N.~Schluter, and J.~Tetreault, Eds., Jul. 2020, pp. 5454--5476.

\bibitem{Bender2021}
E.~M. Bender, T.~Gebru, A.~McMillan-Major, and S.~Shmitchell, ``{On the dangers of stochastic parrots: Can language models be too big?}'' in \emph{{Proceedings of the 2021 ACM Conference on Fairness, Accountability, and Transparency}}, 2021, pp. 610--623.

\bibitem{Mehrabi2022}
N.~Mehrabi, F.~Morstatter, N.~Saxena, K.~Lerman, and A.~Galstyan, ``A survey on bias and fairness in machine learning,'' {arXiv:1908.09635 [cs.LG]}, 2022.

\bibitem{Zhou2024political}
D.~Zhou and Y.~Zhang, ``{Political biases and inconsistencies in bilingual GPT models—the cases of the US and China},'' \emph{Scientific Reports}, vol.~14, no.~1, p. 25048, 2024.

\bibitem{Bikhchandani1992}
S.~Bikhchandani, D.~Hirshleifer, and I.~Welch, ``{A theory of fads, fashion, custom, and cultural change as informational cascades},'' \emph{Journal of Political Economy}, vol. 100, no.~5, pp. 992--1026, 1992.

\bibitem{EpochAI2024Data}
\BIBentryALTinterwordspacing
{Epoch AI}, ``{Data on Notable AI Models},'' 6 2024, accessed: 2025-04-11. [Online]. Available: \url{https://epoch.ai/data/notable-ai-models}
\BIBentrySTDinterwordspacing

\bibitem{Hendrycks2021MATH}
D.~Hendrycks, C.~Burns, S.~Kadavath, A.~Arora, S.~Basart, E.~Tang, D.~Song, and J.~Steinhardt, ``Measuring mathematical problem solving with the {MATH} dataset,'' arXiv:2103.03874 [cs.LG], 2021.

\bibitem{Biderman2023pythia}
S.~Biderman, H.~Schoelkopf, Q.~G. Anthony, H.~Bradley, K.~O’Brien, E.~Hallahan, M.~A. Khan, S.~Purohit, U.~S. Prashanth, E.~Raff, A.~Skowron, L.~Sutawika, and O.~Van Der~Wal, ``{Pythia: A suite for analyzing large language models across training and scaling},'' in \emph{Proceedings of the International Conference on Machine Learning}, 2023, pp. 2397--2430.

\bibitem{Gemmateam2024gemma}
{Gemma Team}, T.~Mesnard, C.~Hardin, R.~Dadashi, S.~Bhupatiraju, S.~Pathak, L.~Sifre, M.~Rivière, M.~S. Kale, J.~Love, P.~Tafti, L.~Hussenot, P.~G. Sessa, A.~Chowdhery, A.~Roberts, A.~Barua, A.~Botev, A.~Castro-Ros, A.~Slone, A.~Héliou, A.~Tacchetti, A.~Bulanova, A.~Paterson, B.~Tsai, B.~Shahriari, C.~L. Lan, C.~A. Choquette-Choo, C.~Crepy, D.~Cer, D.~Ippolito, D.~Reid, E.~Buchatskaya, E.~Ni, E.~Noland, G.~Yan, G.~Tucker, G.-C. Muraru, G.~Rozhdestvenskiy, H.~Michalewski, I.~Tenney, I.~Grishchenko, J.~Austin, J.~Keeling, J.~Labanowski, J.-B. Lespiau, J.~Stanway, J.~Brennan, J.~Chen, J.~Ferret, J.~Chiu, J.~Mao-Jones, K.~Lee, K.~Yu, K.~Millican, L.~L. Sjoesund, L.~Lee, L.~Dixon, M.~Reid, M.~Mikuła, M.~Wirth, M.~Sharman, N.~Chinaev, N.~Thain, O.~Bachem, O.~Chang, O.~Wahltinez, P.~Bailey, P.~Michel, P.~Yotov, R.~Chaabouni, R.~Comanescu, R.~Jana, R.~Anil, R.~McIlroy, R.~Liu, R.~Mullins, S.~L. Smith, S.~Borgeaud, S.~Girgin, S.~Douglas, S.~Pandya, S.~Shakeri, S.~De, T.~Klimenko, T.~Hennigan, V.~Feinberg,
  W.~Stokowiec, Y.~hui Chen, Z.~Ahmed, Z.~Gong, T.~Warkentin, L.~Peran, M.~Giang, C.~Farabet, O.~Vinyals, J.~Dean, K.~Kavukcuoglu, D.~Hassabis, Z.~Ghahramani, D.~Eck, J.~Barral, F.~Pereira, E.~Collins, A.~Joulin, N.~Fiedel, E.~Senter, A.~Andreev, and K.~Kenealy, ``Gemma: Open models based on {G}emini research and technology,'' {arXiv:2403.08295 [cs.CL]}, 2024.

\bibitem{Borenstein2005}
S.~Borenstein, ``{The long-run efficiency of real-time electricity pricing},'' \emph{The Energy Journal}, vol.~26, no.~3, pp. 93--116, 2005.

\bibitem{EIAPricing}
\BIBentryALTinterwordspacing
{U.S. Energy Information Administration}, ``{Electricity Explained: Prices and Factors Affecting Prices},'' 2023, accessed: 2025-04-23. [Online]. Available: \url{https://www.eia.gov/energyexplained/electricity/prices-and-factors-affecting-prices.php}
\BIBentrySTDinterwordspacing

\bibitem{Handa2025}
\BIBentryALTinterwordspacing
K.~Handa, D.~Bent, A.~Tamkin, M.~McCain, E.~Durmus, M.~Stern, M.~Schiraldi, S.~Huang, S.~Ritchie, S.~Syverud, K.~Jagadish, M.~Vo, M.~Bell, and D.~Ganguli. (2025) Anthropic education report: How university students use {C}laude. [Online]. Available: \url{https://www.anthropic.com/news/anthropic-education-report-how-university-students-use-claude}
\BIBentrySTDinterwordspacing

\bibitem{Switzer2022InformationBatteries}
J.~Switzer and B.~Raghavan, ``Information batteries: storing opportunity power with speculative execution,'' \emph{SIGENERGY Energy Inform. Rev.}, vol.~1, no.~1, pp. 1--11, Nov. 2021.

\bibitem{Varshney2014RiskLimitedDispatchKnowledge}
S.~Agarwal, Y.-M. Chee, J.~Lee, R.~R. Sindhgatta, and L.~R. Varshney, ``{Risk-limited dispatch of knowledge work},'' Oct. 2014, {US} Patent App. 13/870,422.

\bibitem{Varaiya2011RiskLimitedDispatchEnergy}
P.~P. Varaiya, F.~F. Wu, and J.~W. Bialek, ``Smart operation of smart grid: Risk-limiting dispatch,'' \emph{Proceedings of the IEEE}, vol.~99, no.~1, pp. 40--57, 2011.

\bibitem{Zhou2023SustainableSupercomputing}
D.~Zhao, S.~Samsi, J.~McDonald, B.~Li, D.~Bestor, M.~Jones, D.~Tiwari, and V.~Gadepally, ``Sustainable supercomputing for {AI}: {GPU} power capping at {HPC} scale,'' in \emph{{Proceedings of the 2023 ACM Symposium on Cloud Computing}}, 2023, pp. 588--596.

\bibitem{Ong2025routellm}
I.~Ong, A.~Almahairi, V.~Wu, W.-L. Chiang, T.~Wu, J.~E. Gonzalez, M.~W. Kadous, and I.~Stoica, ``{RouteLLM}: Learning to route {LLMs} with preference data,'' {arXiv:2406.18665 [cs.LG]}, 2025.

\bibitem{Maddireddy2025locoml}
K.~Maddireddy, S.~K. Methukula, C.~Sridhar, and K.~Vaidhyanathan, ``{LoCoML}: A framework for real-world {ML} inference pipelines,'' {arXiv:2501.14165 [cs.SE]}, 2025.

\bibitem{OpenAI2024GPT4Pricing}
\BIBentryALTinterwordspacing
OpenAI, ``{How much does GPT-4 cost?}'' 2024, accessed: 2025-04-25. [Online]. Available: \url{https://help.openai.com/en/articles/7127956-how-much-does-gpt-4-cost}
\BIBentrySTDinterwordspacing

\bibitem{Anthropic2025UsageLimits}
\BIBentryALTinterwordspacing
Anthropic, ``{About Claude Pro usage},'' 2025, accessed: 2025-04-25. [Online]. Available: \url{https://support.anthropic.com/en/articles/8324991-about-claude-pro-usage}
\BIBentrySTDinterwordspacing

\bibitem{Chen2024typeflyflyingdroneslarge}
G.~Chen, X.~Yu, N.~Ling, and L.~Zhong, ``{TypeFly}: Flying drones with large language model,'' {arXiv:2312.14950 [cs.RO]}, 2024.

\bibitem{Rey2025DroneEdge}
L.~Rey, A.~M. Bernardos, A.~D. Dobrzycki, D.~Carramiñana, L.~Bergesio, J.~A. Besada, and J.~R. Casar, ``A performance analysis of you only look once models for deployment on constrained computational edge devices in drone applications,'' \emph{Electronics}, vol.~14, no.~3, p. 638, Feb. 2025.

\bibitem{Pope2023}
R.~Pope, S.~Douglas, A.~Chowdhery, J.~Devlin, J.~Bradbury, J.~Heek, K.~Xiao, S.~Agrawal, and J.~Dean, ``{Efficiently scaling transformer inference},'' \emph{Proceedings of Machine Learning and Systems}, vol.~5, pp. 606--624, 2023.

\bibitem{Korthikanti2023}
V.~A. Korthikanti, J.~Casper, S.~Lym, L.~McAfee, M.~Andersch, M.~Shoeybi, and B.~Catanzaro, ``{Reducing activation recomputation in large transformer models},'' \emph{Proceedings of Machine Learning and Systems}, vol.~5, pp. 341--353, 2023.

\bibitem{Frantar2023}
E.~Frantar, S.~Ashkboos, T.~Hoefler, and D.-A. Alistarh, ``{OPTQ: Accurate post-training quantization for generative pre-trained transformers},'' in \emph{{Proceedings of the International Conference on Learning Representations (ICLR)}}, 2023.

\bibitem{Gholami2024}
A.~Gholami, Z.~Yao, S.~Kim, C.~Hooper, M.~W. Mahoney, and K.~Keutzer, ``{AI} and memory wall,'' {arXiv:2403.14123 [cs.LG]}, 2024.

\bibitem{Shazeer2017moe}
N.~Shazeer, A.~Mirhoseini, K.~Maziarz, A.~Davis, Q.~Le, G.~Hinton, and J.~Dean, ``Outrageously large neural networks: The sparsely-gated mixture-of-experts layer,'' {arXiv:1701.06538 [cs.LG]}, 2017.

\bibitem{Fedus2022switch}
W.~Fedus, B.~Zoph, and N.~Shazeer, ``Switch transformers: Scaling to trillion parameter models with simple and efficient sparsity,'' {arXiv:2101.03961 [cs.LG]}, 2022.

\bibitem{DeepSeekAI2024}
DeepSeek-AI, A.~Liu, B.~Feng, B.~Wang, B.~Wang, B.~Liu, C.~Zhao, C.~Dengr, C.~Ruan, D.~Dai, D.~Guo, D.~Yang, D.~Chen, D.~Ji, E.~Li, F.~Lin, F.~Luo, G.~Hao, G.~Chen, G.~Li, H.~Zhang, H.~Xu, H.~Yang, H.~Zhang, H.~Ding, H.~Xin, H.~Gao, H.~Li, H.~Qu, J.~L. Cai, J.~Liang, J.~Guo, J.~Ni, J.~Li, J.~Chen, J.~Yuan, J.~Qiu, J.~Song, K.~Dong, K.~Gao, K.~Guan, L.~Wang, L.~Zhang, L.~Xu, L.~Xia, L.~Zhao, L.~Zhang, M.~Li, M.~Wang, M.~Zhang, M.~Zhang, M.~Tang, M.~Li, N.~Tian, P.~Huang, P.~Wang, P.~Zhang, Q.~Zhu, Q.~Chen, Q.~Du, R.~J. Chen, R.~L. Jin, R.~Ge, R.~Pan, R.~Xu, R.~Chen, S.~S. Li, S.~Lu, S.~Zhou, S.~Chen, S.~Wu, S.~Ye, S.~Ma, S.~Wang, S.~Zhou, S.~Yu, S.~Zhou, S.~Zheng, T.~Wang, T.~Pei, T.~Yuan, T.~Sun, W.~L. Xiao, W.~Zeng, W.~An, W.~Liu, W.~Liang, W.~Gao, W.~Zhang, X.~Q. Li, X.~Jin, X.~Wang, X.~Bi, X.~Liu, X.~Wang, X.~Shen, X.~Chen, X.~Chen, X.~Nie, X.~Sun, X.~Wang, X.~Liu, X.~Xie, X.~Yu, X.~Song, X.~Zhou, X.~Yang, X.~Lu, X.~Su, Y.~Wu, Y.~K. Li, Y.~X. Wei, Y.~X. Zhu, Y.~Xu, Y.~Huang, Y.~Li, Y.~Zhao, Y.~Sun, Y.~Li,
  Y.~Wang, Y.~Zheng, Y.~Zhang, Y.~Xiong, Y.~Zhao, Y.~He, Y.~Tang, Y.~Piao, Y.~Dong, Y.~Tan, Y.~Liu, Y.~Wang, Y.~Guo, Y.~Zhu, Y.~Wang, Y.~Zou, Y.~Zha, Y.~Ma, Y.~Yan, Y.~You, Y.~Liu, Z.~Z. Ren, Z.~Ren, Z.~Sha, Z.~Fu, Z.~Huang, Z.~Zhang, Z.~Xie, Z.~Hao, Z.~Shao, Z.~Wen, Z.~Xu, Z.~Zhang, Z.~Li, Z.~Wang, Z.~Gu, Z.~Li, and Z.~Xie, ``{DeepSeek-V2}: A strong, economical, and efficient mixture-of-experts language model,'' {arXiv:2405.04434 [cs.CL]}, 2024.

\bibitem{Jelassi2025mixture}
S.~Jelassi, C.~Mohri, D.~Brandfonbrener, A.~Gu, N.~Vyas, N.~Anand, D.~Alvarez-Melis, Y.~Li, S.~M. Kakade, and E.~Malach, ``Mixture of parrots: Experts improve memorization more than reasoning,'' in \emph{Proceedings of the International Conference on Learning Representations (ICLR)}, 2025.

\bibitem{Wang2025two}
M.~Wang, X.~Chen, Y.~Wang, Z.~He, J.~Xu, T.~Liang, Q.~Liu, Y.~Yao, W.~Wang, R.~Ma, H.~Mi, N.~Zhang, Z.~Tu, X.~Li, and D.~Yu, ``Two experts are all you need for steering thinking: Reinforcing cognitive effort in moe reasoning models without additional training,'' arXiv:2505.14681 [cs.AI], 2025.

\bibitem{Shazeer2019}
N.~Shazeer, ``Fast transformer decoding: One write-head is all you need,'' arXiv:1911.02150 [cs.NE], 2019.

\bibitem{He2025}
C.~He, Y.~Huang, P.~Mu, Z.~Miao, J.~Xue, L.~Ma, F.~Yang, and L.~Mai, ``{WaferLLM}: A wafer-scale {LLM} inference system,'' {arXiv:2502.04563 [cs.LG]}, 2025.

\bibitem{Somala2025}
V.~Somala, ``{Ban the H20: Competing in the Inference Age},'' \url{https://www.chinatalk.media/p/ban-the-h20-competing-in-the-inference}, 2025, accessed: 2025-04-26.

\bibitem{Xing2025SpikeLLM}
X.~Xing, B.~Gao, Z.~Zhang, D.~A. Clifton, S.~Xiao, L.~Du, G.~Li, and J.~Zhang, ``{SpikeLLM}: Scaling up spiking neural network to large language models via saliency-based spiking,'' {arXiv:2407.04752 [cs.LG]}, 2025.

\bibitem{Merolla2014TrueNorth}
P.~A. Merolla, J.~V. Arthur, R.~Alvarez-Icaza, A.~S. Cassidy, J.~Sawada, F.~Akopyan, B.~L. Jackson, N.~Imam, C.~Guo, Y.~Nakamura, B.~Brezzo, I.~Vo, S.~K. Esser, R.~Appuswamy, B.~Taba, A.~Amir, M.~D. Flickner, W.~P. Risk, R.~Manohar, and D.~S. Modha, ``{A million spiking-neuron integrated circuit with a scalable communication network and interface},'' \emph{Science}, vol. 345, no. 6197, pp. 668--673, 2014.

\bibitem{Davies2018Loihi}
M.~Davies, N.~Srinivasa, T.-H. Lin, G.~Chinya, Y.~Cao, S.~H. Choday, G.~Dimou, P.~Joshi, N.~Imam, S.~Jain, Y.~Liao, C.-K. Lin, A.~Lines, R.~Liu, D.~Mathaikutty, S.~McCoy, A.~Paul, J.~Tse, G.~Venkataramanan, Y.-H. Weng, A.~Wild, Y.~Yang, and H.~Wang, ``{Loihi: A neuromorphic manycore processor with on-chip learning},'' \emph{IEEE Micro}, vol.~38, no.~1, pp. 82--99, 2018.

\bibitem{Habibollahi2022}
F.~Habibollahi, M.~Khajehnejad, A.~Gaurav, and B.~J. Kagan, ``Biological neurons vs deep reinforcement learning: Sample efficiency in a simulated game-world,'' in \emph{{Deep Reinforcement Learning Workshop NeurIPS 2022}}, 2022.

\bibitem{Cai2023}
H.~Cai, Z.~Ao, C.~Tian, Z.~Wu, H.~Liu, J.~Tchieu, M.~Gu, K.~Mackie, and F.~Guo, ``{Brain organoid reservoir computing for artificial intelligence},'' \emph{Nature Electronics}, vol.~6, no.~12, pp. 1032--1039, 2023.

\bibitem{EllisMohr2024Directed}
A.~R. Ellis-Mohr and L.~R. Varshney, ``Directed information flow in computing systems with living neurons,'' in \emph{{Proceedings of the 2024 IEEE International Symposium on Information Theory Workshops (ISIT-W)}}, 2024.

\bibitem{Wei2023CoT}
J.~Wei, X.~Wang, D.~Schuurmans, M.~Bosma, B.~Ichter, F.~Xia, E.~Chi, Q.~Le, and D.~Zhou, ``Chain-of-thought prompting elicits reasoning in large language models,'' {arXiv:2201.11903 [cs.CL]}, 2023.

\bibitem{Cheng2024compressedcot}
J.~Cheng and B.~Van~Durme, ``Compressed chain of thought: Efficient reasoning through dense representations,'' {arXiv:2412.13171 [cs.CL]}, 2024.

\bibitem{Shen2025efficientreasoninghiddenthinking}
X.~Shen, Y.~Wang, X.~Shi, Y.~Wang, P.~Zhao, and J.~Gu, ``Efficient reasoning with hidden thinking,'' {arXiv:2501.19201 [cs.CL]}, 2025.

\bibitem{Xia2025cotcompression}
H.~Xia, Y.~Li, C.~T. Leong, W.~Wang, and W.~Li, ``{TokenSkip}: Controllable chain-of-thought compression in {LLMs},'' {arXiv:2502.12067 [cs.CL]}, 2025.

\bibitem{Kang2025c3ot}
Y.~Kang, X.~Sun, L.~Chen, and W.~Zou, ``{C3oT}: Generating shorter chain-of-thought without compromising effectiveness,'' in \emph{{Proceedings of the AAAI Conference on Artificial Intelligence}}, vol.~39, no.~23, 2025, pp. 24\,312--24\,320.

\bibitem{Yao2023ToT}
S.~Yao, D.~Yu, J.~Zhao, I.~Shafran, T.~Griffiths, Y.~Cao, and K.~Narasimhan, ``Tree of thoughts: Deliberate problem solving with large language models,'' in \emph{Advances in Neural Information Processing Systems}, 2023, vol.~36, pp. 11\,809--11\,822.

\bibitem{Suzgun2025}
M.~Suzgun, M.~Yuksekgonul, F.~Bianchi, D.~Jurafsky, and J.~Zou, ``Dynamic cheatsheet: Test-time learning with adaptive memory,'' {arXiv:2504.07952 [cs.LG]}, 2025.

\bibitem{Yue2024inference}
Z.~Yue, H.~Zhuang, A.~Bai, K.~Hui, R.~Jagerman, H.~Zeng, Z.~Qin, D.~Wang, X.~Wang, and M.~Bendersky, ``Inference scaling for long-context retrieval augmented generation,'' arXiv:2410.04343 [cs.CL], 2024.

\bibitem{DeepMind2024}
{Gemini Team, Google}, ``{Gemini 1.5: Unlocking Multimodal Understanding Across Millions of Tokens of Context},'' Google DeepMind, Tech. Rep., Feb. 2024.

\bibitem{Meta2025Llama4}
\BIBentryALTinterwordspacing
{Meta AI}, ``{The Llama 4 Herd: The Beginning of a New Era of Natively Multimodal AI Innovation},'' 2025, accessed: 2025-04-25. [Online]. Available: \url{https://ai.meta.com/blog/llama-4-multimodal-intelligence/}
\BIBentrySTDinterwordspacing

\bibitem{Bian2025scaling}
S.~Bian, M.~Yan, and S.~Venkataraman, ``Scaling inference-efficient language models,'' arXiv:2501.18107 [cs.LG], 2025.

\bibitem{Fedorenko2024}
E.~Fedorenko, S.~T. Piantadosi, and E.~A.~F. Gibson, ``{Language is primarily a tool for communication rather than thought},'' \emph{Nature}, vol. 630, pp. 575--586, June 2024.

\bibitem{Monti2007}
M.~M. Monti, D.~N. Osherson, M.~J. Martinez, and L.~M. Parsons, ``{Functional neuroanatomy of deductive inference: a language-independent distributed network},'' \emph{NeuroImage}, vol.~37, no.~3, pp. 1005--1016, 2007.

\bibitem{Monti2009}
M.~M. Monti, L.~M. Parsons, and D.~N. Osherson, ``{The boundaries of language and thought in deductive inference},'' \emph{Proceedings of the National Academy of Sciences}, vol. 106, no.~30, pp. 12\,554--12\,559, 2009.

\bibitem{Fedorenko2011}
E.~Fedorenko, M.~K. Behr, and N.~Kanwisher, ``{Functional specificity for high-level linguistic processing in the human brain},'' \emph{Proceedings of the National Academy of Sciences}, vol. 108, no.~39, pp. 16\,428--16\,433, 2011.

\bibitem{Monti2012}
M.~M. Monti, L.~M. Parsons, and D.~N. Osherson, ``{Thought beyond language: Neural dissociation of algebra and natural language},'' \emph{Psychological Science}, vol.~23, no.~8, pp. 914--922, 2012.

\bibitem{Amalric2019}
M.~Amalric and S.~Dehaene, ``{A distinct cortical network for mathematical knowledge in the human brain},'' \emph{NeuroImage}, vol. 189, pp. 19--31, April 2019.

\bibitem{Hao2024}
S.~Hao, S.~Sukhbaatar, D.~Su, X.~Li, Z.~Hu, J.~Weston, and Y.~Tian, ``Training large language models to reason in a continuous latent space,'' {arXiv:2412.06769 [cs.CL]}, 2024.

\bibitem{Geiping2025}
J.~Geiping, S.~McLeish, N.~Jain, J.~Kirchenbauer, S.~Singh, B.~R. Bartoldson, B.~Kailkhura, A.~Bhatele, and T.~Goldstein, ``Scaling up test-time compute with latent reasoning: A recurrent depth approach,'' {arXiv:2502.05171 [cs.LG]}, 2025.

\bibitem{Li2025}
C.~Li, W.~Wu, H.~Zhang, Y.~Xia, S.~Mao, L.~Dong, I.~Vulić, and F.~Wei, ``Imagine while reasoning in space: Multimodal visualization-of-thought,'' {arXiv:2501.07542 [cs.CL]}, 2025.

\bibitem{OpenAI2025ImageThinking}
\BIBentryALTinterwordspacing
{OpenAI}, ``{Thinking with Images},'' 2025, accessed: 2025-05-05. [Online]. Available: \url{https://openai.com/index/thinking-with-images/}
\BIBentrySTDinterwordspacing

\bibitem{Kim2022}
G.~Kim, T.~Kwon, and J.~C. Ye, ``{DiffusionCLIP}: Text-guided diffusion models for robust image manipulation,'' {arXiv:2110.02711 [cs.CV]}, 2022.

\bibitem{Ma2025imageinferencescaling}
N.~Ma, S.~Tong, H.~Jia, H.~Hu, Y.-C. Su, M.~Zhang, X.~Yang, Y.~Li, T.~Jaakkola, X.~Jia, and S.~Xie, ``Inference-time scaling for diffusion models beyond scaling denoising steps,'' arXiv:2501.09732 [cs.CV], 2025.

\bibitem{Li2022diffusionlm}
X.~L. Li, J.~Thickstun, I.~Gulrajani, P.~Liang, and T.~B. Hashimoto, ``Diffusion-{LM} improves controllable text generation,'' {arXiv:2205.14217 [cs.CL]}, 2022.

\bibitem{Han2023dlm}
X.~Han, S.~Kumar, and Y.~Tsvetkov, ``{SSD-LM}: Semi-autoregressive simplex-based diffusion language model for text generation and modular control,'' {arXiv:2210.17432 [cs.CL]}, 2023.

\bibitem{Lou2024discretedlm}
A.~Lou, C.~Meng, and S.~Ermon, ``Discrete diffusion modeling by estimating the ratios of the data distribution,'' {arXiv:2310.16834 [stat.ML]}, 2024.

\bibitem{Gong2024scalingdlm}
S.~Gong, S.~Agarwal, Y.~Zhang, J.~Ye, L.~Zheng, M.~Li, C.~An, P.~Zhao, W.~Bi, J.~Han, H.~Peng, and L.~Kong, ``Scaling diffusion language models via adaptation from autoregressive models,'' arXiv:2410.17891 [cs.CL], 2024.

\bibitem{Nie2025dLLMs}
S.~Nie, F.~Zhu, Z.~You, X.~Zhang, J.~Ou, J.~Hu, J.~Zhou, Y.~Lin, J.-R. Wen, and C.~Li, ``Large language diffusion models,'' arXiv:2502.09992 [cs.CL], 2025.

\bibitem{Leviathan2023}
Y.~Leviathan, M.~Kalman, and Y.~Matias, ``{Fast inference from transformers via speculative decoding},'' in \emph{{Proceedings of the International Conference on Machine Learning}}, 2023, pp. 19\,274--19\,286.

\bibitem{Ye2024MultiAgentInfScaling}
H.~Ye, M.~Lin, H.~T. Ng, and S.~Yan, ``Multi-agent sampling: Scaling inference compute for data synthesis with tree search-based agentic collaboration,'' {arXiv:2412.17061 [cs.CL]}, 2024.

\bibitem{Jin2025multiagenttesttime}
C.~Jin, H.~Peng, Q.~Zhang, Y.~Tang, D.~N. Metaxas, and T.~Che, ``Two heads are better than one: Test-time scaling of multi-agent collaborative reasoning,'' {arXiv:2504.09772 [cs.AI]}, 2025.

\bibitem{Qian2025scalingmultiagent}
C.~Qian, Z.~Xie, Y.~Wang, W.~Liu, K.~Zhu, H.~Xia, Y.~Dang, Z.~Du, W.~Chen, C.~Yang, Z.~Liu, and M.~Sun, ``Scaling large language model-based multi-agent collaboration,'' {arXiv:2406.07155 [cs.AI]}, 2025.

\bibitem{Sastry2024}
G.~Sastry, L.~Heim, H.~Belfield, M.~Anderljung, M.~Brundage, J.~Hazell, C.~O'Keefe, G.~K. Hadfield, R.~Ngo, K.~Pilz, G.~Gor, E.~Bluemke, S.~Shoker, J.~Egan, R.~F. Trager, S.~Avin, A.~Weller, Y.~Bengio, and D.~Coyle, ``Computing power and the governance of artificial intelligence,'' {arXiv:2402.08797 [cs.CY]}, 2024.

\bibitem{Shavit2023}
Y.~Shavit, ``What does it take to catch a {C}hinchilla? verifying rules on large-scale neural network training via compute monitoring,'' {arXiv:2303.11341 [cs.LG]}, 2023.

\bibitem{Kulp2024}
G.~Kulp, D.~Gonzales, E.~Smith, L.~Heim, P.~Puri, M.~J.~D. Vermeer, and Z.~Winkelman, ``Hardware-enabled governance mechanisms: Developing technical solutions to exempt items otherwise classified under export control classification numbers {3A090} and {4A090},'' RAND Corporation, Santa Monica, CA, Tech. Rep. WR-A3056-1, Jan. 2024.

\bibitem{Reuel2024}
A.~Reuel, B.~Bucknall, S.~Casper, T.~Fist, L.~Soder, O.~Aarne, L.~Hammond, L.~Ibrahim, A.~Chan, P.~Wills, M.~Anderljung, B.~Garfinkel, L.~Heim, A.~Trask, G.~Mukobi, R.~Schaeffer, M.~Baker, S.~Hooker, I.~Solaiman, A.~S. Luccioni, N.~Rajkumar, N.~Moës, J.~Ladish, N.~Guha, J.~Newman, Y.~Bengio, T.~South, A.~Pentland, S.~Koyejo, M.~J. Kochenderfer, and R.~Trager, ``Open problems in technical {AI} governance,'' {arXiv:2407.14981 [cs.CY]}, 2024.

\bibitem{PrimeIntellect2025decentralizedreasoning}
{Prime Intellect Team}, S.~Jaghouar, J.~Mattern, J.~M. Ong, J.~Straube, M.~Basra, A.~Pazdera, K.~Thaman, M.~D. Ferrante, F.~Gabriel, F.~Obeid, K.~Erdem, M.~Keiblinger, and J.~Hagemann, ``{INTELLECT-2}: A reasoning model trained through globally decentralized reinforcement learning,'' {arXiv:2505.07291 [cs.LG]}, 2025.

\bibitem{VarshneyKS219}
L.~R. Varshney, N.~S. Keskar, and R.~Socher, ``Pretrained {AI} models: Performativity, mobility, and change,'' {arXiv:1909.03290 [cs.CY]}, 2019.

\bibitem{Zhang2024benchmarkissues}
H.~Zhang, J.~Da, D.~Lee, V.~Robinson, C.~Wu, W.~Song, T.~Zhao, P.~Raja, C.~Zhuang, D.~Slack \emph{et~al.}, ``A careful examination of large language model performance on grade school arithmetic,'' in \emph{Advances in Neural Information Processing Systems}, 2024, vol.~37, pp. 46\,819--46\,836.

\bibitem{Cohen2025benchmarkissues}
N.~Cohen-Inger, Y.~Elisha, B.~Shapira, L.~Rokach, and S.~Cohen, ``Forget what you know about {LLMs} evaluations -- {LLMs} are like a chameleon,'' arXiv:2502.07445 [cs.CL], 2025.

\bibitem{Kazemi2025bigbenchextrahard}
M.~Kazemi, B.~Fatemi, H.~Bansal, J.~Palowitch, C.~Anastasiou, S.~V. Mehta, L.~K. Jain, V.~Aglietti, D.~Jindal, P.~Chen, N.~Dikkala, G.~Tyen, X.~Liu, U.~Shalit, S.~Chiappa, K.~Olszewska, Y.~Tay, V.~Q. Tran, Q.~V. Le, and O.~Firat, ``{BIG-Bench} extra hard,'' {arXiv:2502.19187 [cs.CL]}, 2025.

\bibitem{FederalRegister2024Investments}
\BIBentryALTinterwordspacing
{U.S. Department of the Treasury}, ``{Provisions Pertaining to U.S. Investments in Certain National Security Technologies and Products in Countries of Concern},'' Nov. 2024. [Online]. Available: \url{https://www.federalregister.gov/documents/2024/11/15/2024-25422/provisions-pertaining-to-us-investments-in-certain-national-security-technologies-and-products-in}
\BIBentrySTDinterwordspacing

\bibitem{EUAI2024Article51}
\BIBentryALTinterwordspacing
{European Union}, ``{Article 51: Classification of General-Purpose AI Models as General-Purpose AI Models with Systemic Risk},'' Jul. 2024. [Online]. Available: \url{https://artificialintelligenceact.eu/article/51/}
\BIBentrySTDinterwordspacing

\bibitem{Heim2024training}
L.~Heim and L.~Koessler, ``Training compute thresholds: Features and functions in {AI} regulation,'' arXiv:2405.10799 [cs.CY], 2024.

\bibitem{Openai2025o3o4}
\BIBentryALTinterwordspacing
OpenAI, ``o3 and o4-mini system card,'' 2025. [Online]. Available: \url{https://cdn.openai.com/pdf/2221c875-02dc-4789-800b-e7758f3722c1/o3-and-o4-mini-system-card.pdf}
\BIBentrySTDinterwordspacing

\bibitem{Anthropic2024claude4}
\BIBentryALTinterwordspacing
Anthropic, ``{Claude 4 System Card},'' 2024. [Online]. Available: \url{https://www-cdn.anthropic.com/6be99a52cb68eb70eb9572b4cafad13df32ed995.pdf}
\BIBentrySTDinterwordspacing

\bibitem{He2021automl}
X.~He, K.~Zhao, and X.~Chu, ``{AutoML: A survey of the state-of-the-art},'' \emph{Knowledge-based Systems}, vol. 212, p. 106622, 2021.

\bibitem{Zhang2025selfimprovementllm}
G.~Zhang, W.~Liang, O.~Hsu, and K.~Olukotun, ``Adaptive self-improvement {LLM} agentic system for {ML} library development,'' {arXiv:2502.02534 [cs.CL]}, 2025.

\bibitem{DeepMind2025alphaevolve}
\BIBentryALTinterwordspacing
A.~Novikov, N.~Vu, M.~Eisenberger, E.~Dupont, P.-S. Huang, A.~Z. Wagner, S.~Shirobokov, B.~Kozlovskii, F.~J.~R. Ruiz, A.~Mehrabian, M.~P. Kumar, A.~See, S.~Chaudhuri, G.~Holland, A.~Davies, S.~Nowozin, P.~Kohli, and M.~Balog, ``{AlphaEvolve: A coding agent for scientific and algorithmic discovery},'' Google DeepMind, White paper, May 2025. [Online]. Available: \url{https://storage.googleapis.com/deepmind-media/DeepMind.com/Blog/alphaevolve-a-gemini-powered-coding-agent-for-designing-advanced-algorithms/AlphaEvolve.pdf}
\BIBentrySTDinterwordspacing

\bibitem{Peng2024}
X.~Peng, C.~Xia, X.~Yang, C.~Xiong, C.-S. Wu, and C.~Xing, ``{ReGenesis}: {LLMs} can grow into reasoning generalists via self-improvement,'' {arXiv:2410.02108 [cs.CL]}, 2024.

\bibitem{Zhou2022humanlevelprompt}
Y.~Zhou, A.~I. Muresanu, Z.~Han, K.~Paster, S.~Pitis, H.~Chan, and J.~Ba, ``{Large language models are human-level prompt engineers},'' in \emph{{Proceedings of the International Conference on Learning Representations (ICLR)}}, 2022.

\bibitem{Meinke2025}
A.~Meinke, B.~Schoen, J.~Scheurer, M.~Balesni, R.~Shah, and M.~Hobbhahn, ``{Frontier Models are Capable of In-context Scheming},'' Apollo Research, Tech. Rep., January 2025.

\end{thebibliography}

\begin{appendices}
\section{Additional Analysis for CoT vs. ToT(1) Inference Policies}
\label{ap:cotvstot}

This appendix provides a supplemental derivation referenced in Section~\ref{subsec:cotvstot} and used in Proof~\ref{proof:branchingtaskcapability} in Theorem~\ref{thm:task-capability}. It also provides the setup for the numerical analysis in Fig.~\ref{fig:totcotdiff} and shows the individual probabilities---see Fig.~\ref{fig:totcotsuccess}.

\subsection{Asymptotic Derivation for Low-Capability Regime (Prop.~\ref{proof:branchingtaskcapability})}
\label{ap:proof-low}
We begin from the definition of the regularized incomplete beta:
\[
I_x(m,n) = \frac{1}{B(m,n)} \int_0^x u^{m-1}(1-u)^{n-1} du.
\]
In the limit $x \to 0^+$ and $n \gg m$, we expand the integrand:
\[
\int_0^x u^{m-1}(1-u)^{n-1} du \approx \int_0^x u^{m-1}(1 + (n-1)u) du \approx \frac{x^m}{m}.
\]

Next, we use the asymptotic form of the beta function:
\[
B(m,n) = \frac{(m-1)!(n-1)!}{(m+n-1)!}=\frac{(m-1)!}{n\cdot(n+1)...(n+m-1)} \approx (m-1)! \cdot n^{-m}.
\]
Substituting $x = 1 - (1 - \varepsilon)^b \approx b\varepsilon$ and $n = \Omega / (\sqrt{b}\,\omega) - m + 1 \approx \Omega / (\sqrt{b}\,\omega)$, we obtain:
\begin{align*}
\psi^{(b)} 
&\approx \frac{(b\varepsilon)^m}{m!} \left(\frac{\Omega}{\sqrt{b}\,\omega}\right)^m \\
&= \frac{(\varepsilon \,\Omega / \omega)^m}{m!} \cdot b^{m/2}.
\end{align*}

Similarly, for CoT ($b=1$), we have:
\[
\psi^{(1)} \approx \frac{(\varepsilon \,\Omega / \omega)^m}{m!}.
\]

\subsection{Numerical Analysis} \label{subsec:cottotnumerical}
For CoT, the total cost is dominated by attention and is computed using \eqref{eq:CoTTotalCost}. For ToT(1), we equate that with the corresponding cost in \eqref{eq:ToTTotalCost}, plus additional terms reflecting logit accumulation and branching decisions. Specifically, we add a logit term \(\Omega_\text{ToT(1)} + \Omega_\text{ToT(1)} \cdot (b - 1) / (b \cdot \omega)\) to account for the number of tokens contributing to the branch pool, and a comparison term \(\Omega_T / b \cdot \omega / (b - 1)\) to represent the decision cost per branch. The verification cost is assumed negligible, and the verifier is treated as an oracle. We assume a prompt length of $\Omega_{\text{prompt}}=200$ tokens, a model size of 10 billion non-embedding parameters $N$, 96 transformer layers $n_{\text{layer}}$, and an attention dimension of 12,288 $d_{\text{attn}}$. Each branch is allocated $\omega=25$ tokens before re-branching.

\begin{figure}
    \centering
    \includegraphics[width=1\textwidth]{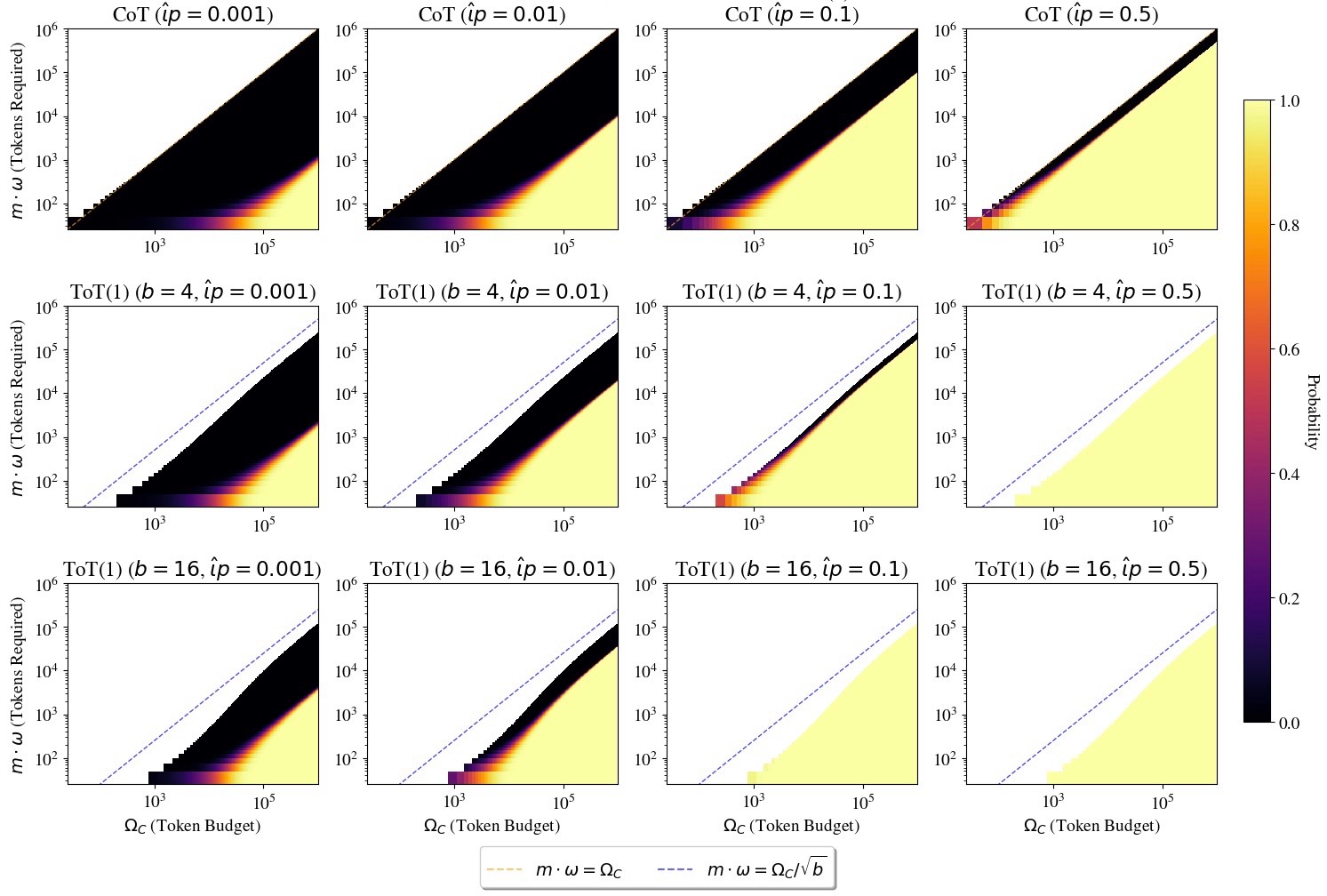}\caption{
    Success probabilities for CoT and ToT(1) across varying token budgets and numbers of required skills. The top row shows CoT; the middle and bottom rows show ToT(1) under increasing branching. On a logarithmic scale, the transition region between near-zero and near-one success for each strategy is narrow yet the differences between strategies are significant.}
    \label{fig:totcotsuccess}
\end{figure}

Fig.~\ref{fig:totcotsuccess} shows the full success probability heatmaps corresponding to the success difference plots in Fig.~\ref{fig:totcotdiff}. The top row displays CoT success rates across varying per-step outcome probabilities \(\hat{\iota}p\), while the middle and bottom rows show the corresponding ToT(1) success rates for branching factors 4 and 16, respectively, under matched total token budgets. These heatmaps provide a direct visual grounding for the advantage regions discussed in the main text. On a logarithmic token scale, the transition region between near-zero and near-one success remains narrow, making comparative gains highly sensitive to both task difficulty and branching strategy.

\section{Skill-Text Framework and Empirical Fit Details}
\begin{figure}[ht]
    \centering    \includegraphics[width=0.5\textwidth]{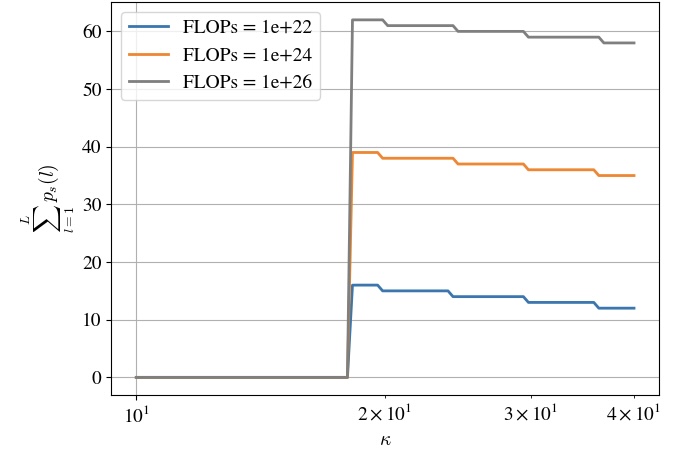}
    \caption{Trained skill competence as a function of the token-to-parameter ratio \(\kappa\), shown across compute budgets. Competence is measured as the sum over all layers of continuous skill learning scores \(p_s^{(\ell)}\), which represent the degree to which each skill is acquired. A sharp transition appears near the Chinchilla-optimal \(\kappa\), beyond which skill competence degrades monotonically due to overtraining. The graph-based framework does not apply in the undertrained regime, which is out of scope for this work.}    \label{fig:sum_psl_vs_kappa}
\end{figure}
The previous skill-text framework \cite{Nayak2025} is adjusted to account for empirically observed overtraining decay of capability beyond the Chinchilla point as demonstrated in Fig.~\ref{fig:sum_psl_vs_kappa}. Undertraining is beyond the scope of this study and mathematical framework.

To generate the total-compute-optimal scaling curves and task-level accuracy predictions shown in Fig.~\ref{fig:AIME}, we adopt a simplified model calibrated to match qualitative features of large-scale pretrained language models, with a focus on alignment with publicly discussed benchmarks for Claude 3.7 Sonnet and OpenAI's o1. While these models may differ internally---e.g., through mixture-of-experts or architectural optimizations---we assume Chinchilla-style dense pretraining and apply rounded parameters to maintain interpretability and analytical tractability.

\subsection*{Claude 3.7 Sonnet Fit} \label{ap:claudefit}
For the Claude 3.7 model (Fig.~\ref{fig:AIME}.2a), we assume a total training FLOP estimate of approximately $3.35 \times 10^{25}$, consistent with EpochAI projections~\cite{EpochAI2024Data}. Task performance is generated by drawing task difficulty levels from the range $l \in [60, 70]$, with a total skill pool capped at $L = 100$. Tasks require between $m = 2$ and $m = 15$ skills, selected from a pool of $S_l = 10^4$ skills per level. Training compute is distributed across $S_l \cdot L$ skills, each consisting of $\tau = 10^4$ tokens. Text pieces are connected to concepts via a bipartite graph with average degree $d_t = 6$. Each concept is associated with $\zeta = 2.5 \times 10^3$ parameters. The resulting parameter count is computed as $N = \zeta \cdot R$, where $R$ is the number of learnable concepts. Each concept is connected to a skill in level $l$ with probability $\exp(-\mathfrak{b}l/L)$, with $\mathfrak{b} = 10$. We adopt a Chinchilla-style training allocation with data-to-parameter ratio $\kappa = 20$. Inference per skill consumes $\omega = 25$ tokens. To reflect real-world task ambiguity and subtask relevance, we assume that only a subset of the skills required are relevant, scaling the required set size by a factor $\beta = 5$, with relevant skill selection directionality coefficient $\delta = 0.5$ such that $\hat{\iota}=\frac{1}{2}+\frac{1}{2}/\lceil m_{l'}(1+\beta) \rceil$ where $m_{l'}$ is given from the interaction of task-dependent probabilities to maximize $I_{\hat{\iota}p_{l'}}(m,T_{\text{max}}-m+1)$.

\subsection*{Solution Strategy Layer Selection}
The preceding analysis assumes reasoning occurs within a single skill layer, but more general strategies---such as mixing layers or adaptive inference---are possible. Practical objectives include maximizing expected accuracy for a fixed step budget, minimizing expected inference cost $\mathbb{E}[C_{\text{inf}}]$, balancing risk-reward across cost and performance, or adaptively searching via least-cost paths. For tractability, we select the layer with highest expected accuracy under a maximum step budget $T_{\max}$. This approach has two main limitations: (1) models may favor minimizing expected steps over maximizing accuracy, as observed in reinforcement learning; and (2) while some tasks are unsolvable in lower layers due to high compositional demand, partial progress in those layers may still be beneficial. Nevertheless, this strategy aligns with inference-constrained deployment scenarios and offers a concise basis for analyzing capabilities. Alternate strategies remain promising directions for future work.

\subsection*{OpenAI o1 Fit}
For the OpenAI o1 model, we use a reported training compute estimate of $3.8 \times 10^{25}$ FLOPs and perform alignment using the AIME dataset benchmarks shown in Fig.~\ref{fig:AIME}.3a–3b. To match observed token-accuracy relationships, we apply the same structural assumptions as above, adjusting only for differences in training and inference token budgets. The DS3 framework applies a per-task inference budget of $\Omega = 64,000$ tokens in Fig.~\ref{fig:AIME}.3b, with model size fixed and test-time accuracy scaling plotted on the right. These align with public reports and highlight similar log-linear patterns to those observed in Claude 3.7. While these empirical fits necessarily abstract from architectural specifics, they allow us to illustrate the generality of the scaling principles under investigation. Importantly, these results are not sensitive to precise values but rather to the monotonicity and structure of the compute allocation space, consistent with the conclusions drawn from our DS3-based modeling framework.

\section{Two-Step Forecasting Methods: Additional Details}\label{ap:chinchilla_params}
For this analysis, we adopt the loss model and fit introduced by Hoffman et al.\cite{Hoffmann2022}, $\mathcal{L}(N,D) = 1.69 + 406.4\cdot  N^{-0.34}+410.7\cdot D^{-0.28}$, and fix the number of downstream tasks to $I=10^9$, optimizing under a total-compute without attention constraint: $C_{\text{total}}=6ND+2NI\Omega$. Skill levels $\mathcal{L}_0$ are defined over the range $[0.1, 4.0]$ with 8 equally sized steps, and each task draws uniformly over $(l, m)$ with $m \in [1, 50]$ skills. We use sigmoid parameters $a = 1$, $b = 5$, and $d = 0$, with the outcome probability per skill depending on the difference between the model-implied loss and task-specific $\mathcal{L}_0$. The tokens per skill step is set to $\omega = 25$.

Chinchilla optimality considers pretraining loss yields approximately equal scaling with compute $D_{\text{tr}}\approx20N$. We maximize accuracy as in \eqref{eq:max_accuracy} on the resulting 2D surface constraint. For stability and analytical tractability, we use a nearest neighbor approach after populating the 3D grid and smooth the log-space curve while preserving the iso-compute constraint through a projection onto the iso-compute surface.

Fig.~\ref{fig:OptimalvsChinchilla} includes: 1. Accuracy surface in logarithmic space. Each point in the $(\log_{10} D,\log_{10}N,\log_{10}\Omega)$ cube is colored by the measured accuracy of the nearest model in the dataset, establishing the response surface on which the subsequent optimization is performed. 2. Optimal resource-allocation trajectories under a fixed total FLOP budget without attention costs \eqref{eq:totalcompute}. The blue curve is the global optimum obtained by unconstrained accuracy maximization; the orange curve follows the Chinchilla rule ($D/N=20$). A translucent iso-compute surface ($5.5\times 10^{21}$ FLOPs mixed color) and the Chinchilla plane (orange sheet) provide geometric context, while connectors (colored by accuracy difference) emphasize the gap at matched budgets. 3a–c. Scaling of parameters $N$, training tokens $D$, and inference tokens $\Omega$ with compute. Solid lines are data. Of note, the curves are not linear in log-log space, but have generally consistent trends marked by dashed lines of the fitted power laws whose exponents and $R^{2}$ are in legends. 4. Optimal training-token-per-parameter ratio $\kappa=D/N$ as a function of compute. Chinchilla is a horizontal line at $\kappa=20$; the global optimum drifts upward at high budgets, favoring more data per parameter. 5. Fraction of total FLOPs spent at inference. The Chinchilla path consistently allocates a larger and dominant share of compute to inference, whereas the global optimum remains more balanced until the largest budgets. 6. Histogram of the fitted power-law exponents from panels 3a–c, highlighting the systematically steeper growth in data and inference tokens. 7. Accuracy curves versus compute. Gray line (right axis) shows the percentage-point accuracy gap demonstrating significant differences. 8. Extra compute required by the Chinchilla model to match the accuracy of the global optimum. A stark penalty is incurred over the entire accuracy range, generally rising with higher accuracy targets underscoring the inefficiency of pretrained constraints when inference costs and scaling is taken into account.

\begin{figure}[htbp]
    \centering    \includegraphics[width=.7\textwidth]{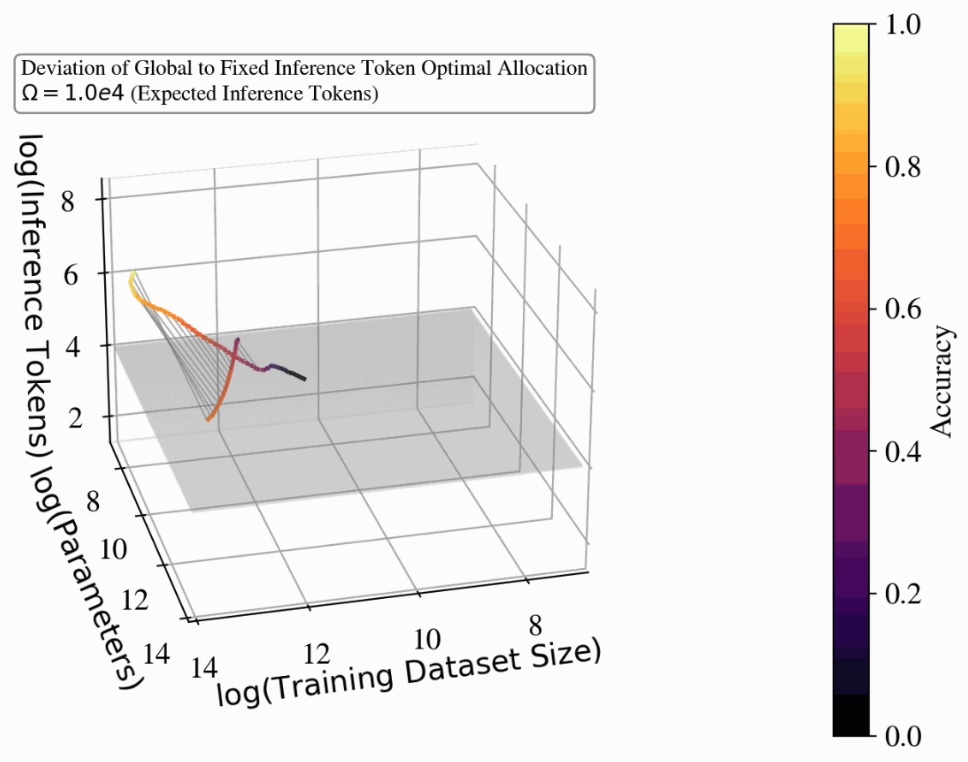}
    \caption{Optimal resource-allocation trajectories under a fixed total FLOP budget without attention costs \eqref{eq:totalcompute}. The unconstrained curve is as described in Fig.~\ref{fig:OptimalvsChinchilla}; the curve on the gray plane follows a fixed token restriction thus optimizing over only parameters and data. Data is colored by accuracy and connecting lines emphasize the gap in optimal resource allocation at matched budgets.}
\label{fig:beyondchinchillacomp}
\end{figure}

Fig.~\ref{fig:beyondchinchillacomp}, illustrates the difference between the global optimal curve and the fixed token constrained optimal which is the result of maximizing accuracy, \eqref{eq:max_accuracy}. Either restriction leads to significant inefficiencies.

\subsection{Infinite Training Data Regime}\label{ap:infinitedata}
The figure in the main text, Fig.~\ref{fig:task_difficulty}, presents inference-compute-optimal scaling across a grid of task difficulties defined by skill level $\mathcal{L}_0$ and number of required skills $m$, under the assumption of infinite training data. Scaling behavior was computed with respect to a total inference compute budget of $C_{\text{inf}} = 2NI\Omega$, neglecting attention cost, with the number of tasks fixed at $I = 10^9$. Pretraining loss was estimated using the empirical model of Hoffman et al.~\cite{Hoffmann2022}, $\mathcal{L}(N,D) = 1.69 + 406.4 \cdot N^{-0.34} + 410.7 \cdot D^{-0.28}$, and mapped to downstream task success probability via a two-step sigmoid function of the form $p = 1 / (1 + \exp[-a(\mathcal{L}_0 - \mathcal{L}(N,D)) + d])$, with parameters $a = 1$, $b = 20$, and $d = 0$. We fixed the number of tokens per skill step to $\omega = 25$, and sampled a grid of 10{,}000 task types from $\mathcal{L}_0 \in [3.0, 1.5]$ (in 100 linear steps such that most difficult 100 represents $\mathcal{L}_0=1.5$ and 1 represents $\mathcal{L}_0=3$) and $m \in {1, \dots, 100}$.
To determine scaling laws, we computed the optimal model size $N$ and inference token count $\Omega$ across 1250 values of total compute, $C \in \log_{10} [10^{8.5}, 10^{17}]$, and performed log-linear regressions of $\log_{10}(N)$ and $\log_{10}(\Omega)$ against $\log_{10}(C_{\text{inf}})$ for each $(\mathcal{L}_0, m)$ pair. We retained only those fits with coefficient of determination $R^2 \geq 0.7$ and at least five valid data points; invalid fits were grayed out in the plots.

The top row of the figure reports the slope of these regressions: parameter scaling exponent (left), token scaling exponent (center), and the token-to-parameter slope ratio (right). The bottom row shows intercepts: model size and token count at a reference compute of $C_{\text{inf}} = 10^{12}$ (left, center), and token count required at a fixed model size (right). Summary statistics over all valid fits are as follows. Parameter slopes ranged from 0.002 to 1.000 with a mean of 0.367, while token slopes ranged from 0.245 to 1.000 with a mean of 0.640. Parameter intercepts ranged from $-1.699$ to $10.092$ (mean: $4.400$), and token intercepts ranged from $-10.422$ to $-0.474$ (mean: $-4.776$). The $R^2$ values for parameter fits had a minimum of 0.724 and a median of 0.998, while token fits had a minimum of 0.700 and a median of 0.999. Out of 10{,}000 fits, 9{,}980 passed the parameter fit threshold and 9{,}871 passed the token fit threshold.

\begin{figure}[htbp]
    \centering    \includegraphics[width=\textwidth]{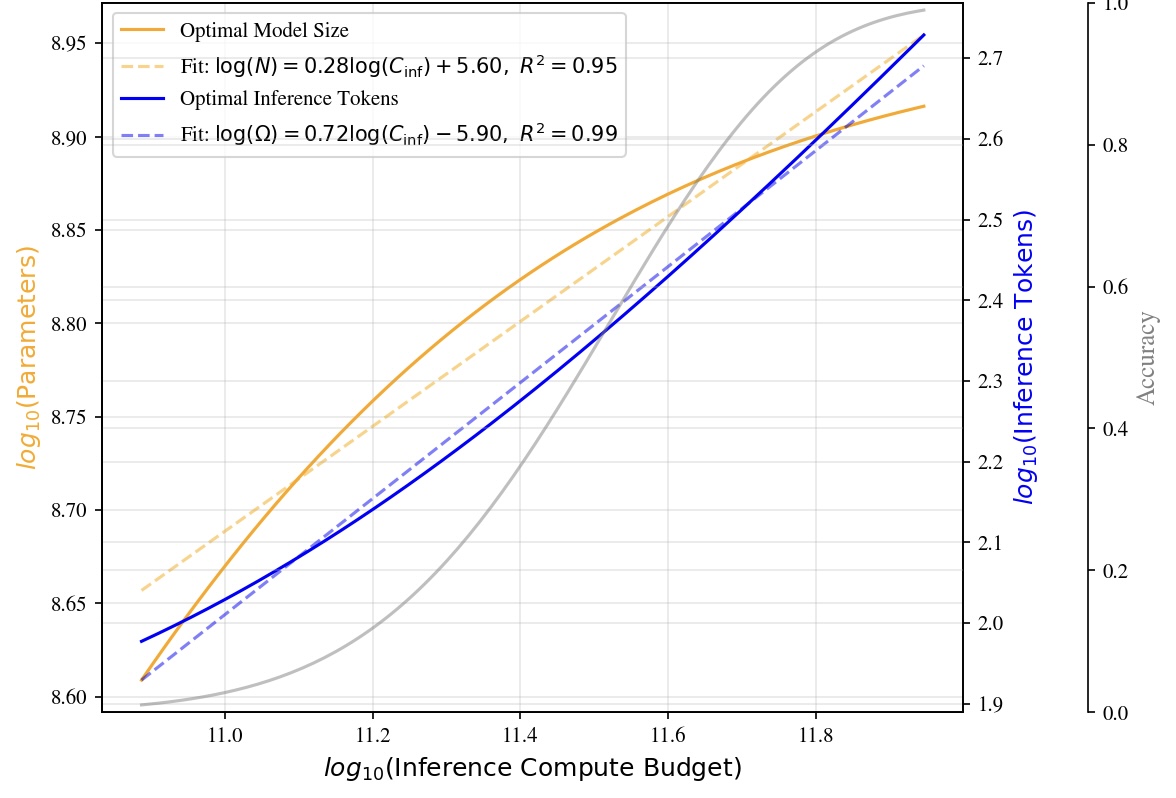}
    \caption{Inference-compute-optimal scaling for a single task in the infinite training data regime, with skill difficulty $\mathcal{L}_0 = 2.0$ and number of skills $m = 5$. The task-level success probability is determined using a sigmoid mapping with parameters $a = 1$ and $b = 20.0$. For numerical stability, skill execution is treated with continuous step sizes rather than discrete token units. The solid curves show optimal model size (orange) and inference tokens (blue) as a function of total inference compute budget $C_{\text{inf}} = 2NI\Omega$, with corresponding accuracy shown in gray (right axis). Dashed lines indicate log-linear fits: $\log_{10}(N) = 0.28 \log_{10}(C_{\text{inf}}) + 5.60$ ($R^2 = 0.95$) and $\log_{10}(\Omega) = 0.72 \log_{10}(C_{\text{inf}}) - 5.90$ ($R^2 = 0.99$). These exponents illustrate the relative contributions of model size and token count to scaling under fixed task complexity.
}
\label{fig:inference_scaling_single_task}
\end{figure}

We demonstrate a single sample fit in Fig.~\ref{fig:inference_scaling_single_task}.

\section{Supplemental Information: Can Reasoning Elicit Emergence?}

\subsection{Proof that \texorpdfstring{$\lim_{x \to 0^+} g(x) = +\infty$}{lim g(x) = infinity}}\label{ap:g_limit}
Here we demonstrate the result of the limit invoked in Proof~\ref{proof:emergentscaling}.

We evaluate the limit
\[
g(x) := \left(\frac{1}{x} - 1\right)(1 - x)^{\frac{1}{x} - 2}, \qquad x \in (0,1),
\]
and show that
\[
\lim_{x \to 0^+} g(x) = +\infty.
\]
We proceed by taking logarithms. Define \( y = g(x) \). Then:
\[
\ln y = \ln\left(\frac{1}{x} - 1\right) + \left(\frac{1}{x} - 2\right)\ln(1 - x).
\]
Noting that
\[
\frac{1}{x} - 1 = \frac{1 - x}{x}
\quad \Rightarrow \quad
\ln\left(\frac{1}{x} - 1\right) = \ln\left(\frac{1 - x}{x}\right) = \ln(1 - x) - \ln x,
\]
we substitute back to obtain:
\[
\ln y = -\ln x + \left(\frac{1}{x} - 1\right)\ln(1 - x).
\]
We now focus on the second term. Let:
\[
L(x) := \left(\frac{1}{x} - 1\right)\ln(1 - x) = \frac{1 - x}{x} \ln(1 - x).
\]
This is an indeterminate form of the type \( \frac{0}{0} \), so we apply L’Hôpital’s Rule:
\begin{align*}
\lim_{x \to 0^+} L(x)
&= \lim_{x \to 0^+} \frac{(1 - x)\ln(1 - x)}{x}
= \lim_{x \to 0^+} \frac{d}{dx}[(1 - x)\ln(1 - x)] \bigg/ \frac{d}{dx}[x] \\
&= \lim_{x \to 0^+} \left( -\ln(1 - x) - 1 \right)
= -\ln(1) - 1 = -1.
\end{align*}
Thus, the logarithm becomes:
\[
\lim_{x \to 0^+} \ln y = \lim_{x \to 0^+} \left( -\ln x - 1 \right) = +\infty.
\]
Exponentiating both sides, we conclude:
\[
\lim_{x \to 0^+} g(x) = \lim_{x \to 0^+} e^{\ln y} = e^{+\infty} = +\infty.
\]

\section{Best of N and Majority Voting}
We provide further details on deriving the asymptotic behavior of Bo$N$ with oracle verifier when assuming the Beta-point mass distribution for task success probability. We discuss model fitting details for both strategies. We also derive the push forward distribution to relate a distribution over skill levels and skills required to the pass@k performance assuming the idealized DS3 model.

\subsection{Asymptotic Behavior of BoN Accuracy}\label{ap:asymptoticBoN}
We analyze the large-$k$ behavior of the pass@$k$ coverage under a Beta-distributed success rate $\psi \sim \mathrm{Beta}(\alpha, \beta)$, given by
$$
\Psi^{\mathrm{Bo}N}(k) = 1 - \int_0^1 (1 - \psi)^k f(\psi) \, d\psi,
$$
where
$$
f(\psi) = (1-\mathcal{A})\delta(\psi)+\mathcal{A} \cdot \frac{\psi^{\alpha - 1}(1 - \psi)^{\beta - 1}}{B(\alpha, \beta)},
\quad \text{and} \quad B(\alpha, \beta) = \frac{\Gamma(\alpha)\Gamma(\beta)}{\Gamma(\alpha + \beta)}.
$$
where $\delta(\psi)$ is the Dirac-delta function. Substituting this into the integral yields:
$$
\Psi^{\mathrm{Bo}N}(k) = \mathcal{A} \left(1 - \frac{B(\alpha, \beta + k)}{B(\alpha, \beta)}\right).
$$
We now expand the Beta ratio using the Gamma representation:
$$
\frac{B(\alpha, \beta + k)}{B(\alpha, \beta)} = 
\frac{\Gamma(\alpha)\Gamma(\beta + k)}{\Gamma(\alpha + \beta + k)} \cdot \frac{\Gamma(\alpha + \beta)}{\Gamma(\alpha)\Gamma(\beta)}
= \frac{\Gamma(\beta + k)}{\Gamma(\alpha + \beta + k)} \cdot \frac{\Gamma(\alpha + \beta)}{\Gamma(\beta)}.
$$
To analyze the large-$k$ behavior, we apply Stirling’s approximation with constants:
$$
\Gamma(z) \sim \sqrt{2\pi}\, z^{z - \frac{1}{2}} e^{-z}, \quad \text{as } z \to \infty.
$$
Apply this to both $\Gamma(\beta + k)$ and $\Gamma(\alpha + \beta + k)$:
$$
\Gamma(\beta + k) \sim \sqrt{2\pi} (\beta + k)^{\beta + k - \frac{1}{2}} e^{-(\beta + k)},
$$
$$
\Gamma(\alpha + \beta + k) \sim \sqrt{2\pi} (\alpha + \beta + k)^{\alpha + \beta + k - \frac{1}{2}} e^{-(\alpha + \beta + k)}.
$$
Take their ratio:
$$
\frac{\Gamma(\beta + k)}{\Gamma(\alpha + \beta + k)}
\sim
\frac{(\beta + k)^{\beta + k - \frac{1}{2}} e^{-(\beta + k)}}
     {(\alpha + \beta + k)^{\alpha + \beta + k - \frac{1}{2}} e^{-(\alpha + \beta + k)}},
$$
since the $\sqrt{2\pi}$ terms cancel. Now simplify the exponentials:
$$
\frac{e^{-(\beta + k)}}{e^{-(\alpha + \beta + k)}} = e^{\alpha},
$$
and simplify the power term using:
$$
\frac{(\beta + k)^{\beta + k - \frac{1}{2}}}{(\alpha + \beta + k)^{\alpha + \beta + k - \frac{1}{2}}}
=
\left( \frac{\beta + k}{\alpha + \beta + k} \right)^{\beta + k - \frac{1}{2}} \cdot
(\alpha + \beta + k)^{-\alpha}.
$$
As $k \to \infty$, we have:
$$
\frac{\beta + k}{\alpha + \beta + k} = 1 - \frac{\alpha}{k} + \mathcal{O}\left(\frac{1}{k^2}\right),
$$
and so:
$$
\left(1 - \frac{\alpha}{k}\right)^{k} \to e^{-\alpha}.
$$
Thus,
$$
\left( \frac{\beta + k}{\alpha + \beta + k} \right)^{\beta + k - \frac{1}{2}} \sim e^{-\alpha},
$$
which cancels the $e^\alpha$ from the exponential ratio earlier. Therefore, the entire expression simplifies to:
$$
\frac{\Gamma(\beta + k)}{\Gamma(\alpha + \beta + k)} \sim (\alpha + \beta + k)^{-\alpha},
\quad \text{as } k \to \infty.
$$
Returning to the Beta ratio:
$$
\frac{B(\alpha, \beta + k)}{B(\alpha, \beta)}
\sim \frac{\Gamma(\alpha + \beta)}{\Gamma(\beta)} \cdot (\alpha + \beta + k)^{-\alpha}.
$$Thus,
$$
\Psi^{\mathrm{Bo}N}(k) \approx \mathcal{A} \left(1 - \frac{\Gamma(\alpha + \beta)}{B(\alpha, \beta)\Gamma(\beta)} \cdot k^{-\alpha} \right)
= \mathcal{A} \left(1 - \frac{\Gamma(\alpha)}{B(\alpha, \beta)} \cdot k^{-\alpha} \right),
$$
since $B(\alpha, \beta) = \Gamma(\alpha)\Gamma(\beta)/\Gamma(\alpha + \beta)$. In the large-$k$ limit, the pass@$k$ coverage exhibits power-law decay:
$$
\Psi^{\mathrm{Bo}N}(k) \sim 1 - C \cdot k^{-\alpha},
\quad \text{with} \quad C = \frac{\Gamma(\alpha)}{B(\alpha, \beta)}.
$$
All exponential and constant terms cancel precisely in the Beta ratio due to the structure of the Gamma function, leaving a clean $k^{-\alpha}$ scaling at leading order.

\subsection{BoN Model Fit}
We fit the theoretical Bo$N$ coverage curve
$$
\Psi^{\mathrm{Bo}N}(k) = \mathcal{A} \left( 1 - \frac{B(\alpha, \beta + k)}{B(\alpha, \beta)} \right)
$$
to empirical pass@$k$ data for each model using nonlinear least squares. Specifically, we minimize the mean squared error (MSE) between predicted and observed values:
$$
\min_{\alpha,\,\beta,\,\mathcal{A}} \sum_{i} \left( \Psi^{\mathrm{Bo}N}(k_i; \alpha, \beta, \mathcal{A}) - \text{pass@}k_i^{\text{(empirical)}} \right)^2,
$$
subject to constraints $\alpha > 0$, $\beta > 0$, and $0 < \mathcal{A} \leq 1$. Initial values are manually selected for each model based on qualitative alignment with the empirical curve, and optimized using L-BFGS-B. The fitting results are visualized by comparing both theoretical predictions and residuals for the initial and optimized parameters. We emphasize that this fitting captures the effective shape of the error distribution over trials. While the Bo$N$ expression arises from an analytically principled model, the fitting procedure here is descriptive and not used for inference or prediction beyond interpolating observed pass@$k$ values.

\subsection{Derivation of the Majority-Vote Saturation Formula and Fitting}\label{ap:mv}
Beginning with the Bo$N$ values for $\mathcal{A},\,\alpha,\,\beta$, we shown the derivation and full model fitting procedure used for majority voting.

With $k$ independent samples, let
$\mathbf y=(y_0,y_1,\dots,y_{|\mathcal Y|})$ be the vote counts for the correct label ($j=0$) and the $|\mathcal Y|-1$ incorrect labels.
Under a per-draw success probability $\psi$ and error spectrum
$q_j(\psi)=(1-\psi)c_j$ (with $\sum_{j=1}^{|\mathcal Y|}c_j=1$), the joint PMF of $\mathbf y$ is the multinomial:
$$
\Pr(\mathbf y\mid\psi)=
\frac{k!}{\prod_{j'=0}^{|\mathcal Y|}y_{j'}!}\,
\psi^{\,y_0}\prod_{j=1}^{|\mathcal Y|}\bigl[(1-\psi)c_j\bigr]^{y_j},
\qquad
\sum_{j'=0}^{|\mathcal Y|}y_{j'}=k .
$$
Majority voting (uniform tie-breaking) declares the answer correct when:
$$
y_0=\max_{j\in\{0,\dots,|\mathcal Y|\}}y_j
\quad\text{and}\quad
\text{a tie if existing among }s
    =\sum_{j=0}^{|\mathcal Y|}\mathbbm 1[y_j=y_0]
\text{ labels is resolved in favor of } y_0.
$$
Hence the conditional success probability is
$$
\Pr(\text{MV correct}\mid\mathbf y)=
\frac{\mathbbm 1 \bigl[y_0=\max_j y_j\bigr]}{
      \sum_{j=0}^{|\mathcal Y|}\mathbbm 1[y_j=y_0]} .
$$
Taking expectation over $\mathbf y$ and over the Beta-point mass prior
$f(\psi)=(1-\mathcal{A})\delta(\psi)+\mathcal A\,\psi^{\alpha-1}(1-\psi)^{\beta-1}/B(\alpha,\beta)$ gives
$$
\Psi^{\mathrm{MV}}(k,\mathcal Y)=
\mathcal A
\int_{0}^{1}
  \sum_{\substack{\sum_{j'=0}^{|\mathcal Y|}y_{j'}=k}}
   \frac{\mathbbm 1[y_0=\max_j y_j]}{
         \sum_{j=0}^{|\mathcal Y|}\mathbbm 1[y_j=y_0]}\;
   \frac{k!}{\prod_{j'}y_{j'}!}\,
   \psi^{\,y_0}\prod_{j=1}^{|\mathcal Y|}[(1-\psi)c_j]^{y_j}\;
   \frac{\psi^{\alpha-1}(1-\psi)^{\beta-1}}{B(\alpha,\beta)}
   \,d\psi
$$
since the delta function expectation is zero. Because $\sum_{j=1}^{|\mathcal Y|}y_j=k-y_0$, the $\psi$-integral is a Beta kernel:
$$
\int_{0}^{1}
  \psi^{\,y_0+\alpha-1}(1-\psi)^{\,k-y_0+\beta-1}\,d\psi
  =B(\alpha+y_0,\beta+k-y_0).
$$
Substituting yields the closed form
$$
\Psi^{\mathrm{MV}}(k,\mathcal Y)=
\mathcal A
\sum_{\substack{\sum_{j'}y_{j'}=k}}
\frac{\mathbbm 1[y_0=\max_j y_j]}
     {\sum_{j=0}^{|\mathcal Y|}\mathbbm 1[y_j=y_0]}
\;
\frac{k!}{\prod_{j'}y_{j'}!}\,
\Bigl(\prod_{j=1}^{|\mathcal Y|}c_j^{\,y_j}\Bigr)
\frac{B(\alpha+y_0,\beta+k-y_0)}{B(\alpha,\beta)}
$$
For fixed $\psi$, the empirical frequencies
$\hat p_0=y_0/k$ and $\hat p_j=y_j/k$ satisfy
$(\hat p_0,\hat p_1,\dots)\xrightarrow{\text{a.s.}}(\psi,(1-\psi)c_1,\dots)$. Majority voting succeeds asymptotically when
$$
\psi > (1-\psi)\max_{1\le j\le|\mathcal Y|-1} c_j
\;\;\Longleftrightarrow\;\;
\psi > r_c,
\qquad
r_c \;=\;\frac{c^{*}}{1+c^{*}},\;
c^{*}:=\max_j c_j .
$$
Therefore the limiting success probability is
$$
P_{\infty}^{\mathrm{MV}}
    \;=\;
    \mathcal A\,
    \Pr_{\psi\sim\mathrm{Beta}(\alpha,\beta)}\bigl[\psi>r_c\bigr]
    \;=\;\mathcal A\,\left(1-
    \Pr_{\psi\sim\mathrm{Beta}(\alpha,\beta)}\bigl[\psi\leq r_c\bigr]\right)\mbox{.}
$$
Since $\psi \sim \mathrm{Beta}(\alpha, \beta)$, the probability $\Pr[\psi \le r_c]$ is exactly the CDF of the Beta distribution at $r_c$, which is $I_{r_c}(\alpha, \beta)$ the regularized incomplete Beta function and defined by:
$$
I_{r_c}(\alpha, \beta) = \frac{1}{B(\alpha, \beta)} \int_0^{r_c} t^{\alpha - 1}(1 - t)^{\beta - 1} dt.
$$
So putting this together:
$$
P_\infty^{\mathrm{MV}} = \mathcal{A} \cdot (1 - I_{r_c}(\alpha, \beta)).
$$
Given an empirical plateau $P_{\infty}^{\mathrm{MV}}$, we set
$$
\Delta := 1-\frac{P_{\infty}^{\mathrm{MV}}}{\mathcal A},
\quad
r_c = I^{-1}_\Delta(\alpha,\beta),
\quad
c^{*}= \dfrac{r_c}{1-r_c}
$$
where for $y=I_{x}(\alpha,\beta)$, the inverse function is $x=I_{y}^{-1}(\alpha,\beta)$. This $c^{*}$ binds the largest single error weight consistent with the observed saturation and the previously fitted Beta prior.

As previously mentioned, any distribution over $c_j$ can be chosen. We assume a truncated geometric spectrum as it demonstrates decaying preference over incorrect outputs. For
$$
c_j = \frac{(1-\lambda)\lambda^{j-1}}{1-\lambda^{|\mathcal Y|-1}},
\qquad
\{j=1,\dots,|\mathcal Y|-1\},
$$
we impose $c^{*}=c_1$ and solve
$$
c^{*}= \frac{1-\lambda}{1-\lambda^{|\mathcal Y|-1}}
\;\;\Longrightarrow\;\;
\lambda = \operatorname{solve}(c^{*},|\mathcal Y|).
$$
We take an alphabet of size $|\mathcal Y|=10$. This fully determines the $c_j$ that enter the finite-$k$ formula $\Psi^{\mathrm{MV}}(k,\mathcal Y)$, enabling Monte Carlo or exact evaluation of $\Psi^{\mathrm{MV}}(k,\mathcal Y)$ for all $k$. For analytical tractability and accuracy, we perform a $10,000$ sample Monte Carlo simulation for each point as for high $k$, the Beta distribution simulation becomes expensive. We do not study or demonstrate extrapolation.

\subsection{Push-Forward Skill Distribution to Estimate Search Directionality}
\label{app:pushforward}
Here, we detail the fitting procedure used to obtain the model–specific search-directionality parameter $\hat{\iota}$ given a simple two-mode Gaussian with parameters
\(\bigl\{\mu_1,\sigma_1,\mu_{100},\sigma_{100}\bigr\}\).

Let \(m\in\{1,100\}\) denote the number of required reasoning steps and let \(\mathcal L_0\) denote the difficulty for a skill as in the two-step sigmoid ansatz. We assume
$$
  f_{m}(m)=0.93\,\delta(m-1)+0.07\,\delta(m-100),
  \qquad
  \mathcal L_0\mid m
    ~\sim~
    \mathcal N\bigl(\mu_{m},\sigma_{m}^{2}\bigr).
$$
where $\delta(\cdot)$ is the Dirac-delta to provide an easy fit to the data since Pythia is step limited and the only model without $\mathcal{A}=1$. This is a simplifying assumption and not necessary. Furthermore, we assume there are not multiple solution strategies such that every task is to be solved with the $\mathcal{L}_0,\,m$ pair given by the distribution. Then, we take a Chinchilla style loss to sigmoid to idealized CoT DS3 mapping to get individual trial success rates:
\begin{align}
  p
    &= a\,\sigma\bigl(b(\mathcal L_0-\mathcal L)\bigr)+d,\quad
     \sigma(z)=\frac{1}{1+e^{-z}}, \quad
      a=1,~b=5,~d=0, \\
  u
    &= \hat{\iota}\,p,\\
  \psi
    &= I_{u}\bigl(m,T_{\max}-m+1\bigr)
\end{align}

Because the map \(\mathcal L_0\mapsto\Psi\) is strictly monotone for each fixed \(m\), the change-of-variables formula gives
\begin{equation}
  f_{\Psi\mid m}(\psi)
  ~=~
  f_{\mathcal L_0\mid m}\bigl(\ell(\psi,m)\bigr)
  \Bigl|
    \tfrac{d\Psi}{d\mathcal L_0}
      \Big|_{\mathcal L_0(\psi,m)}^{-1},
\end{equation}
Then,
\begin{equation}
  \frac{d\Psi}{d\mathcal L_0}
  ~=~
  \hat{\iota}\,a\,b\;
  \sigma\bigl(b(\mathcal L_0-\mathcal L)\bigr)
  \bigl[1-\sigma\bigl(b(\mathcal L_0-\mathcal L)\bigr)\bigr]\,
  \frac{u^{\,m-1}(1-u)^{T_{\max}-m}}{B(m,k)}.
\end{equation}

Combining gives,
\begin{equation}
  f_{\Psi}(\psi)
  =
  (1-\mathcal A)\,\delta(\psi)
  \;+\;
  \mathcal A
  \sum_{m\in\{1,100\}} \Pr[m]\,
  \frac{B(m,k)}{\hat{\iota}\,a\,b}\,
  \frac{
        f_{\mathcal L_0\mid m}\bigl(\ell(\psi,m)\bigr)}
       {
        \sigma\bigl(b(\ell-\mathcal L)\bigr)
        \bigl[1-\sigma\bigl(b(\ell-\mathcal L)\bigr)\bigr]\,
        u^{\,m-1}(1-u)^{k-1}}.
\end{equation}
Here, \(\mathcal A\) is the empirically estimated non-zero spike probability from the Beta and point-mass fit used for Bo$N$. For all models we jointly optimize
\[
  \boldsymbol\theta
  ~=~
  \bigl(\mu_{1},\sigma_{1},
         \mu_{100},\sigma_{100},
         \underline{\hat{\iota}}\bigr)
\]
where $\underline{\hat{\iota}}=(\hat{\iota}_{\mathrm{model\ 1}}, \,\hat{\iota}_{\mathrm{model\ 2}}, \,...)$ by maximum likelihood on the observed pass@$k$ sample. We use a global differential evolution to optimize from `scipy.optimize'. The highly simplified two-mode Gaussian is then demonstrated to capture important features of the distributions and provides a prediction on the search directionality. Extensions could cross reference with increased token budget for individual trials to better estimate the true underlying task distribution.

\subsection*{Tabulated Parameters for Models Considered}
In Table~\ref{tab:modelparams}, we summarize the relevant compute and budget parameters for the models analyzed throughout this study. All models are evaluated assuming a prompt size of 100 tokens. The available output budget is computed as the publicly reported context length minus the prompt size, and the number of allowed reasoning steps \(T_{\text{max}}\) is estimated by dividing the output budget by a constant \(\omega = 25\) tokens per skill step:
\[
T_{\text{max}} = \frac{\text{output budget}}{\omega}.
\]

We report the number of model parameters \(N\), the size of the training dataset (in tokens) $D$, the estimated training loss from the Chinchilla equation $\mathcal{L}(N,D)= 1.69 + \frac{406.4}{N^{0.34}} + \frac{410.7}{D^{0.28}}$, and the number of max reasoning steps permitted under the output token budget $T_{\text{max}}$. This table is used for the empirical fits of pass@($k$) and majority voting curves across models as well as theoretical estimates of search directionality.

\begin{table}[ht]
\centering
\caption{Model Parameters and Effective Step Budgets}
\label{tab:modelparams}
\begin{tabular}{lccccr}
\toprule
\textbf{Model} & \(\boldsymbol{N}\) & \(\boldsymbol{D}\) & \(\boldsymbol{\mathcal{L}}\) & \textbf{Output Budget} & \(\boldsymbol{T_{\text{max}} = \text{output budget}/25}\) \\
\midrule
Llama-3-8B              & \(8 \times 10^9\)  & \(1.5 \times 10^{13}\) & 1.9485 & 8092 & 327.68 \\
Gemma-2B                & \(2 \times 10^9\)  & \(3 \times 10^{12}\)   & 2.1014 & 8092 & 327.68 \\
Pythia-2.8B             & \(2.8 \times 10^9\)& \(3 \times 10^{11}\)   & 2.1906 & 1948 & 81.92  \\
Llama-3-8B-Instruct     & \(8 \times 10^9\)  & \(1.5 \times 10^{13}\) & 1.9485 & 8092 & 327.68 \\
Llama-3-70B-Instruct    & \(70 \times 10^9\) & \(1.5 \times 10^{13}\) & 1.8575 & 8092 & 327.68 \\
\bottomrule
\end{tabular}
\end{table}

\end{appendices}

\end{document}